\documentclass[11pt]{article}
\usepackage{vitercik}
\usepackage{subcaption}
\usepackage{tabularx}
\usepackage{multirow}
\usepackage[symbol]{footmisc}

\newcommand{\profit}{\textnormal{profit}}
\newcommand{\pspace}{\pazocal{P}}
\newcommand{\domain}{\pazocal{X}}
\newcommand{\range}{\pazocal{Y}}
\newcommand{\fclass}{\pazocal{F}}
\newcommand{\hyp}{\pazocal{H}}
\newcommand{\Qset}{\pazocal{Q}}

\newcolumntype{L}[1]{>{\hsize=#1\hsize\raggedright\arraybackslash}X}%

\makeatletter
\newcommand\footnoteref[1]{\protected@xdef\@thefnmark{\ref{#1}}\@footnotemark}
\makeatother

\title{Generalization Guarantees for Multi-item Profit Maximization:\\Pricing, Auctions, and Randomized Mechanisms}
\author{
	Maria-Florina Balcan \\ \small Carnegie Mellon University \\ \small \texttt{ninamf@cs.cmu.edu}
	\and 
	Tuomas Sandholm \\ \small Carnegie Mellon University \\
	\small Optimized Markets, Inc.\\
	\small Strategic Machine, Inc.\\
	\small Strategy Robot, Inc.\\
	\small \texttt{sandholm@cs.cmu.edu}
	\and 
	Ellen Vitercik \\ \small Stanford University \\ \small \texttt{vitercik@stanford.edu}}

\begin{document}
	\maketitle
	
	\begin{abstract}
		We study multi-item profit maximization when there is an underlying distribution over buyers' values. In practice, a full description of the distribution is typically unavailable, so we study the setting where the mechanism designer only has samples from the distribution. If the designer uses the samples to optimize over a complex mechanism class---such as the set of all multi-item, multi-buyer mechanisms---a mechanism may have high average profit over the samples but low expected profit. This raises the central question of this paper: how many samples are sufficient to ensure that a mechanism's average profit is close to its expected profit? To answer this question, we uncover structure shared by many pricing, auction, and lottery mechanisms: for any set of buyers' values, profit is piecewise linear in the mechanism's parameters. Using this structure, we prove new bounds for mechanism classes not yet studied in the sample-based mechanism design literature and match or improve over the best-known guarantees for many classes.
	\end{abstract}

\section{Introduction}

The design of profit-maximizing mechanisms is a fundamental problem with diverse applications including Internet retailing, advertising markets, and strategic sourcing.
This problem has traditionally been studied under the assumption that there is a joint distribution from which the buyers' values are drawn and that the mechanism designer knows this distribution in advance. This assumption has led to groundbreaking theoretical results in the single-item setting~\citep{Myerson81:Optimal}, but transitioning from theory to practice is challenging because the true distribution over buyers' values is typically unknown. Moreover, in the dramatically more challenging multi-item setting, the support of the distribution alone is
often doubly exponential (even if there were just a single buyer with a finite type space\footnote{When each buyer's values are independent from every other buyer's values, the number of support points is $nk^{2^m}$, where $n$ is the number of buyers, $k$ is the number of discrete value levels a buyer can assign to a bundle, and $m$ is the number of items.  This is because each of the $2^m$ bundles can take any of $k$ values.  With correlated valuations, the prior has $k^{2^{nm}}$ support points.}), so obtaining and storing the distribution is typically impossible.

We relax this strong assumption and instead assume that the mechanism designer only has a set of independent samples from the distribution~\citep{Likhodedov04:Boosting,Likhodedov05:Approximating,Sandholm15:Automated}. 
This type of \emph{sample-based mechanism design} reflects current industry practices since many companies---such as online ad exchanges~\citep{He14:Practical, Medina17:Revenue}, sponsored search platforms~\citep{Edelman07:Internet, Tang17:Reinforcement}, and travel companies~\citep{Yee15:Aerosolve}---use historical purchase data to adjust the sales mechanism.

In most multi-item settings, the form of the revenue-maximizing mechanism is still a mystery. Therefore, rather than use the samples to uncover \emph{the} optimal mechanism, much of the literature on sample-based mechanism design suggests that we first fix a reasonably expressive mechanism class and then use the samples to optimize over the class. If, however, the mechanism class is complex and the number of samples is not sufficiently large, a mechanism with high average profit over the set of samples may have low expected profit on the actual unknown distribution: \emph{overfitting} has occurred. This motivates an important question in sample-based mechanism design:

\smallskip

\noindent\emph{Given a set of samples and a mechanism class $\mclass$, what is the difference between the average profit over the samples and the expected profit on the unknown distribution for any mechanism in $\mclass$?}

\smallskip

\noindent If this difference is small, the mechanism in $\mclass$ that maximizes average profit over the set of samples nearly maximizes expected profit over the distribution as well.

We present a general theory for deriving \emph{generalization guarantees} in multi-item settings.
A generalization guarantee for a mechanism class $\mclass$ bounds the difference between the average profit over the samples and expected profit for any mechanism in $\mclass$.
These bounds can be applied no matter how the mechanism designer optimizes over the class, using an automated or manual approach. Optimization algorithms for many of the mechanisms we study have been developed in prior research~\citep{Sandholm15:Automated,Cai17:Learning,Balcan20:Efficient}. 

This paper is part of a line of research that studies how learning theory can be used to design and analyze mechanisms. Most of these papers have studied only single-parameter settings~\citep{Balcan05:Mechanism,Balcan08:Reducing,Elkind07:Designing,Cole14:Sample, Huang15:Making, Mohri14:Learning, Morgenstern15:Pseudo, Roughgarden15:Ironing,Devanur16:Sample, Hartline16:Non,Gonczarowski17:Efficient,  Bubeck17:Online, Alon17:Submultiplicative, Guo19:Settling}. In contrast, we focus on multi-item mechanism design, as have recent papers by~\citet{Morgenstern16:Learning,Syrgkanis17:Sample, Medina17:Revenue, Cai17:Learning}, and \citet{Gonczarowski18:Sample}.

\subsection{Our contributions}

Our contributions come in two interrelated parts.

\paragraph{A general theory that unifies diverse mechanism classes.}
We uncover a structural property shared by a wide variety of mechanisms which allows us to prove generalization guarantees: for any fixed set of bids, profit is a piecewise linear function of the mechanism's parameters. Our main theorem provides generalization bounds for any class exhibiting this structure. We relate the complexity of the partition splitting the parameter space into linear portions to the intrinsic complexity of the mechanism class, which we quantify using \emph{pseudo-dimension.} In turn, pseudo-dimension bounds imply generalization bounds. We prove that many seemingly disparate mechanisms share this structure, and thus our main theorem yields learnability guarantees. By contrast, previous research on multi-item mechanism design focused on deriving guarantees for a few mechanism classes that are ``simple'' by design~\citep{Morgenstern16:Learning,Syrgkanis17:Sample}.
\begin{table}
	\scriptsize
			\begin{tabularx}{\textwidth}{L{0.85} L{1.15} L{1.1} L{.9}}
				\toprule
				\textbf{Category} & \textbf{Mechanism class} & \textbf{Valuations}& \textbf{Result}\\\midrule
				Pricing mechanisms & Item-pricing mechanisms & General, unit-demand, additive & Lemmas~\ref{lem:item_pricing_add}, \ref{lem:item_product}, \ref{thm:item_pricing_unit}, \ref{thm:item_pricing_general}\\\cmidrule{2-4}
				& Two-part tariffs & General & Lemma~\ref{lem:2part}\\\cmidrule{2-4}
				& Non-linear pricing mechanisms & General & Lemmas~\ref{lem:nonlinear}, \ref{lem:nonlinear_additive}\\\midrule
				Auctions & Second-price auctions with reserves & Additive & Lemmas~\ref{lem:second_price}, \ref{lem:second_price_product}, \ref{thm:second_price_sep}\\\cmidrule{2-4}
				& Affine maximizer auctions & General& Lemma~\ref{lem:AMA}\\\cmidrule{2-4}
				& Virtual valuation combinatorial auctions & General& Lemma~\ref{lem:AMA}\\\cmidrule{2-4}
				& Mixed-bundling auctions with reserves & General& Lemma~\ref{lem:MBARP}\\\midrule
				Randomized mechanisms & Lotteries & Additive, unit-demand & Lemmas~\ref{lem:lottery}, \ref{lem:lottery_product}\\\bottomrule
			\end{tabularx}
			\caption{Some of the main mechanism classes we analyze.}\label{tab:classes}
\end{table}
Table~\ref{tab:classes} summarizes some of the main mechanism classes we analyze and Tables~\ref{tab:results_lotteries}, \ref{tab:results_pricing}, and \ref{tab:results_auctions} summarize our bounds.

Our main theorem applies to lotteries, a general representation of randomized mechanisms which generate higher expected revenue than deterministic mechanisms in many settings~\citep{Conitzer03:Applications, Dobzinski09:Power}.
We also provide guarantees for \emph{item-pricing mechanisms} where each item has a price and buyers buy their utility-maximizing bundles. Additionally, we study \emph{multi-part tariffs}, where there is an upfront fee and a price per unit. These tariffs and other non-linear pricing mechanisms have been studied in economics for decades~\citep{Oi71:Disneyland, Feldstein72:Equity, Wilson93:Nonlinear}. 
Our main theorem also applies to many auction classes, such as \emph{second price auctions} and well-studied generalized VCG auctions including \emph{affine maximizer auctions (AMAs)} and \emph{mixed-bundling auctions}~\citep{Sandholm15:Automated, Roberts79:Characterization, Lavi03:Towards, Dobzinski08:Characterizations, Jehiel07:Mixed}. Under AMAs, revenue is not piecewise-linear in the original parameter space, but we show it is piecewise-linear in a higher-dimensional space.

A key challenge we face is the sensitivity of these mechanisms to small changes in their parameters. For example, changing the price of a good can cause a steep drop in profit if the buyer no longer wants to buy it. Meanwhile, for many well-understood function classes in machine learning, there is a close connection between the distance in parameter space between two parameter vectors and the distance in function space between the two corresponding functions. Since profit functions do not exhibit this predictable behavior, we must carefully analyze the structure of the mechanisms we study in order to derive our generalization guarantees.

\paragraph{Data-dependent generalization guarantees.} We strengthen our main theorem when the distribution over buyers' values is ``well-behaved,'' proving generalization guarantees that are independent of the number of items for item-pricing mechanisms, second price auctions with reserves, and lottery mechanisms. Under anonymous prices, our bounds do not depend on the number of buyers either. These guarantees hold when the buyers are additive with values drawn from \emph{item-independent distributions} (buyer $i_1$'s value for item $j$ is independent from her value for item $j'$, but her value for item $j$ may be arbitrarily correlated with buyer $i_2$'s value for item $j$). Buyers with item-independent value distributions have been studied extensively in prior research~ \citep{Cai17:Learning, Yao14:n, Cai16:Duality, Goldner16:Prior, Babaioff17:Menu, Chawla07:Algorithmic, Hart12:Approximate}. This could model buyers at, for example, antique auctions and art auctions (as long as there are no collections to try to assemble or the collections are sold as atomic lots).

\begin{table}
	\scriptsize
	\begin{tabularx}{\textwidth}{L{1} L{1} L{1} L{1}}
		\toprule
		Valuations & Auction class & Our bounds & Prior bounds\\\midrule
		Additive or unit-demand  & Length-$\ell$ lottery menu & $U\sqrt{\ell m \log (\ell m)/N}$ & N/A\\\midrule
		Additive, item-independent\footnotemark[1] &  Length-$\ell$ item lottery menu  &$U\sqrt{\ell \log \ell /N}$ & N/A\\\bottomrule
	\end{tabularx}
	\noindent\par
	{\center
		\scriptsize
		\footnotemark[1]~~{Additive cost function}
	}
	
	\caption{Generalization bounds in big-$\tilde O$ notation for lotteries. The maximum profit achievable by any mechanism in the class over the support of the buyers' valuation distribution is $U$. There are $m$ items, $N$ samples, and the cost function is general unless otherwise noted.}\label{tab:results_lotteries}
\end{table}

\begin{table}
	\scriptsize
	\begin{tabularx}{\textwidth}{L{0.55} L{1} L{0.8} L{1.25} L{1.4}}
		\toprule
		\textbf{Valuations} & \textbf{Mechanism class} & \textbf{Price class} & \textbf{Our bounds} & \textbf{Prior bounds} \\\midrule
		General & \multirow{2}{*}{\specialcell{Length-$\ell$ menus of\\two-part tariffs\\over $\kappa$ units}} & Anonymous & $U\sqrt{\ell \log (\kappa n\ell)/N}$ &N/A\\\cmidrule{3-5}
		& & Non-anonymous & $U\sqrt{n\ell \log (\kappa n\ell)/N}$ &N/A\\\cmidrule{2-5}
		& \specialcell{Non-linear pricing} & Anonymous & $U\sqrt{m\prod_{i = 1}^m (\kappa_i + 1)/N}$\footnotemark[3] &N/A\\\cmidrule{3-5}
		& & Non-anonymous & $U\sqrt{nm\prod_{i = 1}^m (\kappa_i + 1)/N}$\footnotemark[3] &N/A\\\cmidrule{2-5}
		& \multirow{2}{*}{\specialcell{Additively\\decomposable\\non-linear pricing}} & Anonymous & $U\sqrt{m\sum_{i = 1}^m \kappa_i/N}$\footnotemark[3] &N/A\\\cmidrule{3-5}
		& & Non-anonymous & $U\sqrt{nm\sum_{i = 1}^m \kappa_i/N}$\footnotemark[3] &N/A\\\cmidrule{2-5}
		& Item-pricing & Anonymous & $U\sqrt{m^2/N}$& $U\sqrt{m^2/N}$\footnotemark[4]\\\cmidrule{3-5}
		& & Non-anonymous & $U\sqrt{nm(m + \log n)/N}$ & $U\sqrt{nm^2\log n/N}$\footnotemark[4]\\\midrule
		Unit-demand & Item-pricing & Anonymous & $U\sqrt{m \cdot \min\{m, \log (nm)\}/N}$& $U\sqrt{m^2/N}$\footnotemark[4]\\\cmidrule{3-5}
		& & Non-anonymous & $U\sqrt{nm \log (nm)/N}$ & $U\sqrt{nm^2\log n/N}$\footnotemark[4]\\\midrule
		Additive & Item-pricing & Anonymous & $U\sqrt{m \log m/N}$& $U\sqrt{m \log m/N}$\footnotemark[4], $\left(U/\delta\right)\sqrt{m\log\left(nN\right)/N}$\footnotemark[2]\\\cmidrule{3-5}
		& & Non-anonymous & $U\sqrt{nm \log (nm)/N}$ & $U\sqrt{nm \log (nm)/N}$\footnotemark[4], $\left(U/\delta\right)\sqrt{nm\log \left(N\right)/N}$\footnotemark[2]\\\midrule
		\multirow{2}{*}{\specialcell{Additive,\\item-\\independent\footnotemark[1]}} & Item-pricing & Anonymous & $U\sqrt{1/N}$& $U\sqrt{m \log m/N}$\footnotemark[4], $\left(U/\delta\right)\sqrt{m\log\left(nN\right)/N}$\footnotemark[2]\\\cmidrule{3-5}
		& & Non-anonymous& $U\sqrt{n \log n/N}$ & $U\sqrt{nm \log (nm)/N}$\footnotemark[4], $\left(U/\delta\right)\sqrt{nm\log \left(N\right)/N}$\footnotemark[2]\\\bottomrule
	\end{tabularx}
	\noindent\par
	{\center
		\scriptsize
		\footnotemark[1]~~{Additive cost function};\quad
		\footnotemark[3]~~{$\kappa_i$ is an upper bound on the number of units available of item $i$};\quad
		\footnotemark[4]~~{\citet{Morgenstern16:Learning}};\quad
		\footnotemark[2]~~{\citet{Syrgkanis17:Sample}. The probability these bounds fail to hold is $\delta$. In all other bounds, $\delta$ appears in a log so we suppress it using big-$\tilde O$ notation.}
	}
	
	\caption{Generalization bounds in big-$\tilde O$ notation for pricing mechanisms. We denote the maximum profit achievable by any mechanism in the class over the support of the buyers' valuation distribution by $U$. There are $m$ items, $n$ buyers, and $N$ samples. The cost function is general unless otherwise noted.}\label{tab:results_pricing}
\end{table}
\begin{table}
	\scriptsize
	\begin{tabularx}{\textwidth}{L{1} L{1} L{.9} L{1.1}}
		\toprule
		\textbf{Valuations} & \textbf{Auction class} & \textbf{Our bounds} & \textbf{Prior bounds}\\\midrule
		General & AMAs and $\lambda$-auctions & $U\sqrt{n^{m+1}m\log n/N}$ & $cU\sqrt{m/N} n^{m+2}\left(n^2 + \sqrt{n^m}\right)$\footnotemark[5]\footnotemark[6]\\\cmidrule{2-4}
		& VVCAs & $U\sqrt{n^2m2^m\log n/N}$ & $cU\sqrt{m/N} n^{m+2}\left(n^2 + \sqrt{n^m}\right)$\footnotemark[5]\footnotemark[6]\\\cmidrule{2-4}
		&MBARPs &$U \sqrt{m(\log n + m)/N}$&$U \sqrt{m^3 \log n/N}$\footnotemark[5]\\\midrule
		Additive & Second price item auctions with anonymous reserve prices &$U\sqrt{m \log m/N}$ &$U\sqrt{m \log m/N}$\footnotemark[4]\\\cmidrule{2-4}
		& Second price item auctions with non-anonymous reserve prices &$U\sqrt{nm \log (nm)/N} $ &$U\sqrt{nm \log (nm)/N}$\footnotemark[4]\\\midrule
		\multirow{2}{*}{\specialcell{Additive, item-independent\footnotemark[1]}} & Second price item auctions with anonymous reserve prices &$U\sqrt{1/N}$ &$U\sqrt{m \log m/N}$\footnotemark[4]\\\cmidrule{2-4}
		& Second price item auctions with non-anonymous reserve prices &$U\sqrt{n \log n/N} $ &$U\sqrt{nm \log (nm)/N}$\footnotemark[4]\\\bottomrule
	\end{tabularx}
	\noindent\par
	{\center
		\scriptsize
		\footnotemark[1]~~{Additive cost function};\quad
		\footnotemark[6]~~{The value of $c > 1$ depends on the range of the auction parameters};\quad
		\footnotemark[4]~~{\citet{Morgenstern16:Learning}};\quad
		\footnotemark[5]~~{\citet{Balcan16:Sample}}.
	}
	\caption{Generalization bounds in big-$\tilde O$ notation for auctions. We denote the maximum profit achievable by any mechanism in the class over the support of the buyers' valuation distribution by $U$. There are $m$ items, $n$ buyers, and $N$ samples. The cost function is general unless otherwise noted.}\label{tab:results_auctions}
\end{table}%

\subsection{Related research}\label{sec:related}
\subsubsection{Sample-based mechanism design}
Sample-based mechanism design was introduced in the context of \emph{automated mechanism design}
(AMD), where the goal is to design algorithms that take as input information about a set of buyers
and return a mechanism that maximizes an objective such as revenue~\citep{Conitzer02:Mechanism,Sandholm03:Automated,Conitzer04:Self}. The input information about the buyers in early AMD was an explicit description of the distribution over their valuations. Later, sample-based mechanism design was introduced where the input is a set of samples from this distribution~\citep{Likhodedov04:Boosting,Likhodedov05:Approximating,Sandholm15:Automated}. Those papers also introduced the idea of searching for a high-revenue mechanism in a parameterized space where any parameter vector yields a mechanism that satisfies the individual rationality and incentive-compatibility constraints. They did not provide generalization guarantees.

\citet{Balcan05:Mechanism, Balcan08:Reducing} were the first to study the connection between learning theory and revenue maximization. They showed how to use an algorithm $\pazocal{A}$ that returns a high-revenue, manipulable mechanism in order to find a high-revenue, incentive-compatible mechanism. They study settings with unrestricted supply, whereas we primarily focus on settings with limited supply.
	
	More recent research has provided generalization guarantees when there is limited supply, with a particular focus on single-parameter settings~\citep{Alon17:Submultiplicative, Elkind07:Designing,Cole14:Sample, Huang15:Making, Mohri14:Learning, Morgenstern15:Pseudo, Roughgarden15:Ironing, Bubeck17:Online, Chawla14:Mechanism}.  \citet{Devanur16:Sample,Gonczarowski17:Efficient,Guo19:Settling}, and \citet{Hartline16:Non} provide computationally efficient algorithms for learning nearly-optimal single-item auctions in various settings. In contrast, we study multi-parameter settings. Our bounds do not apply to the state-of-the-art in this direction by \citet{Guo19:Settling} because their approach does not involve optimizing over a mechanism class with continuously tunable parameters.

\citet{Balcan14:Learning} drew on classic tools from learning theory to provide algorithms and generalization guarantees for the related problem of learning agents' preferences. Their analysis made connections to the concept of \emph{generalized linear functions} from the structured prediction literature~\citep{Collins00:Discriminative}. Their algorithms predict the future purchases of utility-maximizing agents.

\citet{Morgenstern16:Learning} later used this concept of generalized linear functions to provide sample complexity guarantees for multi-item revenue maximization.
They provide a technique for bounding a mechanism class's pseudo-dimension that requires two steps, described at a high level here and in detail in Appendix~\ref{APP:STRUCTURED}. First, one must show that for any mechanism in the class, its allocation function is a $d$-dimensional linear function for some $d \in \Z$. Next, fixing a set of samples and an allocation per sample, one must bound the pseudo-dimension of the set of revenue functions across all mechanisms that induce those allocations.
	
The guarantees presented in this paper offer several advantages over Morgenstern and Roughgarden's approach. First, our main theorem depends on a structural property---the piecewise-linear form of the revenue function---that is not defined in terms of any learning theory concept (such as generalized linear functions or pseudo-dimension) and thus can be more readily applied.
Moreover, in several cases, \citet{Morgenstern16:Learning} proved loose guarantees using structured prediction; in their appendix, they used a first-principles approach to prove stronger guarantees.
Their structured prediction proof technique requires them to bound the total number of allocations a mechanism class can induce on a set of samples. Their bound is a bit loose, and we are able to tighten it using the techniques we develop in this paper, as we detail in Appendix~\ref{APP:STRUCTURED}.
By combining our analysis techniques with tools from structured prediction, we are able to match the tighter bounds  that \citet{Morgenstern16:Learning}, which answer an open question they posed. Finally, we apply our guarantees to a wide variety of mechanism classes, both simple and complex, whereas \citet{Morgenstern16:Learning} applied their guarantees to three mechanism classes that are ``simple'' by design: item-pricing mechanisms, grand-bundle-pricing mechanisms (where the grand bundle is sold as a single unit), and second-price item auctions.

 \citet{Syrgkanis17:Sample} also suggests a general technique for providing generalization guarantees which he applies to several ``simple'' mechanism classes: the same three as \citet{Morgenstern16:Learning} as well as single-item \emph{$t$-level auctions}~\citep{Morgenstern15:Pseudo}. His generalization guarantees apply only to \emph{empirical revenue maximization} algorithms, which return the mechanism in a class that maximizes average revenue over the samples. This is in contrast to our bounds (and those by \citet{Morgenstern16:Learning}) which apply uniformly to every mechanism in a given class. This is crucial when empirical revenue maximization is not computationally feasible. Another advantage of our bounds is that they grow logarithmically in $\frac{1}{\delta}$ where $\delta$ is the probability that the bound fails to hold (as do those by \citet{Morgenstern16:Learning}). In contrast, the bounds by \citet{Syrgkanis17:Sample} grow linearly in $\frac{1}{\delta}$.

In Section~\ref{SEC:COMPARISON}, we provide more details on how our results compare to those by \citet{Morgenstern16:Learning} and \citet{Syrgkanis17:Sample}, as well as a detailed comparison of our results to other papers on the sample complexity of multi-item revenue maximization~\citep{Balcan16:Sample, Medina17:Revenue, Cai17:Learning,Gonczarowski18:Sample}.

\subsubsection{Dynamic mechanism design} Dynamic pricing is a  similar but distinct problem from ours where prices are adjusted over a finite time horizon and the consumer demand function is unknown~\citep[e.g.,][all of whom study single-item settings]{Araman09:Dynamic,Besbes09:Dynamic,Broder12:Dynamic}. The goal is typically to minimize regret (the difference between the cumulative profit of the best prices in hindsight and that of the chosen prices).

\subsubsection{Approximation guarantees}\label{sec:apx}
Many mechanisms we analyze can guarantee approximately-optimal revenue.

\paragraph{Item-pricing mechanisms.} For a single unit-demand buyer with a bounded value distribution, item-pricing mechanisms can yield a constant fraction of optimal revenue~\citep{Chawla07:Algorithmic}. Moreover, given multiple unit-demand buyers and constraints on which allocations are feasible, item-pricing mechanisms provide a constant-factor approximation~ \citep{Chawla10:Multi}.

For an additive buyer with independent values, item-pricing mechanisms provide a $O(\log^2 m)$ fraction of optimal revenue~\citep{Hart12:Approximate}, later improved to $O(\log m)$~\citep{Li13:Revenue}.
The better of an item-pricing mechanism or selling the grand bundle as a single unit provides a constant-factor approximation in this setting~\citep{Babaioff14:Simple} and generalizations thereof~\citep{Bateni15:Revenue,Eden21:Simple}\footnote{For multiple additive buyers, a VCG mechanism with bidder entries fees achieves a constant-factor approximation~\citep{Yao14:n}.}. Similarly, for a subadditive buyer, the better of an item-pricing mechanism and a more general bundling mechanism provides a constant-factor approximation~\citep{Rubinstein15:Simple}.

For multiple buyers with XOS values over independent items, a non-anonymous item-pricing mechanism or an anonymous item-pricing mechanism with an entry fee is a constant-factor approximation~\citep{Cai17:Simple}. For  subadditive buyers, the approximation is $O(\log\log m)$~\citep{Dutting20:Log} and a single random price is a $O(2^{\sqrt{\log m \log\log m}})$ approximation~\citep{Balcan08:Item}.

\paragraph{Two-part tariffs.} For buyers with additive values up to a matroid feasibility constraint, a sequential variant of two-part tariffs provides a constant-factor approximation~\citep{Chawla16:Mechanism}.

\paragraph{Lotteries.}
For a single additive buyer with independent values, \citet{Babaioff17:Menu} proved that a lottery menu of length $(\log m / \epsilon)^{O(m)}$ is a $(1-\epsilon)$-factor approximation. For a single unit-demand buyer with independent item values, \citet{Kothari19:Approximation} introduce the notion of \emph{symmetric menu complexity}, which is the number of menu entries up to permutations of the items. A quasi-polynomial symmetric menu complexity suffices to guarantee a $(1-\epsilon)$ approximation.

\section{Preliminaries and notation}\label{sec:prelim}

We study the problem of selling $m$ items to $n$ buyers. We denote a bundle of items as a quantity vector $\vec{q} \in \Z_{\geq 0}^m$. The number of units of item $i$ in the bundle is $q[i]$. The bundle consisting of only one copy of the $i^{th}$ item is denoted by the standard basis vector $\vec{e}_i$, where $e_i[i] = 1$ and $e_i[j] = 0$ for all $j \not= i$. Each buyer $j \in [n]$ has a valuation function $v_j$ over bundles of items. We denote an allocation as $Q = \left(\vec{q}_1, \dots, \vec{q}_n\right)$ where $\vec{q}_j$ is the bundle that buyer $j$ receives. The cost to produce $\vec{q}$ is $c\left(\vec{q}\right)$ and the cost to produce the allocation $Q$ is $c\left(Q\right)$.
Suppose there are $\kappa_i$ units available of item $i$. Let $K = \prod_{i = 1}^m \left(\kappa_i+1\right)$. We use $\vec{v}_j = \left(v_j\left(\vec{q}_1\right), \dots, v_j\left(\vec{q}_K\right)\right)$ to denote buyer $j$'s values for all of the $K$ bundles and we use $\vec{v} = \left(\vec{v}_1, \dots, \vec{v}_n\right)$ to denote a vector of buyer values. We use the notation $\cX$ to denote the set of all valuation vectors $\vec{v}$. Additive buyers have values $v_j\left(\vec{q}\right) = \sum_{i = 1}^m q[i] v_j\left(\vec{e}_i\right)$ and unit-demand buyers have values $v_j\left(\vec{q}\right) = \max_{i : q[i] \geq 1} v_j\left(\vec{e}_i\right)$. The mechanisms  we study are dominant strategy incentive compatible, so we assume that the bids equal the buyers' valuations.

There is an unknown distribution $\pazocal{D}$ over buyers' values. 
The notation $\profit_M\left(\vec{v}\right)$ denotes the profit of a mechanism $M$ on the valuation vector $\vec{v}$. We use the notation $\profit_{\dist}\left(M\right) = \E_{\vec{v} \sim \dist}\left[\profit_M\left(\vec{v}\right)\right]$ and for a set of samples $\sample$, we use the notation \[\profit_{\sample}\left(M\right) = \frac{1}{|\sample|}\sum_{\vec{v} \in \sample}\profit_M\left(\vec{v}\right).\]

We study real-valued functions parameterized by vectors $\vec{p}$ in $\R^d$, denoted as $f_{\vec{p}}:\domain \to \R.$ For a fixed $\vec{v} \in \domain$, we often consider $f_{\vec{p}}\left(\vec{v}\right)$ as a function of its parameters, which we denote as $f_{\vec{v}}\left(\vec{p}\right)$.

\section{Generalization guarantees}\label{SEC:MAIN}
We provide generalization bounds for a variety of mechanism classes. These guarantees bound the difference between the expected profit and average empirical profit of any mechanism in the class.

\begin{definition}\label{def:gen_guar}
A \emph{generalization guarantee} for a mechanism class $\pazocal{M}$ is a function $\epsilon_{\cM} : \Z_{\geq 1} \times (0,1) \to \R_{\geq 0}$ defined such that for any $\delta \in (0,1)$, any $N \in \Z_{\geq 1}$, and any distribution $\dist$ over buyers' values, with probability at least $1-\delta$ over the draw of a set $\sample \sim \dist^N$, for any $M \in \pazocal{M}$, the difference between the average profit of $M$ over $\sample$ and the expected profit of $M$ over $\dist$ is at most $\epsilon_{\pazocal{M}}(N, \delta)$:
\[\Pr_{\sample \sim \dist^N} \left[\exists M \in \cM \text{ such that } \left|\frac{1}{N}\sum_{\vec{v} \in \sample} \profit_M\left(\vec{v}\right) - \E_{\vec{v} \sim \dist}\left[\profit_M\left(\vec{v}\right)\right]\right| > \epsilon_{\cM}(N, \delta) \right] < \delta.\]\end{definition}

Generalization guarantees allow the mechanism designer to relate the expected profit of a mechanism in $\pazocal{M}$ which achieves maximum average profit over the set of samples to the expected profit of an optimal mechanism in $\pazocal{M}$. We summarize this connection in the following remark.

\begin{remark}\label{rem:opt}
For a set of samples $\sample \sim \dist^N$, let $\hat{M} = \argmax_{M \in \cM}\left\{\sum_{\vec{v} \in \sample} \profit_M(\vec{v})\right\}$ maximize average profit over $\sample$ and let $M^* = \argmax_{M \in \cM}\left\{\E_{\vec{v} \sim \dist}\left[\profit_M(\vec{v})\right]\right\}$ maximize expected profit. Then $\Pr_{\sample \sim \dist^N} \left[\E_{\vec{v} \sim \dist}\left[\profit_{M^*}\left(\vec{v}\right) - \profit_{\hat{M}}\left(\vec{v}\right)\right] > 2\epsilon_{\cM}\left(N, \delta\right)\right] < \delta.$
\end{remark}

Similar bounds also hold for mechanisms with approximately optimal average profit over the samples (see Corollaries~\ref{cor:add_apx} and \ref{cor:mult_apx} in Appendix~\ref{APP:MAIN}).

\subsection{General structure for sample-based mechanism design}

Our general theorem uses structure shared by a variety of mechanism classes to characterize the function $\epsilon_{\pazocal{M}}\left(N, \delta\right)$.
Our results apply broadly to \emph{parameterized} sets $\mclass$ of mechanisms where every mechanism in $\mclass$ is defined by a vector $\vec{p}\in \R^d$, such as a vector of prices. Our guarantees apply to mechanism classes where for every valuation  $\vec{v} \in \domain$, the profit as a function of the parameters $\vec{p}$, denoted $\profit_{\vec{v}}\left(\vec{p}\right)$, is piecewise linear. We illustrate this property via several simple examples.

\begin{example}\label{ex:2PT}
In a \emph{two-part tariff,} there are multiple units (i.e., copies) of an item for sale. The seller sets an \emph{upfront fee} $p_1$ and a \emph{price per unit} $p_2$. Here, we consider the simple case where there is a single buyer. If the buyer buys $t \geq 1$ units, he pays $p_1 + p_2\cdot t$. Two-part tariffs have been studied extensively \citep{Oi71:Disneyland, Feldstein72:Equity, Wilson93:Nonlinear} and are prevalent throughout daily life. For example, health clubs often require an upfront membership fee plus a fee per month. Amusement parks often require an entrance fee with an additional payment per ride.
In many cities, purchasing a public transportation card requires an upfront fee and an additional cost per ride.  \citet{Balcan20:Efficient} showed how to learn two-part tariffs that maximize average revenue over a training set.

Suppose there are $\kappa$ units of the item for sale. The buyer will buy  $t \in [\kappa]$ units so long as $v_1\left(t\right) - \left(p_1 + p_2 \cdot t\right) >  v_1\left(t'\right) - \left(p_1 + p_2 \cdot t'\right)$ for all $t' \not= t$ and $v_1\left(t\right) - \left(p_1 + p_2 \cdot t\right) > 0$. Therefore, there are at most ${\kappa + 1 \choose 2}$ hyperplanes splitting $\R^2$ into regions such that within any one region, the number of units bought is fixed, in which case profit is linear in $p_1$ and $p_2$.
\begin{figure}
	\includegraphics{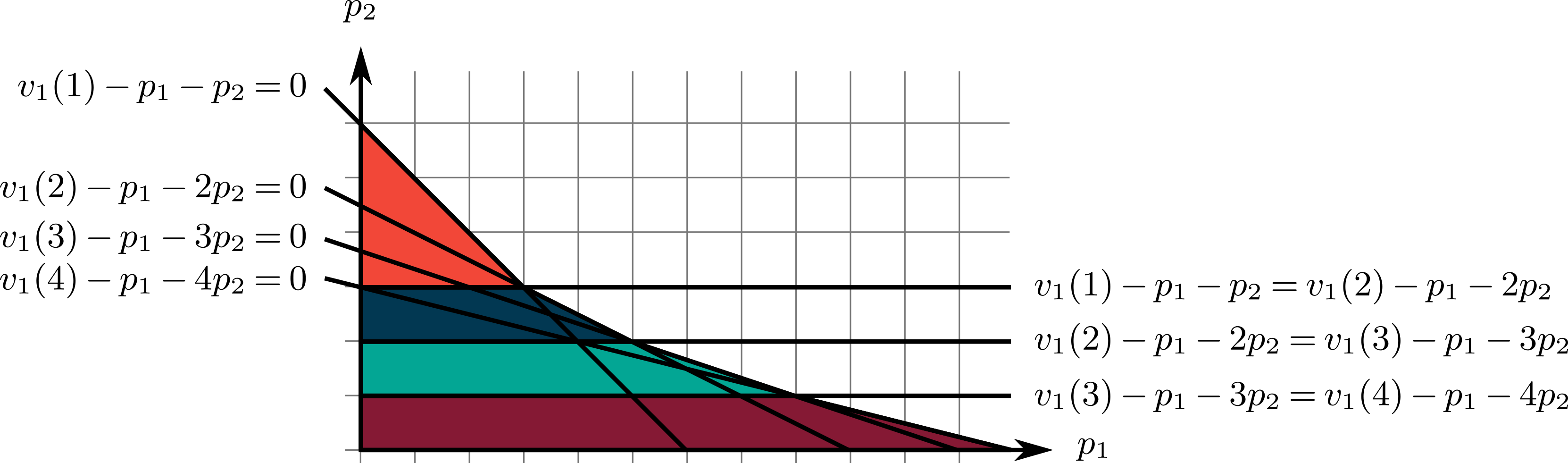}\centering
	\caption{Partition of the two-part tariff parameter space into piecewise-linear regions. There are four units for sale and one buyer with values $v_1(1) = 6$, $v_1(2) = 9$,  $v_1(3) = 11$, and $v_1(4) = 12$. The buyer will buy one unit in the top orange region where $v_1(1) - p_1 - p_2 > v_1(i) - p_1 - i \cdot p_2$ for all $i \in \{2, 3, 4\}$ and $v_1(1) - p_1 - p_2 > 0$. The buyer will buy two units in the second-to-the-top blue region, three units in the second-to-the-bottom green region, and four units in the bottom red region.
	}
	\label{fig:2PT_detailed}
\end{figure}
See Figure~\ref{fig:2PT_detailed} for an illustration.
\end{example}

\begin{example}\label{ex:item_pricing}
Under an \emph{item-pricing mechanism}, there are multiple items, multiple buyers, and a single unit of each item for sale. 
Under anonymous prices, the seller sets a price $p_i$ per item $i$. There is an arbitrary ordering on the buyers such that the first buyer buys the bundle that maximizes his utility, then the next buyer buys the bundle of remaining items that maximizes his utility, and so on.
Buyer $j$ will prefer bundle $\vec{q}_1 \in \{0,1\}^m$ over $\vec{q}_2$ if $v_j(\vec{q}_1) - \sum_{i: q_1[i] = 1} p_i > v_j(\vec{q}_2) - \sum_{i: q_2[i] = 1} p_i$, so his preference ordering over bundles is determined by these ${2^m \choose 2}$ hyperplanes. Once the buyers' preference orderings are fixed, the bundles they buy are fixed. In any region of the price space where the purchased bundles are fixed, profit is a linear in the prices.
\begin{figure}
	\includegraphics{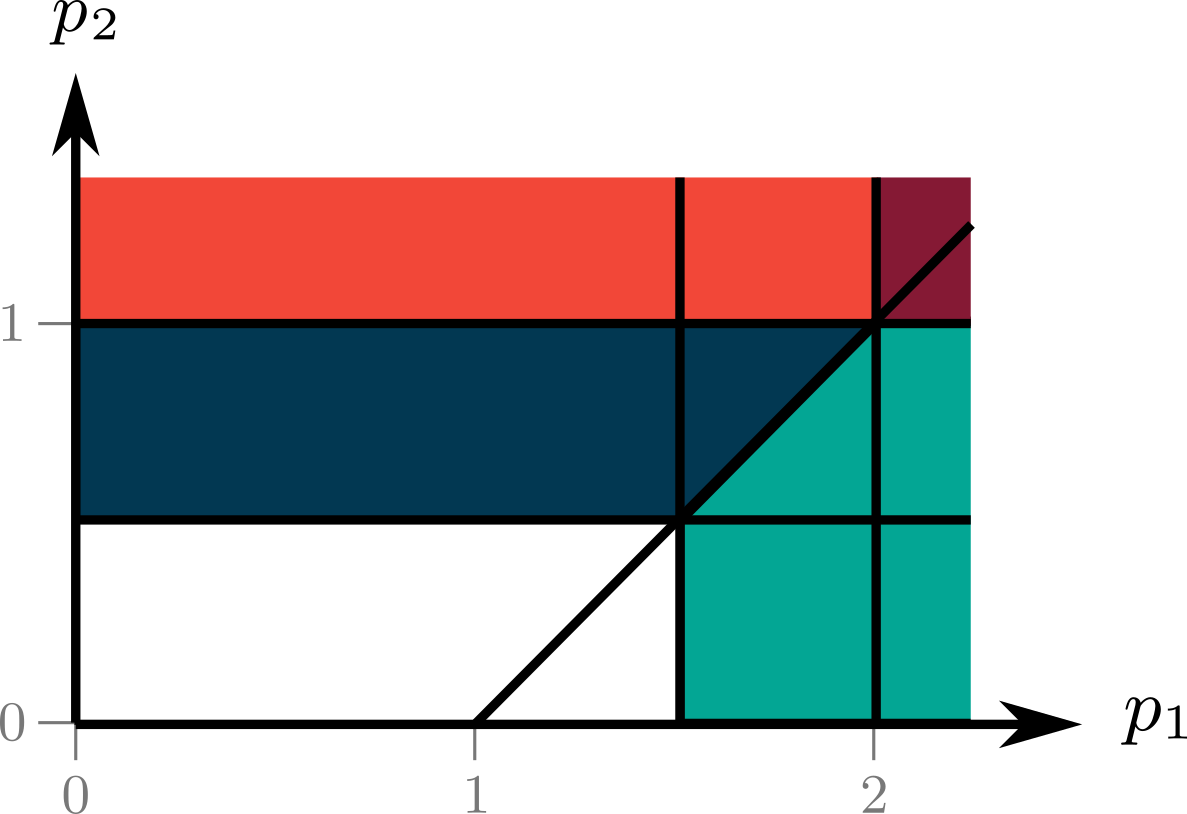}\centering
	\caption{Partition of the \emph{item-pricing} parameter space into piecewise-linear regions when there are two buyers, two items, and buyer 1 comes before buyer 2 in the ordering. Buyer 1's value for item 1 is $v_1(1,0) = 2$, her value for item 2 is $v_1(0, 1) = 1$, and her value for both items is $v_1(1,1) = 2.5$. Buyer 2's values are $v_2(1, 0) = 0, v_2(0, 1) = 1$, and $v_2(1, 1) = 1$. In the orange region, buyer 1 will buy item 1 because $v_1(1, 0) - p_1 > v_1(0,1) - p_2$, $v_1(1, 0) - p_1 > v_1(1,1) - (p_1 + p_2)$, and $v_1(1, 0) - p_1 > 0$. Buyer 2 will not buy item 2 because $v_2(0, 1) - p_2 < 0$. In the red region, neither buyer will buy any item. In the blue region, buyer 1 will buy item 1 and buyer 2 will buy item 2. In the green region, buyer 1 will buy item 2 and buyer 2 will not buy anything. Finally, in the white region, buyer 1 will buy both items.
	}
	\label{fig:pricing}
\end{figure}
See Figure~\ref{fig:pricing} for an illustration.
\end{example}

We analyze the ``complexity'' of the partition splitting $\R^d$ into regions where $\profit_{\vec{v}}\left(\vec{p}\right)$ is linear.
\begin{definition}[$\left(d,t\right)$-delineable]\label{def:delineable} A mechanism class $\mclass$ is \emph{$\left(d,t\right)$-delineable} if:
\begin{enumerate}
\item The class $\mclass$ consists of mechanisms parameterized by vectors $\vec{p}$ from a set $\pspace \subseteq \R^d$; and
\item For any valuation vector $\vec{v} \in \domain$, there is a set $\hyp$ of $t$ hyperplanes such that for any connected component $\pspace'$ of $\pspace \setminus \hyp$, $\profit_{\vec{v}}\left(\vec{p}\right)$ is linear over $\pspace'.$ (As is standard, $\pspace \setminus \hyp$ indicates set removal.)
\end{enumerate}
\end{definition}
\begin{figure}
	\centering
	\begin{subfigure}{0.21\textwidth}
		\includegraphics{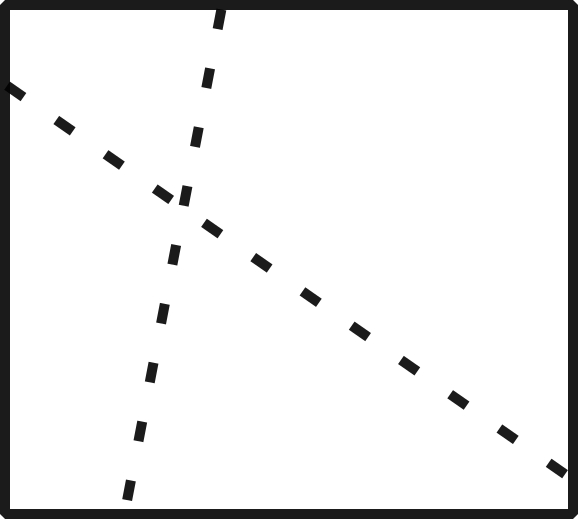} \centering
		\caption{A partition by hyperplanes.}
		\label{fig:overlay1}
	\end{subfigure}\qquad
	\begin{subfigure}{0.21\textwidth}
		\includegraphics{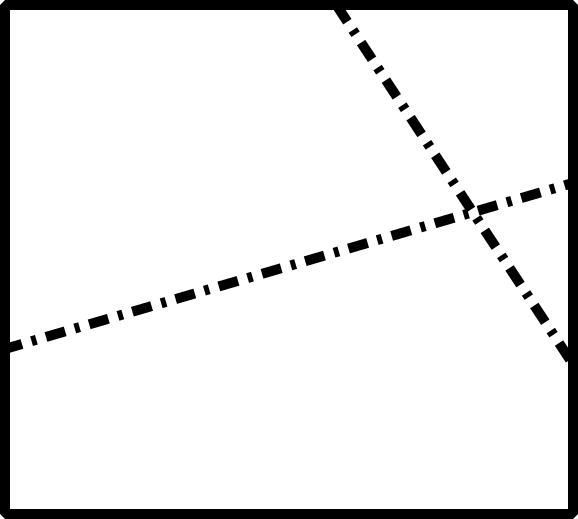}\centering
		\caption{Another partition by hyperplanes.}
		\label{fig:overlay2}
	\end{subfigure}\qquad
	\begin{subfigure}{0.21\textwidth}
		\includegraphics{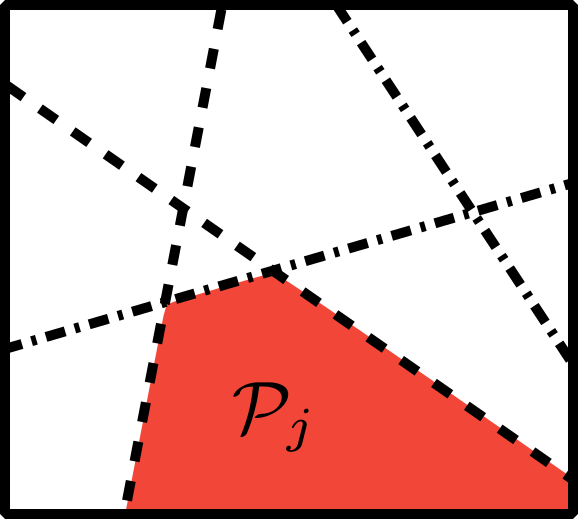}\centering
		\caption{Overlay of partitions (a) and (b).}
		\label{fig:overlay3}
	\end{subfigure}\qquad
	\begin{subfigure}{0.21\textwidth}
		\includegraphics{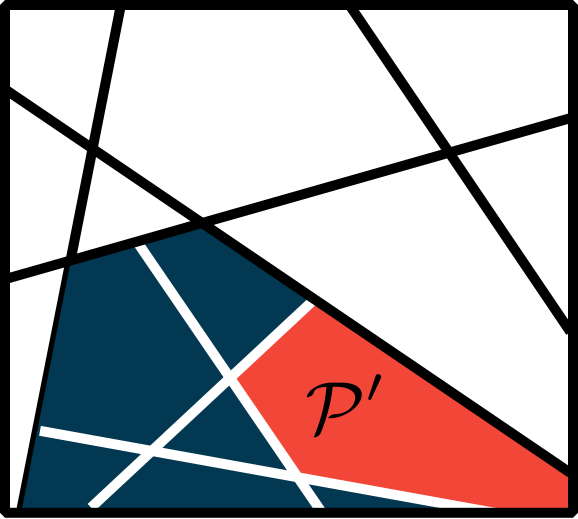}\centering
		\caption{A further subdivision of each region.}
		\label{fig:subpartition}
	\end{subfigure}
	\caption{Illustrations of the proof of Lemma~\ref{lem:main_pdim}.}
	\label{fig:overlay}
\end{figure}
For example, in Figures~\ref{fig:overlay1} and \ref{fig:overlay2}, there are four connected components.
We relate delineability to the mechanism class's intrinsic complexity using \emph{pseudo-dimension}.
\subsection{Pseudo-dimension}

Pseudo-dimension is a well-studied tool used to measure the complexity of a function class.
Pseudo-dimension captures the following intuition: functions in a ``complex'' class should be able to fit complex patterns.
We first introduce the notion of \emph{shattering} for general function classes.

\begin{definition}
Let $\fclass$ be a set of functions $f:\pazocal{A} \to \R$ with an abstract domain $\pazocal{A}$. We say that $z^{\left(1\right)}, \dots, z^{\left(N\right)} \in \R$ \emph{witness the shattering} of $\sample= \left\{x^{\left(1\right)}, \dots, x^{\left(N\right)}\right\} \subseteq \cA$ by $\pazocal{F}$ if for all $T \subseteq \sample$, there is a function $f_T \in \pazocal{F}$ such that for all $x^{\left(i\right)} \in T$, $f_T\left(x^{\left(i\right)}\right) \leq z^{\left(i\right)}$ and for all $x^{\left(i\right)} \not\in T$, $f_T\left(x^{\left(i\right)}\right) > z^{\left(i\right)}$.
\end{definition}

Figure~\ref{fig:shattering} in Appendix~\ref{APP:MAIN} provides a visualization. The larger the set a function class can shatter, the more complex that function class is, an intuition formalized by pseudo-dimension.

\begin{definition}[\citet{Pollard84:Convergence}]
Let $\fclass$ be a set of functions $f:\pazocal{A} \to \R$ and let $\sample \subseteq \pazocal{A}$ be the largest set that can be shattered by $\fclass$. The \emph{pseudo-dimension} of $\fclass$, denoted $\pdim(\fclass)$, is  $|\sample|$.
\end{definition}

In the language of mechanism design, let $\sample= \left\{\vec{v}^{\left(1\right)}, \dots, \vec{v}^{\left(N\right)}\right\}$ be a subset of $\domain$. We say that $z^{\left(1\right)}, \dots, z^{\left(N\right)} \in \R$ \emph{witness} the shattering of $\sample$ by $\pazocal{M}$ if for all $T \subseteq \sample$, there is a mechanism $M_T \in \pazocal{M}$ such that for all $\vec{v}^{\left(i\right)} \in T$, $\profit_{M_T}\left(\vec{v}^{\left(i\right)}\right) \leq z^{\left(i\right)}$ and for all $\vec{v}^{\left(i\right)} \not\in T$, $\profit_{M_T}\left(\vec{v}^{\left(i\right)}\right) > z^{\left(i\right)}$. The pseudo-dimension of $\pazocal{M}$, denoted $\pdim\left(\mclass\right)$, is the size of the largest set that is shatterable by $\pazocal{M}$.

\citet{Pollard84:Convergence} and \citet{Dudley87:Universal} provide generalization guarantees in terms of pseudo-dimension, which we describe below in the language of mechanism design.
\begin{theorem}\label{thm:pdim}
For any mechanism class $\mclass$, let $U$ be the maximum profit of any mechanism in $\mclass$ over the support of $\dist$. There is a generalization guarantee $\epsilon_{\cM} : \Z_{\geq 1} \times (0,1) \to \R_{\geq 0}$ defined such that \[\epsilon_{\pazocal{M}}\left(N, \delta\right) = 120U\sqrt{\frac{\pdim(\mclass)}{N}} + 4U \sqrt{\frac{2\ln(4/\delta)}{N}}.\]
\end{theorem}

\subsection{General theorem for sample-based mechanism design}

In the following theorem, which is our main theorem, we relate pseudo-dimension to delineability.

\begin{theorem}\label{thm:main_pdim}
Let $\mclass$ be a $\left(d, t\right)$-delineable mechanism class. Given a distribution $\dist$ over buyers' values, let $U$ be the maximum profit of any mechanism in $\mclass$ over the support of $\dist$. Then \[\epsilon_{\pazocal{M}}\left(N, \delta\right) = 120U\sqrt{\frac{9d \log (4dt)}{N}} + 4U \sqrt{\frac{2\ln(4/\delta)}{N}}\] is a generalization guarantee for $\mclass$.
\end{theorem}

\begin{proof}
This theorem follows directly from the following lemma.
\end{proof}

\begin{lemma}\label{lem:main_pdim}
If $\mclass$ is a mechanism class that is $\left(d, t\right)$-delineable, then $\pdim(\mclass) \leq 9d \log (4dt)$.
\end{lemma}

\begin{proof}
For any set $\sample = \left\{\vec{v}^{(1)}, \dots, \vec{v}^{(N)}\right\}$ of valuation vectors and real values $z^{(1)}, \dots, z^{(N)} \in \R$, we show that there is a partitioning of the parameter space into at most  $dN^d\cdot d(Nt)^d$ regions such that for all $\vec{p}$ in any one region and all $\vec{v}^{(i)}$, either $\profit_{\vec{v}^{(i)}}\left(\vec{p}\right) \leq z^{(i)}$ or $\profit_{\vec{v}^{(i)}}\left(\vec{p}\right) > z^{(i)}$. We will then use this fact to bound $\pdim(\mclass)$.
To this end, let $\hyp^{(i)}$ be the set of $t$ hyperplanes such
that for any connected component $\pspace'$ of $\pspace \setminus \hyp^{(i)}$, $\profit_{\vec{v}^{(i)}}\left(\vec{p}\right)$ is linear over $\pspace'.$ Let $\pspace_1, \dots, \pspace_{\tau}$ be the connected components
of $\pspace \setminus \left(\bigcup_{i = 1}^N \hyp^{(i)}\right)$. For each set $\pspace_j$ and each $i \in [N]$, $\pspace_j$ is contained in a single connected component of $\pspace \setminus \hyp^{(i)}$, which means that $\profit_{\vec{v}^{(i)}}\left(\vec{p}\right)$ is linear over $\pspace_j.$
(See Figures~\ref{fig:overlay1}-\ref{fig:overlay3} for illustrations.)
Since $\left|\hyp^{(i)}\right|\leq t$ for all $i \in [N]$, $\tau < d(Nt)^d$~\citep[][Theorem 1]{Buck43:Partition}.

For any region $\pspace_j$ and $\vec{v}^{(i)} \in \sample$,
let  $\vec{a}_j^{(i)} \in \R^d$ and $b_j^{(i)} \in \R$ be defined
such that $\profit_{\vec{v}^{(i)}}\left(\vec{p}\right) = \vec{a}_j^{(i)} \cdot \vec{p} + b_j^{(i)}$ for all $\vec{p} \in \pspace_j$. On one side of the
hyperplane $\vec{a}_j^{(i)} \cdot \vec{p} + b_j^{(i)} = z^{(i)}$, $\profit_{\vec{v}^{(i)}}\left(\vec{p}\right) \leq z^{(i)}$ and on the other
side, $\profit_{\vec{v}^{(i)}}\left(\vec{p}\right) > z^{(i)}$.
Let $\hyp_{\pspace_j}$ be all $N$ hyperplanes for all $N$ samples, i.e., $\hyp_{\pspace_j} = \left\{\vec{a}_j^{(i)} \cdot \vec{p} + b_j^{(i)} = z^{(i)} : i \in [N]\right\}.$ In any
connected component $\pspace'$ of $\pspace_j \setminus \hyp_{\pspace_j}$ (illustrated in Figure~\ref{fig:subpartition}), for all $i \in [N]$, $\profit_{\vec{v}^{(i)}}\left(\vec{p}\right)$ is either greater than $z^{(i)}$ or
less than $z^{(i)}$  for all $\vec{p} \in \pspace'$. The number of
connected components of $\pspace_j \setminus \hyp_{\pspace_j}$ is at most $dN^d$.
Thus, the total number of regions where
for all $i \in [N]$, $\profit_{\vec{v}^{(i)}}\left(\vec{p}\right)$ is either greater
than $z^{(i)}$ or less than $z^{(i)}$ is at most $dN^d\cdot d(Nt)^d$.

We now use this fact to bound $\pdim(\mclass)$. Suppose $\pdim(\mclass) = \bar{N}$, so there is a set \[\{\vec{v}^{(1)}, \dots, \vec{v}^{(\bar{N})}\}\] that is shattered by $\mclass$ with witnesses $z^{(1)}, \dots, z^{(\bar{N})} \in \R$. For any $T \subseteq [\bar{N}]$, there is a parameter vector $\vec{p}_T \in \pspace$ such that $\profit_{\vec{p}_T}\left(\vec{v}^{(i)}\right) \geq z^{(i)}$ if and only if $i\in T$. Let $\pspace^* = \left\{\vec{p}_T : T \subseteq [\bar{N}]\right\}$. There are $k \leq d\bar{N}^d\cdot d(\bar{N}t)^d$ regions $\cP_1, \dots, \cP_k$ where for each region $\cP_j$ and each $i \in [\bar{N}]$, either $\profit_{\vec{v}^{(i)}}\left(\vec{p}\right)\geq z^{(i)}$ for all $\vec{p} \in \cP_j$ or $\profit_{\vec{v}^{(i)}}\left(\vec{p}\right) < z^{(i)}$.
At most one
vector in $\pspace^*$ can come from any one region. This means that $|\pspace^*| = 2^{\bar{N}} < d\bar{N}^d\cdot d(\bar{N}t)^d$. The result follows from Lemma~\ref{lem:log_ineq}.
\end{proof} 

\subsection{Delineable mechanism classes}\label{sec:delineable}
We now show that a diverse array of mechanism classes are delineable, so we can apply Theorem~\ref{thm:main_pdim}.
We warm up with Examples~\ref{ex:2PT} and \ref{ex:item_pricing}, which imply the following lemmas.

\begin{lemma}\label{lem:easy2part}
The class of two-part tariffs for one buyer and $\kappa$ units of an item is $\left(2,{\kappa + 1 \choose 2}\right)$-delineable.
\end{lemma}

\begin{lemma}
The class of anonymous item-pricing mechanisms is $\left(m, n{2^m \choose 2}\right)$-delineable.
\end{lemma}

\subsubsection{Non-linear pricing mechanisms.}\label{sec:non_linear}
Non-linear pricing mechanisms are used to sell multiple units of a set of items. We make the following natural assumption which says that as the number of units in an allocation grows, the cost will eventually exceed the buyers' welfare. 
\begin{assumption}\label{assumption:unit_cap}
There is a cap $\kappa_i \in \Z$ per item $i$ such that it costs more to produce $\kappa_i$ units of item $i$ than the buyers will pay.
In other words, for all $\vec{v}$ in the support of $\dist$ and all allocations $Q = \left(\vec{q}_1, \dots, \vec{q}_n\right)$, if there exists an item $i$ such that $\sum_{j = 1}^n q_j[i] > \kappa_i$, then $\sum_{j = 1}^n v_j\left(\vec{q}_j\right) - c\left(Q\right) < 0$.
\end{assumption}

\paragraph{Menus of two-part tariffs.} Menus of two-part tariffs are a generalization of Example~\ref{ex:2PT}. The seller offers the buyers $\ell$ different two-part tariffs and each buyer chooses the tariff and number of units that maximizes his utility. For example, consumers often choose among various membership tiers---typically with a larger upfront fee and lower future payments---for health clubs,  wholesale  stores,  amusement  parks, credit cards, and cellphone plans.
Under non-anonymous prices, let $\left(p_{1, j}^{(1)}, p_{2, j}^{(1)}\right), \dots, \left(p_{1, j}^{(\ell)}, p_{2, j}^{(\ell)}\right)$ be the menu of two-part tariffs that the seller offers to buyer $j$. Here, $p_{1,j}^{(i)}$ is the upfront fee of the $i^{th}$ tariff and $p_{2,j}^{(i)}$ is the price per unit. Under anonymous prices, $p_{1,1}^{(i)} = \cdots = p_{1,n}^{(i)}$ and $p_{2,1}^{(i)} = \cdots = p_{2,n}^{(i)}$. Each buyer  chooses the tariff $t_j \in [\ell]$ and the number of units $q_j \geq 1$ maximizing his utility, and pays $p_{1,j}^{(t_j)} + p_{2,j}^{(t_j)} \cdot q_j$. In this context, an allocation is  a vector $Q = \left(q_1, \dots, q_n\right)$ where $q_j \in \Z_{\geq 0}$ is the number of units buyer $j$ buys.

 We make the natural assumption that the seller will not choose prices that result in negative profit. In other words, he will select a \emph{profit non-negative menu of two-part tariffs}, formalized below.

\begin{definition}\label{def:PNN_2PT}
For anonymous prices (respectively, non-anonymous), let $\pspace \subseteq \R^{2\ell}$ (respectively, $\pspace' \subseteq \R^{2n\ell}$) be the set of prices where no matter which tariff each buyer chooses and no matter how many units he buys, the seller will obtain non-negative profit. In other words, for each buyer $j \in [n]$, each tariff $t_j \in [\ell]$, and each allocation $Q = \left(q_1, \dots, q_n\right)$, $\sum_{j = 1}^n p_{1,j}^{(t_j)} \cdot \textbf{1}_{\{q_j \geq 1\}} + p_{2,j}^{(t_j)} \cdot q_j - c\left(Q\right) \geq 0.$ The set of \emph{profit non-negative menus of two-part tariffs} is defined by parameters in $\pspace$ (resp., $\pspace'$).
\end{definition}

Under Assumption~\ref{assumption:unit_cap}, no matter which parameters the seller chooses in $\pspace$ or $\pspace'$, if all buyers simultaneously choose the tariff and the number of units $(q_1, \dots, q_n)$ that maximize their utilities, then $\sum_{j = 1}^n q_j \leq \kappa$. See Lemma~\ref{lem:kappa_bnded_2pt}  for the proof. This allows us to prove the following lemma.

\begin{restatable}{lemma}{twoPart}\label{lem:2part}
Let $\pazocal{M}$ and $\pazocal{M}'$ be the classes of anonymous and non-anonymous profit non-negative length-$\ell$ menus of two-part tariffs. Under Assumption~\ref{assumption:unit_cap}, $\mclass$ is $\left(2\ell, n\left(\kappa \ell\right)^2\right)$-delineable and $\mclass'$ is $\left(2n\ell, n\left(\kappa \ell\right)^2\right)$-delineable.
\end{restatable}

\paragraph{General non-linear pricing mechanisms.} We study general non-linear pricing mechanisms under Wilson's \emph{bundling interpretation} \citep{Wilson93:Nonlinear}: if the prices are anonymous, there is a price per quantity vector $\vec{q}$ denoted $p\left(\vec{q}\right)$. The buyers simultaneously choose the bundles maximizing their utilities. If the prices are non-anonymous, there is a price per vector $\vec{q}$ and buyer $j \in [n]$ denoted $p_j\left(\textbf{q}\right)$. These general non-linear pricing mechanisms include \emph{multi-part tariffs} as a special case.
Without assumptions, the parameter space infinite-dimensional since the seller could set prices for every bundle $\vec{q} \in \Z_{\geq 0}^m$. In Lemma~\ref{lem:kappa_bnded_NL}, we show that under Assumption~\ref{assumption:unit_cap}, no buyer will choose a bundle $\vec{q}$ with $q[i] > \kappa_i$ for any $i \in [m]$ if the seller chooses a \emph{profit non-negative non-linear pricing mechanism}. The definition is similar to Definition~\ref{def:PNN_2PT} and is in Appendix~\ref{APP:MAIN} (Definition~\ref{def:PNN_NL}).

\begin{restatable}{lemma}{nonlinear}\label{lem:nonlinear}
Let $\pazocal{M}$ and $\pazocal{M}'$ be the classes of anonymous and non-anonymous profit non-negative non-linear pricing mechanisms. Under Assumption~\ref{assumption:unit_cap},
$\mclass$ is $\left(K, nK^2\right)$-delineable and $\mclass'$ is $\left(nK, nK^2\right)$-delineable.
\end{restatable}

We prove polynomial bounds when prices are additive over items (Lemma~\ref{lem:nonlinear_additive}).

\subsubsection{Item-pricing mechanisms.}
We now apply Theorem~\ref{thm:main_pdim} to anonymous and non-anonymous item-pricing mechanisms. Unlike non-linear pricing, there is only one unit of each item for sale. 
Under anonymous prices, the seller sets a price per item. Under non-anonymous prices, there is a buyer-specific price per item. We make the common assumption~\citep[e.g.,][]{Feldman15:Combinatorial, Babaioff14:Simple, Cai16:Duality} that there is a fixed, arbitrary ordering on the buyers such that the first buyer arrives and buys the bundle that maximizes his utility, then the next buyer arrives and buys the bundle of remaining items that maximizes his utility, and so on.

\begin{restatable}{lemma}{itemAdd}\label{lem:item_pricing_add}
Let $\pazocal{M}$ (resp., $\pazocal{M}'$) be the class of item-pricing mechanisms with anonymous (resp., non-anonymous) prices. For additive buyers, $\mclass$ is $\left(m, m\right)$-delineable and $\mclass'$ is $\left(nm,nm\right)$-delineable.
\end{restatable}

In Appendix~\ref{APP:STRUCTURED}, we connect the hyperplane structure we investigate in this paper to the structured prediction literature in machine learning~\citep{Collins00:Discriminative}, thus strengthening our generalization bounds for item-pricing mechanisms under buyers with unit-demand and general valuations and answering an open question by \citet{Morgenstern16:Learning}.

\subsubsection{Auctions.} We now present applications of Lemma~\ref{lem:main_pdim} to auctions in single-unit settings.

\paragraph{Second price item auctions with reserves.} We study additive buyers in this setting. Under non-anonymous reserves, there is a price $p_j\left(\vec{e}_i\right)$ for each item $i$ and buyer $j$. The buyers submit bids on the items. For each item $i$, the highest bidder $j$ wins the item if her bid is above $p_j\left(\vec{e}_i\right)$. She pays the maximum of the second highest bid and $p_j\left(\vec{e}_i\right)$. Under anonymous reserves, $p_1\left(\vec{e}_i\right)  = \cdots = p_n\left(\vec{e}_i\right)$.
\begin{restatable}{lemma}{secondPrice}\label{lem:second_price}
Let $\pazocal{M}$ and $\pazocal{M}'$ be the classes of anonymous and non-anonymous second price item auctions. Then $\mclass$ is $\left(m, m\right)$-delineable and $\mclass'$ is $\left(nm, m\right)$-delineable.
\end{restatable}

In Section~\ref{SEC:COMPARISON}, we compare these results with those of prior research~\citep{Morgenstern16:Learning,Devanur16:Sample,Syrgkanis17:Sample}.

\paragraph{Mixed bundling auctions with reserve prices (MBARPs).}
MBARPs~\citep{Jehiel07:Mixed,Tang12:Mixed} are a VCG generalization.
Intuitively, the MBARP enlarges the set of agents to include the seller, whose values are defined by reserve prices. The auction boosts the social welfare of any allocation where the grand bundle is allocated and then runs the VCG over this larger set of buyers.
Formally, MBARPs are defined by a parameter $\gamma \geq 0$ and reserves $p\left(\vec{e}_1\right), \dots, p\left(\vec{e}_m\right)$. Let $\lambda$ be a function such that $\lambda\left(Q\right) = \gamma$ if some buyer receives the grand bundle under allocation $Q$ and 0 otherwise. For an allocation $Q$, let $\vec{q}_Q$ be the items not allocated. The MBARP allocation is \[\text{argmax}\left\{\sum_{j = 1}^n v_j\left(\vec{q}_j\right) + \sum_{i : q_Q[i] = 1} p\left(\vec{e}_i\right) + \lambda\left(Q\right) - c\left(Q\right)\right\}.\] The payments are defined as in the VCG mechanism (see Definition~\ref{def:MBARP} in Appendix~\ref{APP:MAIN}).

\begin{restatable}{lemma}{MBARP}\label{lem:MBARP}
Let $\pazocal{M}$ be the set of MBARPs. Then $\mclass$ is $\left(m+1, (n+1)^{2m+1}\right)$-delineable.
\end{restatable}

\emph{Mixed-bundling auctions}~\citep{Jehiel07:Mixed} are MBARPs with no reserve prices. We provide a stronger, specialized guarantee for this class in Appendix~\ref{app:MBA}.

\paragraph{Affine maximizer auctions (AMAs).} AMAs are the
only \emph{ex post} truthful mechanisms over unrestricted
value domains~\citep{Roberts79:Characterization} and  under
natural assumptions, every truthful multi-item auction
is an ``almost'' AMA, that is, an AMA for sufficiently high values~\citep{Lavi03:Towards}.\footnote{Surprisingly, even when the buyers have additive values, AMAs can generate higher revenue than running a separate Myerson auction for each item~\citep{Sandholm15:Automated}.}
An AMA is defined by a weight per buyer $w_j \in \R_{> 0}$ and a boost per allocation $\lambda\left(Q\right) \in \R_{\geq 0}$. Its allocation maximizes the weighted social welfare $\sum_{j = 1}^n w_jv_j\left(\vec{q}_j\right) + \lambda\left(Q\right) - c\left(Q\right).$ The payments have the same form as the VCG payments (see Definition~\ref{def:AMA} in Appendix~\ref{APP:MAIN}).
A virtual valuation combinational auction (VVCA) \citep{Likhodedov04:Boosting} is an AMA where each $\lambda\left(Q\right)$ is split into $n$ terms such that $\lambda\left(Q\right) = \sum_{j = 1}^n \lambda_j\left(Q\right)$ where $\lambda_j\left(Q\right) = c_{j,\vec{q}}$ for all allocations $Q$ that give buyer $j$ exactly bundle $\vec{q}$. Finally, $\lambda$-auctions~\citep{Jehiel07:Mixed} are defined such that $w_1 = \cdots = w_n = 1$.

\begin{restatable}{lemma}{AMA}\label{lem:AMA}
Let $\pazocal{M}$, $\mclass'$, and $\mclass''$ be the classes of AMAs, VVCAs, and $\lambda$-auctions, respectively. Letting $t = \left(n+1\right)^{2m+1}$, we have that $\mclass$ is $\left(2n(n+1) + (n+1)^{m+1},t\right)$-delineable, $\mclass'$ is $\left(n2^m(3 + 2n), t\right)$-delineable, and $\mclass''$ is $\left(\left(n+1\right)^m, t\right)$-delineable.
 \end{restatable}

Lemma~\ref{lem:AMA} implies that exponentially-many samples are sufficient to avoid overfitting. In Appendix~\ref{app:lower}, we prove an exponential number of samples is also necessary.

\subsubsection{Lotteries.}
Lotteries are randomized mechanisms which typically have higher  revenue than deterministic mechanisms.  We analyze a single additive buyer and generalize to unit-demand buyers and multiple buyers in Appendix~\ref{sec:additional_lotteries}.
A \emph{length-$\ell$ lottery menu} is a set $M = \{(\vec{\phi}^{(0)}, p^{(0)}), (\vec{\phi}^{(1)}, p^{(1)}), \dots, (\vec{\phi}^{(\ell)}, p^{(\ell)})\} \subseteq \R^m \times \R$, where $\vec{\phi}^{\left(0\right)} = \vec{0}$ and $p^{\left(0\right)}= 0$. Under the lottery $\left(\vec{\phi}^{(j)}, p^{(j)}\right)$, the buyer pays $p^{(j)}$ and receives each item $i$ with probability $\phi^{(j)}[i]$. For a buyer with values $\vec{v}$, let $\left(\vec{\phi}_{\vec{v}}, p_{\vec{v}}\right) \in M$ be the lottery that maximizes the his expected utility and let $\vec{q} \sim \phi_{\vec{v}}$ denote the allocation. The expected profit is $\profit_{M} \left(\vec{v}\right) = p_{\vec{v}} - \E_{\vec{q} \sim \vec{\phi}_{\vec{v}}}\left[c\left(\vec{q}\right)\right]].$
The challenge in bounding the pseudo-dimension of the class $\mclass$ of these lotteries is that $\E_{\vec{q} \sim \vec{\phi}_{\vec{v}}}\left[c\left(\vec{q}\right)\right]$ is not piecewise linear in $\vec{\phi}^{\left(0\right)}, \dots, \vec{\phi}^{\left(\ell\right)}$. Instead, we bound the pseudo-dimension of a related class $\mclass'$ and show that optimizing over $\mclass'$ amounts to optimizing over $\mclass$ itself. To motivate $\mclass'$, note that if $\vec{z} \sim U\left([0,1]^m\right)$, then $\Pr_{\vec{z}}[z[j] \leq \phi_{\vec{v}}[j]] = \phi_{\vec{v}}[j]$, so $\E_{\vec{q} \sim \vec{\phi}_{\vec{v}}}\left[c\left(\vec{q}\right)\right] = \E_{\vec{z}}\left[c\left(\sum_{j: z[j] < \phi_{\vec{v}}[j]} \vec{e}_j\right)\right]$. For $M \in \mclass$, we define $\profit_{M} '\left(\vec{v}, \vec{z}\right) := p_{\vec{v}} - c\left(\sum_{j: z[j] < \phi_{\vec{v}}[j]} \vec{e}_j\right)$ and $\mclass' = \left\{\profit_M' : M \in \mclass\right\}$. The class $\mclass'$ is delineable: for any $\left(\vec{v}, \vec{z}\right)$, buyer's chosen lottery and the bundle $\sum_{j: z[j] < \phi_{\vec{v}}[j]} \vec{e}_j$ are determined by hyperplanes.

\begin{restatable}{lemma}{lottery}\label{lem:lottery}
The class $\mclass'$ is $\left(\ell\left(m+1\right), \left(\ell+1\right)^2 + m\ell\right)$-delineable.
\end{restatable}

The following lemma guarantees that optimizing over $\mclass'$ amounts to optimizing over $\mclass$ itself.

\begin{restatable}{lemma}{lotteryEquiv} \label{lem:lottery_equiv}
With probability $1-\delta$ over $\left(\vec{v}^{\left(1\right)}, \vec{z}^{\left(1\right)}\right), \dots, \left(\vec{v}^{\left(N\right)}, \vec{z}^{\left(N\right)}\right) \sim \dist \times U[0,1]^{m},$ for all $M \in \mclass$, \[\left|\frac{1}{N} \sum_{i = 1}^N \profit_{M} '\left(\vec{v}^{\left(i\right)}, \vec{z}^{\left(i\right)}\right) - \E_{\vec{v} \sim \dist}[\profit_M\left(\vec{v}\right)] \right| \leq 120U\sqrt{\frac{\pdim(\mclass')}{N}} + 4U \sqrt{\frac{2\ln(4/\delta)}{N}}.\]
\end{restatable}

This section demonstrates that a wide variety of mechanism classes $\cM$ are delineable. Therefore, Theorem~\ref{thm:main_pdim} immediately implies a generalization bound $\epsilon_{\cM}(N, \delta)$ for a diverse array of mechanisms.

\section{Distribution-dependent generalization guarantees}\label{SEC:DATA}

In this section, we provide stronger results when the buyers' values are additive and drawn from \emph{item-independent} distributions, which means that for all $i_1, i_2 \in [n]$ and $j ,j' \in [m]$, buyer $i_1$'s values for items $j$ and $j'$ are independent, but her values may be correlated with buyer $i_2$'s values. We also require that the mechanism class's profit functions decompose additively. For example, under item-pricing mechanisms, the profit decomposes into the profit obtained from selling item 1, plus the profit obtained by selling item 2, and so on. Surprisingly, our bounds do not depend on the number of items and under anonymous prices, they do not depend on the number of buyers either.

To prove distribution-dependent generalization guarantees, we use Rademacher complexity~\citep{Bartlett02:Rademacher,Koltchinskii01:Rademacher}. In contrast, pseudo-dimension implies bounds that are worst-case over the distribution.
We prove that it is impossible to obtain guarantees that are independent of the number of items using pseudo-dimension alone (Theorem~\ref{thm:pdim_lower}).

\begin{definition}\label{def:DD_gen_guar}
A \emph{distribution-dependent generalization guarantee} for a mechanism class $\pazocal{M}$ and a distribution $\dist$ over buyers' values is a function $\epsilon_{\cM}^\dist : \Z_{\geq 1} \times (0,1) \to \R_{\geq 0}$ defined such that for any sample size $N \in \Z_{\geq 1}$ and any $\delta \in (0,1)$, with probability at least $1-\delta$ over the draw of a set $\sample \sim \dist^N$, for any mechanism $M$ in $\pazocal{M}$, the difference between the average profit of $M$ over $\sample$ and the expected profit of $M$ over $\dist$ is at most $\epsilon_{\pazocal{M}}^\dist(N, \delta)$. In other words,
\[\Pr_{\sample \sim \dist^N} \left[\exists M \in \cM \text{ such that } \left|\frac{1}{N}\sum_{\vec{v} \in \sample} \profit_M\left(\vec{v}\right) - \E_{\vec{v} \sim \dist}\left[\profit_M\left(\vec{v}\right)\right]\right| > \epsilon_{\cM}^{\dist}(N, \delta) \right] < \delta.\]\end{definition}
The generalization guarantee $\epsilon_{\cM}$ is worst case in that in holds for any distribution $\dist$. In contrast, the distribution-dependent bound $\epsilon_{\cM}^{\dist}$ may be much tighter when the distribution is ``well-behaved''.

We now define Rademacher complexity, which measures the ability of a class of mechanism profit functions to fit random noise. Intuitively, more complex classes should fit random noise better than simple classes. 
The \emph{empirical Rademacher complexity} of $\pazocal{M}$ with respect to $\sample= \left\{\vec{v}^{\left(1\right)}, \dots, \vec{v}^{\left(N\right)}\right\}$ is  \[\erad\left(\pazocal{M}\right) = \E_{\vec{\sigma}}\left[\sup_{M \in \pazocal{M}} \frac{1}{N} \sum_{i = 1}^N \sigma_i \cdot \profit_M\left(\vec{v}^{(i)}\right) \right],\] where $\sigma_i \sim U\left(\left\{-1,1\right\}\right)$. Classic learning-theoretical results~\citep{Bartlett02:Rademacher,Koltchinskii01:Rademacher} imply the distribution-dependent generalization bound $\epsilon_{\cM}^{\dist}(N, \delta) = 2\E_{\sample \sim \dist^N}\left[\erad\left(\pazocal{M}\right)\right] + U \sqrt{\frac{2\ln \left(2/\delta\right)}{N}},$ where $U$ is the maximum profit of any mechanism in $\mclass$ over the support of $\dist$.
It is well-known that Rademacher complexity and pseudo-dimension are connected as follows.

\begin{restatable}{lemma}{pseudimerad}[\cite{Pollard84:Convergence, Dudley87:Universal}]\label{lem:pdim2erad}
	For any mechanism class $\mclass$ and any set of samples $\sample$ of size $N$, $\erad(\pazocal{M}) = O\left(U \sqrt {\frac{\pdim(\pazocal{M})}{N}}\right).$
\end{restatable}

We show that if the profit functions of a class $\mclass$ decompose additively into simpler functions, then we can bound $\erad\left(\mclass\right)$ using the Rademacher complexity of those simpler functions. We use this to prove tighter bounds for several mechanism classes under additive buyers with values drawn from item-independent distributions. This includes product distributions, a setting that has been studied extensively~\citep[e.g.,][]{Hart12:Approximate, Cai17:Learning, Yao14:n, Cai16:Duality, Babaioff17:Menu}.
 Formally, a mechanism class $\mclass$ \emph{decomposes additively} if for all $M \in \mclass$, there are $T$ functions $f_{1, M}, \dots, f_{T, M}$ such that $\profit_{M}\left(\cdot\right) = f_{1, M}\left(\cdot\right) + \cdots + f_{T, M}\left(\cdot\right)$.

\begin{corollary}\label{cor:data}
Suppose that $\mclass$ is a set of additively decomposable mechanisms. Moreover, suppose that for all $M \in \mclass$, the range of $f_{i, M}$ over the support of $\dist$ is $[0, U_i]$ and that the class $\left\{f_{i, M} : M \in \mclass\right\}$ is $\left(d_i,t_i\right)$-delineable.
For any set $\sample \sim \dist^N$, \[\erad\left(\mclass\right) \leq 180\sum_{i = 1}^T U_i \sqrt{\frac{d_i\log \left(4d_it_i\right)}{N}}.\]
\end{corollary}

\begin{proof}
This follows from Theorem~\ref{thm:main_pdim}, Lemma~\ref{lem:pdim2erad}, and the fact that for any sets $\pazocal{G}$ and $\pazocal{G}'$ of functions with a domain $\cA$ and any $\sample \subseteq \cA$, $\erad\left(\left\{g + g' : g \in \pazocal{G}, g' \in \pazocal{G}'\right\}\right) \leq \erad\left(\pazocal{G}\right) + \erad\left(\pazocal{G}'\right)$.
\end{proof}

We now instantiate Corollary~\ref{cor:data} for several mechanism classes. The proofs are in Appendix~\ref{APP:DATA}.

\begin{restatable}{lemma}{secondPriceProduct}\label{lem:second_price_product}
Let $\mclass$ and $\mclass'$ be the sets of second-price auctions with anonymous and non-anonymous reserves. Suppose the buyers are additive, $\dist$ is item-independent, and the cost function is additive. For any set $\sample \sim \dist^N$, $\erad\left(\mclass\right) \leq 180U\sqrt{1/N}$ and $\erad\left(\mclass'\right) \leq 180U\sqrt{n \log (4n)/N}$.
\end{restatable}

\begin{restatable}{lemma}{itemProduct}\label{lem:item_product}
Let $\mclass$ and $\mclass'$ be the sets of anonymous and non-anonymous item-pricing mechanisms. Suppose the buyers are additive, $\dist$ is item-independent, and the cost function is additive. For any set of samples $\sample \sim \dist^N$, $\erad\left(\mclass\right) \leq 180U\sqrt{1/N}$ and $\erad\left(\mclass'\right) \leq 180U\sqrt{n \log (4n)/N}$.
\end{restatable}

We prove similar guarantees for menus of item lotteries (Lemma~\ref{lem:lottery_product}).
Finally, we provide lower bounds showing that one could not prove the generalization guarantees implied by Lemmas~\ref{lem:second_price_product} and \ref{lem:item_product}---which do not depend on the number of items---using pseudo-dimension alone.
\begin{restatable}{theorem}{pdimLower}\label{thm:pdim_lower}
Let $\pazocal{M}$ and $\pazocal{M}'$ be the classes of anonymous and non-anonymous item-pricing mechanisms. Then $\pdim\left(\pazocal{M}\right) \geq m$ and $\pdim\left(\pazocal{M}'\right) \geq nm$. The same holds if $\pazocal{M}$ and $\pazocal{M}'$ are the classes of second-price auctions with anonymous and non-anonymous reserves.
\end{restatable}

\section{Optimizing the profit-generalization tradeoff}\label{SEC:SPM}
In this section, we use our results from Section~\ref{SEC:MAIN} to provide guarantees for optimizing the \emph{profit-generalization tradeoff}, drawing on classic machine learning results on \emph{structural risk minimization}~\citep{Vapnik74:Theory,Blumer87:Occam}.
\begin{figure}
	\centering
	\includegraphics{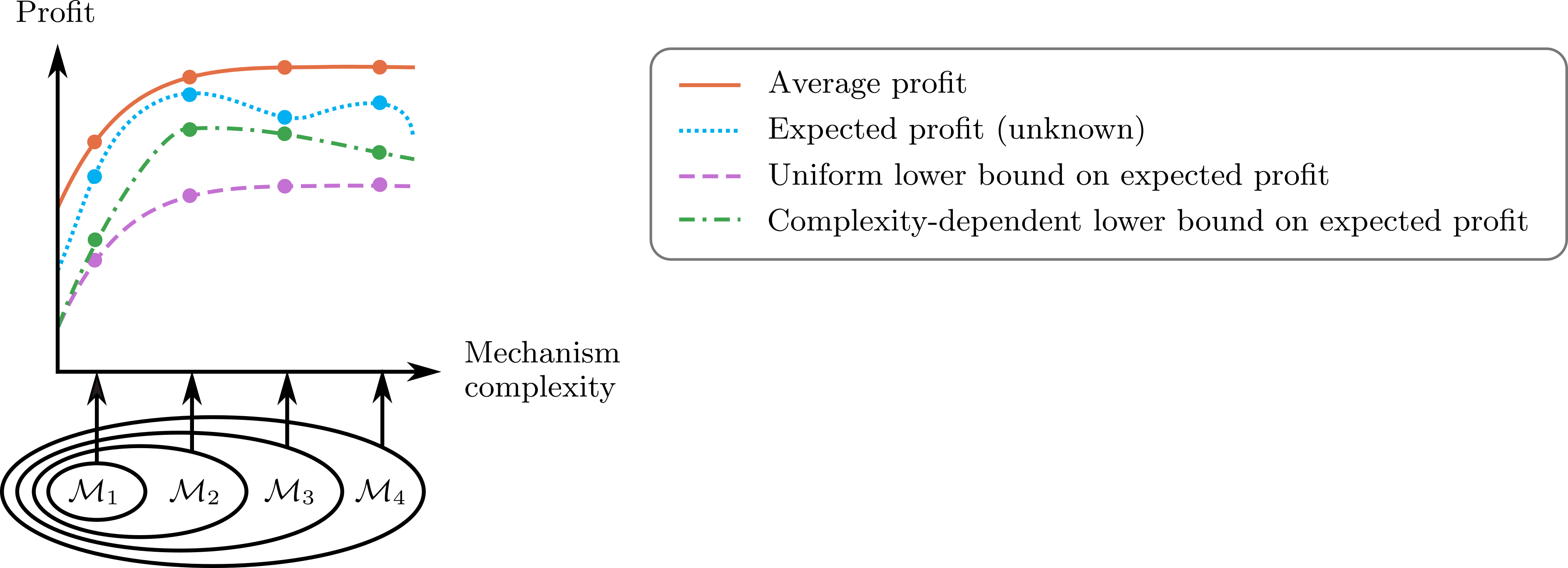}
	\caption{Uniform generalization guarantees versus stronger complexity-dependent bounds for a mechanism class $\cM = \cM_4 \subseteq \cM_3 \supseteq \cM_2 \supseteq \cM_1$. See Section~\ref{SEC:SPM} for a description.}\label{fig:SPM}
\end{figure}
We illustrate this tradeoff\footnote{These figures are purely illustrative; they are not based on a simulation or real data.} in Figure~\ref{fig:SPM} with a mechanism class $\cM$ that decomposes into a nested sequence $\pazocal{M}_1 \subseteq \cdots \subseteq \pazocal{M}_4 = \pazocal{M}$.
The $x$-axis measures the intrinsic complexity (e.g., pseudo-dimension) of the subclasses. The orange solid line illustrates the average profit over a fixed set of samples $\sample$ of the mechanism $\hat{M}_i \in \pazocal{M}_i$ that maximizes average profit. In particular, the dot on the orange solid line above $\cM_i$ illustrates $\profit_{\sample}(\hat{M}_i)$. Since $\cM_i \subseteq \cM_j$ for $i \leq j$, $\profit_{\sample}(\hat{M}_i) \leq \profit_{\sample}(\hat{M}_j)$. Similarly, the dot on the blue dotted line above $\cM_i$ illustrates the expected profit of $\hat{M}_i$. This line begins decreasing when the complexity grows to the point that overfitting occurs. The purple dashed line illustrates a uniform lower bound $\profit_{\sample}(\hat{M}_i) - \epsilon_{\pazocal{M}}(N, \delta)$ on the expected profit of $\hat{M}_i$.

Our general theorem allows us to easily derive bounds $\epsilon_{\pazocal{M}_i}(N, \delta)$ for each class $\pazocal{M}_i$. We can then ``spread'' $\delta$ across all subsets $\mclass_1, \dots, \mclass_t$ using a function $w: \N \to [0,1]$ such that $\sum w(i) \leq 1$. By a union bound, with probability $1-\delta$, for all $M \in \mclass$, $|\profit_{\sample}(M) - \profit_{\dist}(M)| \leq \min_{i : M \in \mclass_i}  \epsilon_{\pazocal{M}_i}(N, \delta \cdot w(i))$.
This is illustrated by the green dashed-dotted line in Figure~\ref{fig:SPM}, where the lower bound on the expected profit of $\hat{M}_i$ is $\profit_{\sample}(\hat{M}_i) - \epsilon_{\pazocal{M}_i}(N, \delta \cdot w(i))$. By maximizing this complexity-dependent lower bound, the designer can determine that $\hat{M}_2$ is better than $\hat{M}_4$.

The decomposition of $\pazocal{M}$ into subsets and the choice of weights allow the designer to encode his prior knowledge about the market. For example, if mechanisms in $\pazocal{M}_i$ are likely to be profitable, he can increase $w(i)$, which in turn decreases $\epsilon_{\pazocal{M}_i}(N, \delta \cdot w(i))$, thereby implying stronger guarantees.

We now apply this analysis to item pricing. To perform \emph{market segmentation}, the seller can break the buyers into $k$ groups and charge each group a different price. For $k \in [n]$, let $\mclass_k$ be the class of non-anonymous pricing mechanisms with $k$ price groups: for all mechanisms in $\mclass_k$, there is a partition of the buyers $B_1, \dots, B_k$ such that for all $t \in [k]$, all buyers $j,j' \in B_t$, and all items $i \in [m]$, $p_j(\vec{e}_i) = p_{j'}(\vec{e}_i)$.  We derive the following guarantee for this hierarchy.
\begin{restatable}{theorem}{itemSPM}\label{thm:item_pricing_SPM}
Let $\pazocal{M}$ be the class of non-anonymous item-pricing mechanisms over additive buyers. With probability $1-\delta$ over the draw $\sample \sim \pazocal{D}^N$, for any $k \in [n]$ and any mechanism $M \in \pazocal{M}_k$, \[\left|\profit_{\sample}\left(M\right) - \profit_{\dist}\left(M\right)\right| \leq 360U \sqrt{\frac{km \log \left(4nm\right)}{N}} + 4U\sqrt{\frac{2}{N}\ln \frac{4}{\delta \cdot w\left(k\right)}}.\]
\end{restatable}

We prove results for two-part tariffs, AMA, $\lambda$-auctions, and lottery menus in Appendix~\ref{APP:SPM}.

\section{Comparison of our results to prior research}\label{SEC:COMPARISON}
We compare our results to prior research that provides generalization bounds for some of the mechanisms we study. 
\citet{Morgenstern16:Learning} studied ``simple'' multi-item pricing mechanisms and second-price auctions.
See Tables~\ref{tab:results_pricing} and \ref{tab:results_auctions} and Appendix~\ref{sec:divisible} for a comparison.

\citet{Syrgkanis17:Sample} provided bounds specifically for the mechanism that maximizes average revenue over the samples, whereas our bounds apply to every mechanism in a given class. This is important when exactly optimizing average revenue is intractable. To illustrate their bounds, let $\hat{M}$ be the anonymous item-pricing mechanism maximizing average revenue over  $N$ samples. \citet{Syrgkanis17:Sample} proved that with probability $1-\delta$, $|\profit_{\dist}(\hat{M}) - \max_{M \in \mclass}\profit_{\dist}(M)| = O((U/\delta)\sqrt{m\log(nN)/N})$. When $\dist$ is item-independent, our bound $O(U\sqrt{\log(1/\delta)/N})$ is an improvement. Otherwise, our bound $O(U \sqrt{m \log (m)/N} + U\sqrt{\log(1/\delta)/N})$ is incomparable. \citet{Syrgkanis17:Sample} proved a similar bound for non-anonymous prices (see Table~\ref{tab:results_pricing}) which is also incomparable.

\citet{Cai17:Learning} provided learning algorithms for buyers with values drawn from product distributions.
For additive and unit-demand buyers with values bounded in $[0,H]$, we match their guarantees, which are based on those of~\citet{Morgenstern16:Learning}. They also study buyers with XOS, constrained additive, and subadditive values in which case our results do not provide an improvement.
For example, for XOS and constrained additive buyers, \citet{Cai17:Learning} provided algorithms which return item-pricing mechanisms with entry fees. Our results would imply pessimistic bounds for this class due to the exponential number of parameters. To circumvent this, their proofs use specific structural properties exhibited by bidders with product distributions, whereas the primary focus of this paper is to provide a general theory applicable to many different mechanisms and buyer types.

\citet{Medina17:Revenue} studied a different model than ours where items are defined by feature vectors and the seller has access to a bid predictor mapping feature vectors to bids.

Among other results, \citet[][Section 6.1]{Devanur17:Sample} proved that for the class $\mclass$ of second price item auctions with non-anonymous reserves, $N = O((U/\epsilon)^2(n \log (U/\epsilon) + \log (1/\delta)))$ samples are sufficient to ensure that with probability $1-\delta$, for all $M \in \pazocal{M}$, $|\profit_{\sample}(M) - \profit_{\dist}(M)| \leq \epsilon$. Our Lemma~\ref{lem:second_price} implies $O((U/\epsilon)^2(n \log n + \log (1/\delta)))$ samples are sufficient, which is incomparable.

\citet{Gonczarowski18:Sample} studied a setting where there are $n$ buyers with additive, independent values in the interval $[0,H]$ for $m$ items, as well as a generalization to Lipschitz valuations. They proved that poly$(n, m, H, 1/\epsilon)$ samples are sufficient to learn an approximately incentive compatible mechanism with $\epsilon$-approximately optimal revenue. From a computation perspective, it is not known how to efficiently find an $\epsilon$-approximately optimal mechanism in this setting where the number of types is exponential in the number of items.
In contrast, our guarantees apply uniformly to any mechanism from within a variety of parameterized classes, so the seller can use our guarantees to bound the expected profit of the mechanism he obtains via \emph{any} optimization procedure.
However, there may not be a mechanism in these classes with nearly optimal revenue.

\section{Conclusion}
We studied profit maximization when the mechanism designer has a set of samples from the distribution over buyers' values.  We identified structural similarities of mechanism classes including non-linear pricing mechanisms, generalized VCG mechanisms such as affine maximizer auctions, and lotteries: profit is a piecewise-linear function of the mechanism class's parameters. These similarities led us to a general theorem that gives generalization bounds for a broad range of mechanism classes. It offers the first generalization guarantees for many important classes and also matches and improves over many existing bounds.
Finally, we provided guarantees for optimizing a fundamental tradeoff in sample-based mechanism design: more complex mechanisms have higher average profit over the samples than simpler mechanisms, but require more samples to avoid overfitting.

An important direction for future research is the development of learning algorithms for multi-item profit maximization. Learning algorithms have been proposed for several of the mechanism classes we consider, including two-part tariffs~\citep{Balcan20:Efficient}, affine maximizer auctions~\citep{Sandholm15:Automated}, and item-pricing mechanisms~\citep[][who also provide algorithms for other multi-item mechanism classes]{Cai17:Learning}. A line of research also provides learning algorithms for single-item profit maximization~\citep{Devanur16:Sample,Hartline16:Non,Gonczarowski17:Efficient,Guo19:Settling}.

Another direction is to use tools such as Rademacher complexity to provide generalization bounds for non-worst-case distributions beyond item-independent distributions (the focus of Section~\ref{SEC:DATA}). For example, suppose any buyer's values for any items are correlated, but his values are independent of any other buyer's values. Can the bounds in this paper be improved?

\paragraph{Acknowledgements.} This material is based on work supported by the National Science Foundation under grants CCF-1422910, CCF-1535967, CCF-1733556, CCF-1910321, IIS-1617590, IIS-1618714, IIS-1718457, IIS-1901403, SES-1919453, and a Graduate Research Fellowship; the ARO under awards W911NF2010081 and W911NF1710082; the Defense Advanced Research Projects Agency under cooperative agreement HR00112020003; an Amazon Research Award; a Microsoft Research Faculty Fellowship; an AWS Machine Learning Research Award; a Bloomberg Data Science research grant; an IBM PhD Fellowship; and a  fellowship  from  Carnegie  Mellon  University’s  Center  for Machine Learning and Health.

\bibliographystyle{plainnat}
\bibliography{../../VitercikLibrary}

\begin{thebibliography}{77}
\providecommand{\natexlab}[1]{#1}
\providecommand{\url}[1]{\texttt{#1}}
\expandafter\ifx\csname urlstyle\endcsname\relax
  \providecommand{\doi}[1]{doi: #1}\else
  \providecommand{\doi}{doi: \begingroup \urlstyle{rm}\Url}\fi

\bibitem[Alon et~al.(2017)Alon, Babaioff, Gonczarowski, Mansour, Moran, and
  Yehudayoff]{Alon17:Submultiplicative}
Noga Alon, Moshe Babaioff, Yannai~A Gonczarowski, Yishay Mansour, Shay Moran,
  and Amir Yehudayoff.
\newblock Submultiplicative {G}livenko-{C}antelli and uniform convergence of
  revenues.
\newblock \emph{Proceedings of the Annual Conference on Neural Information
  Processing Systems (NIPS)}, 2017.

\bibitem[Anthony and Bartlett(2009)]{Anthony09:Neural}
Martin Anthony and Peter Bartlett.
\newblock \emph{Neural Network Learning: Theoretical Foundations}.
\newblock Cambridge University Press, 2009.

\bibitem[Araman and Caldentey(2009)]{Araman09:Dynamic}
Victor~F. Araman and René Caldentey.
\newblock Dynamic pricing for nonperishable products with demand learning.
\newblock \emph{Operations Research}, 57\penalty0 (5):\penalty0 1169--1188,
  2009.

\bibitem[Babaioff et~al.(2014)Babaioff, Immorlica, Lucier, and
  Weinberg]{Babaioff14:Simple}
Moshe Babaioff, Nicole Immorlica, Brendan Lucier, and S.~Matthew Weinberg.
\newblock A simple and approximately optimal mechanism for an additive buyer.
\newblock In \emph{Proceedings of the IEEE Symposium on Foundations of Computer
  Science (FOCS)}, 2014.

\bibitem[Babaioff et~al.(2017)Babaioff, Gonczarowski, and
  Nisan]{Babaioff17:Menu}
Moshe Babaioff, Yannai~A Gonczarowski, and Noam Nisan.
\newblock The menu-size complexity of revenue approximation.
\newblock In \emph{Proceedings of the Annual Symposium on Theory of Computing
  (STOC)}, 2017.

\bibitem[Balcan et~al.(2005)Balcan, Blum, Hartline, and
  Mansour]{Balcan05:Mechanism}
Maria-Florina Balcan, Avrim Blum, Jason~D. Hartline, and Yishay Mansour.
\newblock Mechanism design via machine learning.
\newblock In \emph{Proceedings of the IEEE Symposium on Foundations of Computer
  Science (FOCS)}, pages 605--614, 2005.

\bibitem[Balcan et~al.(2008{\natexlab{a}})Balcan, Blum, Hartline, and
  Mansour]{Balcan08:Reducing}
Maria-Florina Balcan, Avrim Blum, Jason Hartline, and Yishay Mansour.
\newblock Reducing mechanism design to algorithm design via machine learning.
\newblock \emph{Journal of Computer and System Sciences}, 74:\penalty0 78--89,
  December 2008{\natexlab{a}}.

\bibitem[Balcan et~al.(2008{\natexlab{b}})Balcan, Blum, and
  Mansour]{Balcan08:Item}
Maria-Florina Balcan, Avrim Blum, and Yishay Mansour.
\newblock Item pricing for revenue maximization.
\newblock In \emph{Proceedings of the ACM Conference on Economics and
  Computation (EC)}, pages 50--59, 2008{\natexlab{b}}.

\bibitem[Balcan et~al.(2014)Balcan, Daniely, Mehta, Urner, and
  Vazirani]{Balcan14:Learning}
Maria-Florina Balcan, Amit Daniely, Ruta Mehta, Ruth Urner, and Vijay~V
  Vazirani.
\newblock Learning economic parameters from revealed preferences.
\newblock In \emph{Proceedings of the Conference on Web and Internet Economics
  (WINE)}, 2014.

\bibitem[Balcan et~al.(2016)Balcan, Sandholm, and Vitercik]{Balcan16:Sample}
Maria-Florina Balcan, Tuomas Sandholm, and Ellen Vitercik.
\newblock Sample complexity of automated mechanism design.
\newblock In \emph{Proceedings of the Annual Conference on Neural Information
  Processing Systems (NIPS)}, 2016.

\bibitem[Balcan et~al.(2020)Balcan, Prasad, and Sandholm]{Balcan20:Efficient}
Maria-Florina Balcan, Siddharth Prasad, and Tuomas Sandholm.
\newblock Efficient algorithms for learning revenue-maximizing two-part
  tariffs.
\newblock In \emph{Proceedings of the International Joint Conference on
  Artificial Intelligence (IJCAI)}, 2020.

\bibitem[Bartlett and Mendelson(2002)]{Bartlett02:Rademacher}
Peter~L Bartlett and Shahar Mendelson.
\newblock Rademacher and {G}aussian complexities: Risk bounds and structural
  results.
\newblock \emph{Journal of Machine Learning Research}, 3\penalty0
  (Nov):\penalty0 463--482, 2002.

\bibitem[Bateni et~al.(2015)Bateni, Dehghani, Hajiaghayi, and
  Seddighin]{Bateni15:Revenue}
MohammadHossein Bateni, Sina Dehghani, MohammadTaghi Hajiaghayi, and Saeed
  Seddighin.
\newblock Revenue maximization for selling multiple correlated items.
\newblock In \emph{Proceedings of the European Symposium on Algorithms (ESA)},
  2015.

\bibitem[Besbes and Zeevi(2009)]{Besbes09:Dynamic}
Omar Besbes and Assaf Zeevi.
\newblock Dynamic pricing without knowing the demand function: Risk bounds and
  near-optimal algorithms.
\newblock \emph{Operations Research}, 57\penalty0 (6):\penalty0 1407--1420,
  2009.

\bibitem[Blumer et~al.(1987)Blumer, Ehrenfeucht, Haussler, and
  Warmuth]{Blumer87:Occam}
Anselm Blumer, Andrzej Ehrenfeucht, David Haussler, and Manfred~K Warmuth.
\newblock Occam's razor.
\newblock \emph{Information processing letters}, 24\penalty0 (6):\penalty0
  377--380, 1987.

\bibitem[Broder and Rusmevichientong(2012)]{Broder12:Dynamic}
Josef Broder and Paat Rusmevichientong.
\newblock Dynamic pricing under a general parametric choice model.
\newblock \emph{Operations Research}, 60\penalty0 (4):\penalty0 965--980, 2012.

\bibitem[Bubeck et~al.(2017)Bubeck, Devanur, Huang, and
  Niazadeh]{Bubeck17:Online}
S{\'e}bastien Bubeck, Nikhil~R Devanur, Zhiyi Huang, and Rad Niazadeh.
\newblock Online auctions and multi-scale online learning.
\newblock \emph{Proceedings of the ACM Conference on Economics and Computation
  (EC)}, 2017.

\bibitem[Buck(1943)]{Buck43:Partition}
R.~C. Buck.
\newblock Partition of space.
\newblock \emph{Amer. Math. Monthly}, 50:\penalty0 541--544, 1943.
\newblock ISSN 0002-9890.

\bibitem[Cai and Daskalakis(2017)]{Cai17:Learning}
Yang Cai and Constantinos Daskalakis.
\newblock Learning multi-item auctions with (or without) samples.
\newblock In \emph{Proceedings of the IEEE Symposium on Foundations of Computer
  Science (FOCS)}, 2017.

\bibitem[Cai and Zhao(2017)]{Cai17:Simple}
Yang Cai and Mingfei Zhao.
\newblock Simple mechanisms for subadditive buyers via duality.
\newblock In \emph{Proceedings of the Annual Symposium on Theory of Computing
  (STOC)}, 2017.

\bibitem[Cai et~al.(2016)Cai, Devanur, and Weinberg]{Cai16:Duality}
Yang Cai, Nikhil~R. Devanur, and S.~Matthew Weinberg.
\newblock A duality based unified approach to {B}ayesian mechanism design.
\newblock In \emph{Proceedings of the Annual Symposium on Theory of Computing
  (STOC)}, 2016.

\bibitem[Chawla and Miller(2016)]{Chawla16:Mechanism}
Shuchi Chawla and J~Benjamin Miller.
\newblock Mechanism design for subadditive agents via an ex ante relaxation.
\newblock In \emph{Proceedings of the ACM Conference on Economics and
  Computation (EC)}, 2016.

\bibitem[Chawla et~al.(2007)Chawla, Hartline, and
  Kleinberg]{Chawla07:Algorithmic}
Shuchi Chawla, Jason~D Hartline, and Robert Kleinberg.
\newblock Algorithmic pricing via virtual valuations.
\newblock In \emph{Proceedings of the ACM Conference on Economics and
  Computation (EC)}, 2007.

\bibitem[Chawla et~al.(2010)Chawla, Hartline, Malec, and Sivan]{Chawla10:Multi}
Shuchi Chawla, Jason~D Hartline, David~L Malec, and Balasubramanian Sivan.
\newblock Multi-parameter mechanism design and sequential posted pricing.
\newblock In \emph{Proceedings of the Annual Symposium on Theory of Computing
  (STOC)}, 2010.

\bibitem[Chawla et~al.(2014)Chawla, Hartline, and
  Nekipelov]{Chawla14:Mechanism}
Shuchi Chawla, Jason Hartline, and Denis Nekipelov.
\newblock Mechanism design for data science.
\newblock In \emph{Proceedings of the ACM Conference on Economics and
  Computation (EC)}, 2014.

\bibitem[Cole and Roughgarden(2014)]{Cole14:Sample}
Richard Cole and Tim Roughgarden.
\newblock The sample complexity of revenue maximization.
\newblock In \emph{Proceedings of the Annual Symposium on Theory of Computing
  (STOC)}, 2014.

\bibitem[Collins(2000)]{Collins00:Discriminative}
Michael Collins.
\newblock Discriminative reranking for natural language parsing.
\newblock \emph{Proceedings of the International Conference on Machine Learning
  (ICML)}, 2000.

\bibitem[Conitzer and Sandholm(2002)]{Conitzer02:Mechanism}
Vincent Conitzer and Tuomas Sandholm.
\newblock Complexity of mechanism design.
\newblock In \emph{Proceedings of the Conference on Uncertainty in Artificial
  Intelligence (UAI)}, 2002.

\bibitem[Conitzer and Sandholm(2003)]{Conitzer03:Applications}
Vincent Conitzer and Tuomas Sandholm.
\newblock Applications of automated mechanism design.
\newblock In \emph{UAI-03 workshop on Bayesian Modeling Applications}, 2003.

\bibitem[Conitzer and Sandholm(2004)]{Conitzer04:Self}
Vincent Conitzer and Tuomas Sandholm.
\newblock Self-interested automated mechanism design and implications for
  optimal combinatorial auctions.
\newblock In \emph{Proceedings of the ACM Conference on Economics and
  Computation (EC)}, 2004.

\bibitem[Devanur et~al.(2016)Devanur, Huang, and Psomas]{Devanur16:Sample}
Nikhil~R Devanur, Zhiyi Huang, and Christos-Alexandros Psomas.
\newblock The sample complexity of auctions with side information.
\newblock In \emph{Proceedings of the Annual Symposium on Theory of Computing
  (STOC)}, 2016.

\bibitem[Devanur et~al.(2017)Devanur, Huang, and Psomas]{Devanur17:Sample}
Nikhil~R Devanur, Zhiyi Huang, and Christos-Alexandros Psomas.
\newblock The sample complexity of auctions with side information.
\newblock \emph{arXiv preprint arXiv:1511.02296}, 2017.

\bibitem[Dobzinski and Dughmi(2009)]{Dobzinski09:Power}
Shahar Dobzinski and Shaddin Dughmi.
\newblock On the power of randomization in algorithmic mechanism design.
\newblock In \emph{Proceedings of the IEEE Symposium on Foundations of Computer
  Science (FOCS)}, 2009.

\bibitem[Dobzinski and Sundararajan(2008)]{Dobzinski08:Characterizations}
Shahar Dobzinski and Mukund Sundararajan.
\newblock On characterizations of truthful mechanisms for combinatorial
  auctions and scheduling.
\newblock In \emph{Proceedings of the ACM Conference on Economics and
  Computation (EC)}, 2008.

\bibitem[Dudley(1987)]{Dudley87:Universal}
Richard Dudley.
\newblock Universal {D}onsker classes and metric entropy.
\newblock \emph{The Annals of Probability}, 15\penalty0 (4):\penalty0
  1306--1326, 1987.

\bibitem[D{\"u}tting et~al.(2020)D{\"u}tting, Kesselheim, and
  Lucier]{Dutting20:Log}
Paul D{\"u}tting, Thomas Kesselheim, and Brendan Lucier.
\newblock An $o (\log \log m)$ prophet inequality for subadditive combinatorial
  auctions.
\newblock In \emph{Proceedings of the IEEE Symposium on Foundations of Computer
  Science (FOCS)}, 2020.

\bibitem[Edelman et~al.(2007)Edelman, Ostrovsky, and
  Schwarz]{Edelman07:Internet}
Benjamin Edelman, Michael Ostrovsky, and Michael Schwarz.
\newblock Internet advertising and the generalized second-price auction:
  Selling billions of dollars worth of keywords.
\newblock \emph{The American Economic Review}, 97\penalty0 (1):\penalty0
  242--259, March 2007.

\bibitem[Eden et~al.(2021)Eden, Feldman, Friedler, Talgam-Cohen, and
  Weinberg]{Eden21:Simple}
Alon Eden, Michal Feldman, Ophir Friedler, Inbal Talgam-Cohen, and S~Matthew
  Weinberg.
\newblock A simple and approximately optimal mechanism for a buyer with
  complements.
\newblock \emph{Operations Research}, 69\penalty0 (1):\penalty0 188--206, 2021.

\bibitem[Elkind(2007)]{Elkind07:Designing}
Edith Elkind.
\newblock Designing and learning optimal finite support auctions.
\newblock In \emph{Proceedings of the ACM-SIAM Symposium on Discrete Algorithms
  (SODA)}, 2007.

\bibitem[Feldman et~al.(2015)Feldman, Gravin, and
  Lucier]{Feldman15:Combinatorial}
Michal Feldman, Nick Gravin, and Brendan Lucier.
\newblock Combinatorial auctions via posted prices.
\newblock In \emph{Proceedings of the ACM-SIAM Symposium on Discrete Algorithms
  (SODA)}, 2015.

\bibitem[Feldstein(1972)]{Feldstein72:Equity}
Martin~S Feldstein.
\newblock Equity and efficiency in public sector pricing: the optimal two-part
  tariff.
\newblock \emph{The Quarterly Journal of Economics}, pages 176--187, 1972.

\bibitem[Goldner and Karlin(2016)]{Goldner16:Prior}
Kira Goldner and Anna~R Karlin.
\newblock A prior-independent revenue-maximizing auction for multiple additive
  bidders.
\newblock In \emph{Proceedings of the Conference on Web and Internet Economics
  (WINE)}, 2016.

\bibitem[Gonczarowski and Nisan(2017)]{Gonczarowski17:Efficient}
Yannai~A Gonczarowski and Noam Nisan.
\newblock Efficient empirical revenue maximization in single-parameter auction
  environments.
\newblock In \emph{Proceedings of the Annual Symposium on Theory of Computing
  (STOC)}, pages 856--868, 2017.

\bibitem[Gonczarowski and Weinberg(2018)]{Gonczarowski18:Sample}
Yannai~A Gonczarowski and S~Matthew Weinberg.
\newblock The sample complexity of up-to-$\varepsilon$ multi-dimensional
  revenue maximization.
\newblock In \emph{Proceedings of the IEEE Symposium on Foundations of Computer
  Science (FOCS)}, 2018.

\bibitem[Guo et~al.(2019)Guo, Huang, and Zhang]{Guo19:Settling}
Chenghao Guo, Zhiyi Huang, and Xinzhi Zhang.
\newblock Settling the sample complexity of single-parameter revenue
  maximization.
\newblock \emph{Proceedings of the Annual Symposium on Theory of Computing
  (STOC)}, 2019.

\bibitem[Hart and Nisan(2012)]{Hart12:Approximate}
Sergiu Hart and Noam Nisan.
\newblock Approximate revenue maximization with multiple items.
\newblock In \emph{Proceedings of the ACM Conference on Economics and
  Computation (EC)}, 2012.

\bibitem[Hartline and Taggart(2016)]{Hartline16:Non}
Jason Hartline and Samuel Taggart.
\newblock Non-revelation mechanism design.
\newblock \emph{arXiv preprint arXiv:1608.01875}, 2016.

\bibitem[He et~al.(2014)He, Pan, Jin, Xu, Liu, Xu, Shi, Atallah, Herbrich,
  Bowers, and Candela]{He14:Practical}
Xinran He, Junfeng Pan, Ou~Jin, Tianbing Xu, Bo~Liu, Tao Xu, Yanxin Shi,
  Antoine Atallah, Ralf Herbrich, Stuart Bowers, and Joaquin~Quinonero Candela.
\newblock Practical lessons from predicting clicks on ads at {F}acebook.
\newblock In \emph{Proceedings of the International Workshop on Data Mining for
  Online Advertising}, 2014.

\bibitem[Huang et~al.(2015)Huang, Mansour, and Roughgarden]{Huang15:Making}
Zhiyi Huang, Yishay Mansour, and Tim Roughgarden.
\newblock Making the most of your samples.
\newblock In \emph{Proceedings of the ACM Conference on Economics and
  Computation (EC)}, 2015.

\bibitem[Jehiel et~al.(2007)Jehiel, Meyer-Ter-Vehn, and
  Moldovanu]{Jehiel07:Mixed}
Philippe Jehiel, Moritz Meyer-Ter-Vehn, and Benny Moldovanu.
\newblock Mixed bundling auctions.
\newblock \emph{Journal of Economic Theory}, 134\penalty0 (1):\penalty0
  494--512, 2007.

\bibitem[Koltchinskii(2001)]{Koltchinskii01:Rademacher}
Vladimir Koltchinskii.
\newblock Rademacher penalties and structural risk minimization.
\newblock \emph{IEEE Transactions on Information Theory}, 47\penalty0
  (5):\penalty0 1902--1914, 2001.

\bibitem[Kothari et~al.(2019)Kothari, Singla, Mohan, Schvartzman, and
  Weinberg]{Kothari19:Approximation}
Pravesh Kothari, Sahil Singla, Divyarthi Mohan, Ariel Schvartzman, and
  S~Matthew Weinberg.
\newblock Approximation schemes for a unit-demand buyer with independent items
  via symmetries.
\newblock In \emph{Proceedings of the IEEE Symposium on Foundations of Computer
  Science (FOCS)}, 2019.

\bibitem[Lavi et~al.(2003)Lavi, Mu'Alem, and Nisan]{Lavi03:Towards}
Ron Lavi, Ahuva Mu'Alem, and Noam Nisan.
\newblock Towards a characterization of truthful combinatorial auctions.
\newblock In \emph{Proceedings of the IEEE Symposium on Foundations of Computer
  Science (FOCS)}, 2003.

\bibitem[Li and Yao(2013)]{Li13:Revenue}
Xinye Li and Andrew Chi-Chih Yao.
\newblock On revenue maximization for selling multiple independently
  distributed items.
\newblock \emph{Proceedings of the National Academy of Sciences}, 110\penalty0
  (28):\penalty0 11232--11237, 2013.

\bibitem[Likhodedov and Sandholm(2004)]{Likhodedov04:Boosting}
Anton Likhodedov and Tuomas Sandholm.
\newblock Methods for boosting revenue in combinatorial auctions.
\newblock In \emph{Proceedings of the AAAI Conference on Artificial
  Intelligence}, 2004.

\bibitem[Likhodedov and Sandholm(2005)]{Likhodedov05:Approximating}
Anton Likhodedov and Tuomas Sandholm.
\newblock Approximating revenue-maximizing combinatorial auctions.
\newblock In \emph{Proceedings of the AAAI Conference on Artificial
  Intelligence}, 2005.

\bibitem[Matou{\v{s}}ek and Vondr{\'a}k(2001)]{Matouvsek01:Probabilistic}
Ji{\v{r}}{\'\i} Matou{\v{s}}ek and Jan Vondr{\'a}k.
\newblock The probabilistic method.
\newblock \emph{Lecture Notes, Department of Applied Mathematics, Charles
  University, Prague}, 2001.

\bibitem[Medina and Vassilvitskii(2017)]{Medina17:Revenue}
Andr{\'e}s~Mu{\~n}oz Medina and Sergei Vassilvitskii.
\newblock Revenue optimization with approximate bid predictions.
\newblock \emph{Proceedings of the Annual Conference on Neural Information
  Processing Systems (NIPS)}, 2017.

\bibitem[Mohri and Medina(2014)]{Mohri14:Learning}
Mehryar Mohri and Andr{\'e}s~Mu{\~n}oz Medina.
\newblock Learning theory and algorithms for revenue optimization in second
  price auctions with reserve.
\newblock In \emph{Proceedings of the International Conference on Machine
  Learning (ICML)}, 2014.

\bibitem[Morgenstern and Roughgarden(2015)]{Morgenstern15:Pseudo}
Jamie Morgenstern and Tim Roughgarden.
\newblock On the pseudo-dimension of nearly optimal auctions.
\newblock In \emph{Proceedings of the Annual Conference on Neural Information
  Processing Systems (NIPS)}, 2015.

\bibitem[Morgenstern and Roughgarden(2016)]{Morgenstern16:Learning}
Jamie Morgenstern and Tim Roughgarden.
\newblock Learning simple auctions.
\newblock In \emph{Proceedings of the Conference on Learning Theory (COLT)},
  2016.

\bibitem[Myerson(1981)]{Myerson81:Optimal}
Roger Myerson.
\newblock Optimal auction design.
\newblock \emph{Mathematics of Operation Research}, 6:\penalty0 58--73, 1981.

\bibitem[Oi(1971)]{Oi71:Disneyland}
Walter~Y Oi.
\newblock A {D}isneyland dilemma: Two-part tariffs for a {M}ickey {M}ouse
  monopoly.
\newblock \emph{The Quarterly Journal of Economics}, 85\penalty0 (1):\penalty0
  77--96, 1971.

\bibitem[Pollard(1984)]{Pollard84:Convergence}
David Pollard.
\newblock \emph{Convergence of Stochastic Processes}.
\newblock Springer, 1984.

\bibitem[Roberts(1979)]{Roberts79:Characterization}
Kevin Roberts.
\newblock The characterization of implementable social choice rules.
\newblock In J-J Laffont, editor, \emph{Aggregation and Revelation of
  Preferences}. North-Holland Publishing Company, 1979.

\bibitem[Roughgarden and Schrijvers(2016)]{Roughgarden15:Ironing}
Tim Roughgarden and Okke Schrijvers.
\newblock Ironing in the dark.
\newblock In \emph{Proceedings of the ACM Conference on Economics and
  Computation (EC)}, 2016.

\bibitem[Rubinstein and Weinberg(2015)]{Rubinstein15:Simple}
Aviad Rubinstein and S~Matthew Weinberg.
\newblock Simple mechanisms for a subadditive buyer and applications to revenue
  monotonicity.
\newblock In \emph{Proceedings of the ACM Conference on Economics and
  Computation (EC)}, 2015.

\bibitem[Sandholm(2003)]{Sandholm03:Automated}
Tuomas Sandholm.
\newblock Automated mechanism design: A new application area for search
  algorithms.
\newblock In \emph{Proceedings of the International Conference on Principles
  and Practice of Constraint Programming (CP)}, 2003.

\bibitem[Sandholm and Likhodedov(2015)]{Sandholm15:Automated}
Tuomas Sandholm and Anton Likhodedov.
\newblock Automated design of revenue-maximizing combinatorial auctions.
\newblock \emph{Operations Research}, 63\penalty0 (5):\penalty0 1000--1025,
  September--October 2015.

\bibitem[Shalev-Shwartz and Ben-David(2014)]{Shalev14:Understanding}
Shai Shalev-Shwartz and Shai Ben-David.
\newblock \emph{Understanding machine learning: From theory to algorithms}.
\newblock Cambridge University Press, 2014.

\bibitem[Syrgkanis(2017)]{Syrgkanis17:Sample}
Vasilis Syrgkanis.
\newblock A sample complexity measure with applications to learning optimal
  auctions.
\newblock \emph{Proceedings of the Annual Conference on Neural Information
  Processing Systems (NIPS)}, 2017.

\bibitem[Tang(2017)]{Tang17:Reinforcement}
Pingzhong Tang.
\newblock Reinforcement mechanism design.
\newblock In \emph{Proceedings of the International Joint Conference on
  Artificial Intelligence (IJCAI)}, 2017.

\bibitem[Tang and Sandholm(2012)]{Tang12:Mixed}
Pingzhong Tang and Tuomas Sandholm.
\newblock Mixed-bundling auctions with reserve prices.
\newblock In \emph{Proceedings of the Conference for Autonomous Agents and
  Multi-Agent Systems (AAMAS)}, 2012.

\bibitem[Vapnik and Chervonenkis(1974)]{Vapnik74:Theory}
Vladimir Vapnik and Alexey Chervonenkis.
\newblock Theory of pattern recognition, 1974.

\bibitem[Wilson(1993)]{Wilson93:Nonlinear}
Robert~B Wilson.
\newblock \emph{Nonlinear pricing}.
\newblock Oxford University Press on Demand, 1993.

\bibitem[Yao(2014)]{Yao14:n}
Andrew Chi-Chih Yao.
\newblock An n-to-1 bidder reduction for multi-item auctions and its
  applications.
\newblock In \emph{Proceedings of the ACM-SIAM Symposium on Discrete Algorithms
  (SODA)}, 2014.

\bibitem[Yee and Ifrach(2015)]{Yee15:Aerosolve}
Hector Yee and Bar Ifrach.
\newblock Aerosolve: Machine learning for humans.
\newblock \emph{Open Source}, 2015.
\newblock URL \url{http://nerds.airbnb.com/aerosolve/}.

\end{thebibliography}

\appendix
\section{Proofs from Section~\ref{SEC:MAIN}}\label{APP:MAIN}

\begin{figure}
	\centering
	\begin{subfigure}{0.45\textwidth}
		\includegraphics{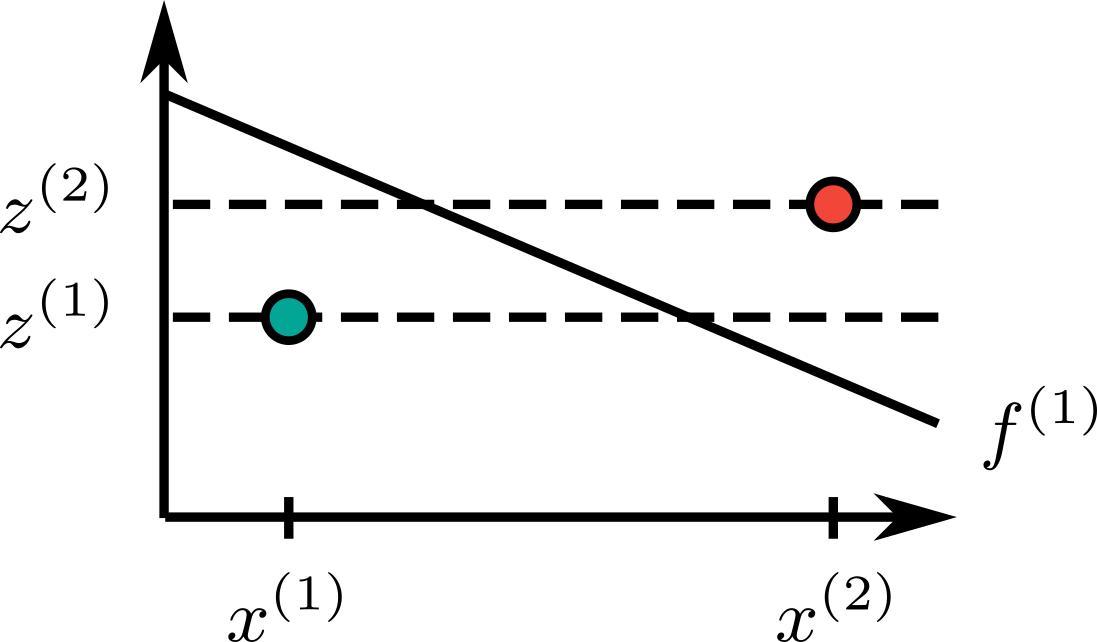}\centering
		\caption{Illustration of an affine function $f^{(1)} : \R \to \R$ where $f^{(1)}\left(x^{(1)}\right)$ is larger than the witness $z^{(1)}$ and $f^{(1)}\left(x^{(2)}\right) < z^{(2)}$.}
	\end{subfigure}\qquad
	\begin{subfigure}{0.45\textwidth}
		\includegraphics{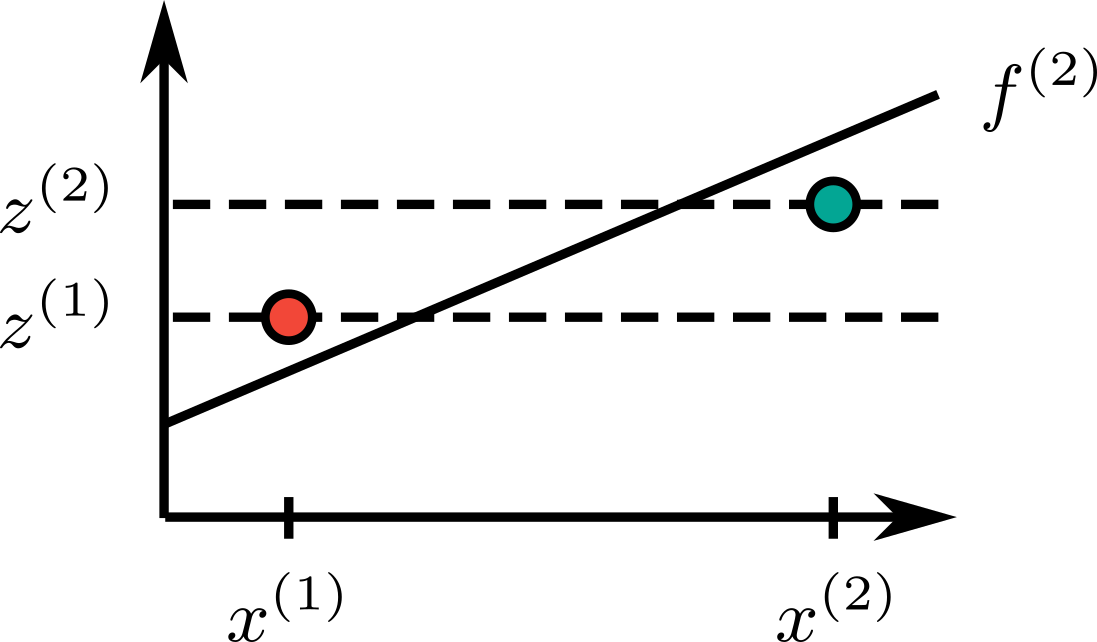}\centering
		\caption{Illustration of another affine function $f^{(2)} : \R \to \R$. Here, $f^{(2)}\left(x^{(1)}\right) < z^{(1)}$ and $f^{(2)}\left(x^{(2)}\right) > z^{(2)}$.}
	\end{subfigure}\qquad
	\begin{subfigure}{0.45\textwidth}
		\includegraphics{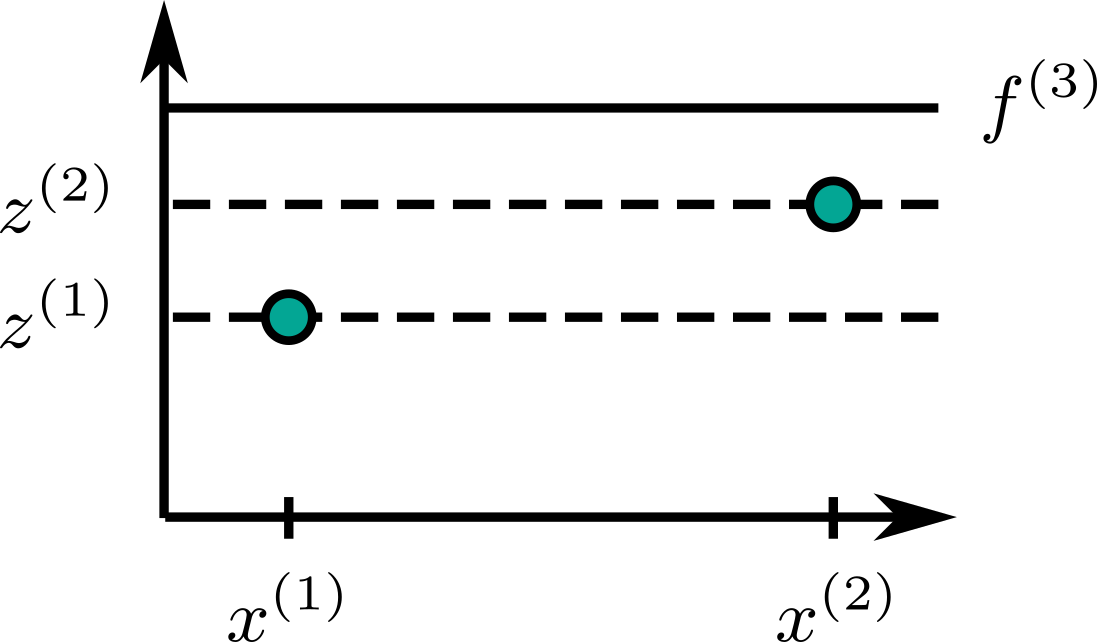}\centering
		\caption{Illustration of a third function $f^{(3)}$. Here, $f^{(3)}\left(x^{(1)}\right) > z^{(1)}$ and $f^{(3)}\left(x^{(2)}\right) > z^{(2)}$.}
	\end{subfigure}\qquad
	\begin{subfigure}{0.45\textwidth}
		\includegraphics{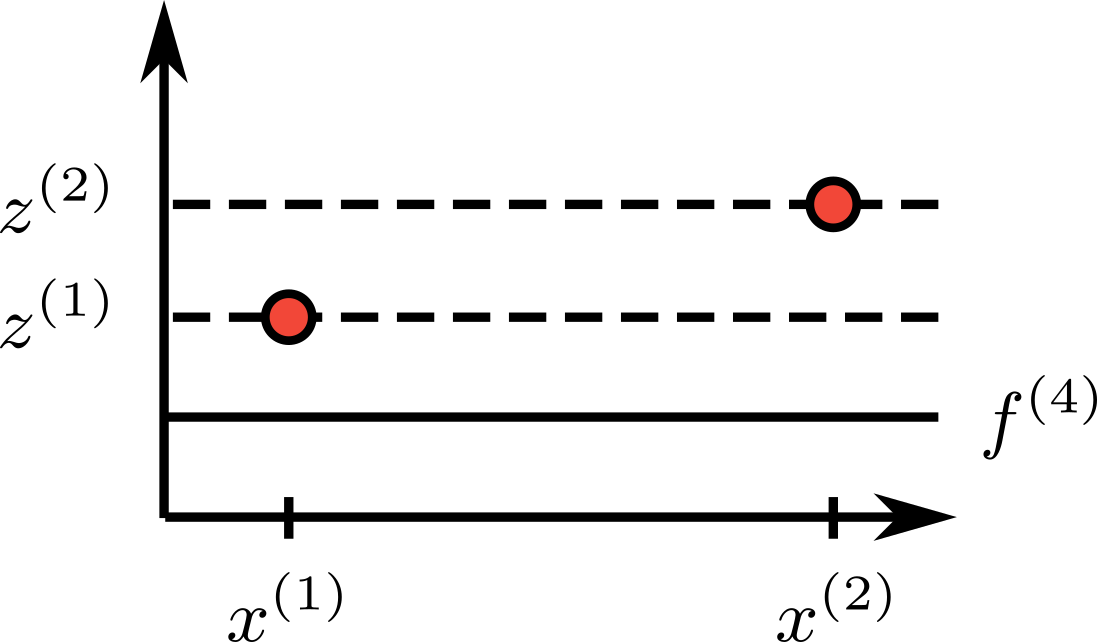}\centering
		\caption{Illustration of one last function $f^{(4)}$. Here, $f^{(4)}\left(x^{(1)}\right) < z^{(1)}$ and $f^{(4)}\left(x^{(2)}\right) < z^{(2)}$.}
	\end{subfigure}
	\caption{The two points $x^{(1)}$ and $x^{(2)}$ can be shattered by the set $\pazocal{F}$ of affine functions mapping $\R$ to $\R$.}\label{fig:shattering}
\end{figure}

\begin{corollary}\label{cor:add_apx}
	Let $M^* \in \cM$ be the mechanism that maximizes expected profit over the distribution over buyers' values. For any $\delta \in (0,1)$, with probability at least $1-\delta$ over the draw of a set of samples $\sample$ of size $N$ from the distribution over buyers' values, the difference between the expected profit of $\hat{M}_{\rho}$ and expected profit of $M^*$ is at most $\rho + \epsilon_{\pazocal{M}}\left(N, \frac{\delta}{2}\right) + U\sqrt{\frac{1}{2N} \ln \frac{4}{\delta}}.$
\end{corollary}

\begin{proof}
Let $\pazocal{M}(\sample)$ be the mechanism in $\pazocal{M}$ that maximizes empirical profit over $\sample$. With probability at least $1-\delta$,
\begin{align}
 {\normalfont \profit}_{\pazocal{D}}\left(\hat{M}_{\rho}\right) + \epsilon_{\pazocal{M}}\left(N, \frac{\delta}{2}\right) &\geq   {\normalfont \profit}_{\sample}\left(\hat{M}_{\rho}\right)\label{ineq:1a} \\
&\geq {\normalfont \profit}_{\sample}\left(\pazocal{M}(\sample)\right) - \rho\label{ineq:2a}\\ 
&\geq  {\normalfont \profit}_{\sample}\left({M^*}\right) - \rho\label{ineq:3a} \\
& \geq {\normalfont \profit}_{\pazocal{D}}\left({M^*}\right) -  U \sqrt{\frac{2\ln(4/\delta)}{2N}}-\rho.\label{ineq:4a}
\end{align}

Inequality~\eqref{ineq:1a} follows from standard uniform convergence bounds: with probability at least $1-\delta/2$, \[\left|{\normalfont \profit}_{\dist}\left(\hat{M}_{\rho}\right) - {\normalfont \profit}_{\sample}\left(\hat{M}_{\rho}\right)\right| \leq \epsilon_{\pazocal{M}}\left(N, \frac{\delta}{2}\right).\] Inequality~\eqref{ineq:2a} follows from the fact that $\hat{M}_{\rho}$ has empirical profit that is within an additive factor of $\rho$ from empirically optimal over the set of samples, or in other words, ${\normalfont \profit}_\sample\left(\hat{M}_{\rho}\right) \geq {\normalfont \profit}_\sample\left(\pazocal{M}(\sample)\right) - \rho$.  Inequality~\eqref{ineq:3a} follows because $\pazocal{M}(\sample)$ is the empirical profit maximizer (i.e., it maximizes ${\normalfont \profit}_\pazocal{S}\left(M\right)$). Finally, inequality~\eqref{ineq:4a} is a result of Hoeffding's inequality, which guarantees that with probability at least $1-\delta/2$, ${\normalfont \profit}_{\sample}\left(M^*\right) \geq {\normalfont \profit}_{\pazocal{D}}\left(M^*\right) -  U \sqrt{\frac{2\ln(4/\delta)}{2N}}$.

Rearranging, we get that \[ {\normalfont \profit}_{\pazocal{D}}\left(\hat{M}_{\rho}\right) \geq {\normalfont \profit}_{\pazocal{D}}\left(M^*\right) - \epsilon_{\pazocal{M}}\left(N, \frac{\delta}{2}\right) - U \sqrt{\frac{2\ln(4/\delta)}{2N}} - \rho,\] as claimed.
\end{proof}

\begin{restatable}{corollary}{multApx}\label{cor:mult_apx}
	Let $\pazocal{M}$ be a mechanism class and let $M^* \in \pazocal{M}$ be a mechanism with maximum expected profit. Given a set of samples $\sample$, let $\hat{M}_{\alpha}$ be a mechanism in $\pazocal{M}$ with empirical profit that is at least an $\alpha$-fraction of the empirically optimal: $\sum_{\vec{v} \in \sample} \profit_{\hat{M}_{\alpha}}\left(\vec{v}\right) \geq \alpha \cdot \max_{M \in \mclass} \sum_{\vec{v} \in \sample} \profit_{M}\left(\vec{v}\right) $. With probability at least $1-\delta$ over the draw $\sample \sim \dist^N$, the difference between the expected profit of $\hat{M}_{\alpha}$ and an $\alpha$-fraction of the expected profit of $M^*$ is at most $ \epsilon_{\pazocal{M}}\left(N, \frac{\delta}{2}\right) + U\alpha \sqrt{\frac{\ln(4/\delta)}{2N}}$: \[ \E_{\vec{v}\sim \dist}\left[{\normalfont \profit}_{\hat{M}_{\alpha}}(\vec{v})\right] \geq \alpha \cdot \E_{\vec{v}\sim \dist}\left[{\normalfont \profit}_{M^*}(\vec{v})\right] - \epsilon_{\pazocal{M}}\left(N, \frac{\delta}{2}\right) - U \alpha \sqrt{\frac{2\ln(4/\delta)}{2N}}.\]
\end{restatable}

\begin{proof}
Let $\sample = \left\{\vec{v}^1, \dots, \vec{v}^N\right\}$ be a set of samples of buyer valuations.
With probability at least $1-\delta$,
\begin{align*}
 &{\normalfont \profit}_{\pazocal{D}}\left(\hat{M}_{\alpha}\right) + \epsilon_{\pazocal{M}}\left(N, \frac{\delta}{2}\right) \geq   {\normalfont \profit}_{\sample}\left(\hat{M}_{\alpha}\right) \geq \alpha \cdot \max_{M \in \mclass} \profit_{\sample}(M)\\
 \geq  \text{ }&\alpha \cdot {\normalfont \profit}_{\sample}\left({M^*}\right) \geq \alpha \cdot {\normalfont \profit}_{\pazocal{D}}\left({M^*}\right) -  U \alpha \sqrt{\frac{2\ln(4/\delta)}{2N}}.
\end{align*}

These inequalities follow for the same reasons as in the proof of Corollary~\ref{cor:add_apx}.
\end{proof}

\begin{lemma}[\citet{Shalev14:Understanding}, Lemma A.2]\label{lem:log_ineq}
Let $a \geq 1$ and $b > 0$. Then $x < a\log x + b$ implies that $x < 4a \log (2a) + 2b$.
\end{lemma}

\begin{lemma}\label{lem:kappa_bnded_2pt}
No matter which parameters the mechanism designer chooses in $\pspace$ or $\pspace'$, if all buyers simultaneously choose the tariff and the number of units $(q_1, \dots, q_n)$ that maximize their utilities, then $\sum_{j = 1}^n q_j \leq \kappa$.
\end{lemma}

\begin{proof}
We prove this lemma for non-anonymous prices, and the lemma for anonymous prices follow since they are a special case of non-anonymous prices. For a contradiction, suppose there exists a set of buyers' values $\vec{v}$ and a non-anonymous menu of two-part tariffs with parameters in $\pspace'$ such that if $t_j$ is the tariff that buyer $j$ chooses and $q_j$ is the number of units he chooses, $\sum_{j = 1}^n q_j > \kappa$. Since the mechanisms are profit non-negative, we know that $\sum_{j = 1}^n p_{1,j}^{(t_j)} \cdot \textbf{1}_{\{q_j \geq 1\}} + p_{2,j}^{(t_j)} \cdot q_j - c\left(Q\right) \geq 0$, where $Q = (q_1, \dots, q_n)$. We also know that each buyer's value for the units he bought is greater than the price: $\sum_{j = 1}^n v_j(q_j) \geq \sum_{j = 1}^n p_{1,j}^{(t_j)} \cdot \textbf{1}_{\{q_j \geq 1\}} + p_{2,j}^{(t_j)} \cdot q_j$. Therefore, $\sum_{j = 1}^n v_j(q_j) - c(Q) \geq 0$. However, this contradicts Assumption~\ref{assumption:unit_cap}, so the lemma holds.
\end{proof}

\twoPart*

\begin{proof}
A length-$\ell$ menu of two-part tariffs is defined by $2\ell$ parameters. The first $2$ parameters (denoted $\left(p_0^{(1)}, p_1^{(1)}\right)$) define the first tariff in the menu, the second $2$ parameters (denoted $\left(p_0^{(2)}, p_1^{(2)}\right)$) define the second tariff in the menu, and so on. Buyer $j$ will prefer to buy $q \geq 1$ units using $i^{th}$ menu entry (defined by the parameters $\left(p_0^{(i)}, p_1^{(i)}\right)$)  so long as
$v_j(q)  - \left(p_0^{(i)} + p_1^{(i)}q\right)> v_j(q')  - \left(p_0^{(i')} + p_1^{(i')}q'\right)$ for any $i' \not=i$ and $q' \not=q$. In total, these inequalities define $O\left(n\left(\kappa \ell\right)^2\right)$ hyperplanes in $\R^{2\ell}$. In any region defined by these hyperplanes, the menu entries and quantities demanded by all $n$ buyers are fixed. In any such region, profit is linear in the fixed fees and unit prices.

In the case of non-anonymous reserve prices, the same argument holds, except that every length-$\ell$ menu of two-part tariffs is defined by $2n\ell$ parameters: for each buyer, we must set the fixed fee and unit price for each of the $\ell$ menu entries.
\end{proof}

\begin{lemma}\label{lem:kappa_bnded_NL}
No matter which parameters the mechanism designer chooses in $\pspace$ or $\pspace'$, if all buyers simultaneously choose the bundles that maximize their utilities, then $\sum_{j = 1}^n q_j[i] \leq \kappa_i$ for all $i \in [m]$.
\end{lemma}

\begin{proof}
We prove this lemma for non-anonymous prices, and the lemma for anonymous prices follow since they are a special case of non-anonymous prices. For a contradiction, suppose there exists a set of buyers' values $\vec{v}$ and a non-anonymous non-linear pricing mechanism with parameters in $\pspace'$ such that if $\vec{q}_j$ is the bundle buyer $j$ chooses, $\sum_{j = 1}^n q_j[i] > \kappa_i$ for some $i \in [m]$. Since the mechanisms are profit non-negative, we know that $\sum_{j = 1}^n p_j(\vec{q}_j)- c\left(Q\right) \geq 0$, where $Q = \left(\vec{q}_1, \dots, \vec{q}_n\right)$. We also know that each buyer's value for the units he bought is greater than the price: $\sum_{j = 1}^n v_j(\vec{q}_j) \geq \sum_{j = 1}^n p_j(\vec{q}_j)$. Therefore, $\sum_{j = 1}^n v_j(\vec{q}_j) - c(Q) \geq 0$. However, this contradicts Assumption~\ref{assumption:unit_cap}, so the lemma holds.
\end{proof}

\begin{definition}[Profit non-negative non-linear pricing mechanisms]\label{def:PNN_NL}
	In the case of anonymous prices (respectively, non-anonymous), let $\pspace$ (respectively, $\pspace'$) be the set of mechanism parameters such that for each buyer $j \in [n]$ and each allocation $Q = \left(\vec{q}_1, \dots, \vec{q}_n\right)$, the seller's utility is non-negative: $\sum_{j = 1}^n p_j(\vec{q}_j) - c\left(Q\right) \geq 0.$ The set of profit non-negative non-linear pricing mechanisms is defined by parameters in $\pspace$ (respectively, $\pspace'$).
\end{definition}

\nonlinear*

\begin{proof}
We begin by analyzing the case where there are anonymous prices. By Lemma~\ref{lem:kappa_bnded_NL}, the mechanism designer might as well set the price of any bundle $\vec{q}$ such that $q[i] \geq \kappa_i$ for some $i \in [m]$ to $\infty$. Therefore, every non-linear pricing mechanism is defined by $d = \prod_{i = 1}^m \left(\kappa_i+1\right)$ parameters because that is the number of different bundles and there is a price per bundle. Buyer $j$ will prefer the bundle corresponding to the quantity vector $\vec{q}$ over the bundle corresponding to the quantity vector $\vec{q}'$ if $v_j(\vec{q}) - p(\vec{q}) \geq v_j(\vec{q}') - p(\vec{q}')$. Therefore, there are at most $\prod_{i = 1}^m \left(\kappa_i+1\right)^2$ hyperplanes in $\R^{d}$ determining each buyer's preferred bundle --- one hyperplane per pair of bundles. This means that there are a total of $n\prod_{i = 1}^m \left(\kappa_i+1\right)^2$ hyperplanes in $\R^{d}$ such that in any one region induced by these hyperplanes, the bundles demanded by all $n$ buyers are fixed and profit is linear in the prices of these $n$ bundles.

In the case of non-anonymous prices, the same argument holds, except that every non-linear pricing mechanism is defined by $n\prod_{i = 1}^m \left(\kappa_i+1\right)$ parameters --- one parameter per bundle-buyer pair.
\end{proof}

\begin{definition}[Additively decomposable non-linear pricing mechanisms] Additively decomposable non-linear pricing mechanisms are a subset of non-linear pricing mechanisms where the prices are additive over the items. Specifically, if the prices are anonymous, there exist $m$ functions $p^{(i)}:[\kappa_i] \to \R$ for all $i \in [m]$ such that for every quantity vector $\vec{q}$, $p(\vec{q}) = \sum_{i: q[i] \geq 1} p^{(i)}(q[i])$. If the prices are non-anonymous, there exist $nm$ functions $p^{(i)}_j:[\kappa_i] \to \R$ for all $i \in [m]$ and $j \in [n]$ such that for every quantity vector $\vec{q}$, $p_j(\vec{q}) = \sum_{i: q[i] \geq 1} p^{(i)}_j(q[i])$.\end{definition}
\begin{restatable}{lemma}{addNonlinear}\label{lem:nonlinear_additive}
Let $\pazocal{M}$ and $\pazocal{M}'$ be the classes of additively decomposable non-linear pricing mechanisms with anonymous and non-anonymous prices, respectively. Then
$\mclass$ is \[\left(\sum_{i = 1}^m (\kappa_i+1), n\prod_{i = 1}^m \left(\kappa_i+1\right)^2\right)\text{-delineable}\] and $\mclass'$ is $\left(n\sum_{i = 1}^m \left(\kappa_i+1\right), n\prod_{i = 1}^m \left(\kappa_i+1\right)^2\right)$-delineable.
\end{restatable}

\begin{proof}
In the case of anonymous prices, any additively decomposable non-linear pricing mechanism is defined by $d = \sum_{i = 1}^m (\kappa_i+1)$ parameters. As in the proof of Lemma~\ref{lem:nonlinear}, there are a total of $n\prod_{i = 1}^m (\kappa_i+1)^2$ hyperplanes in $\R^{d}$ such that in any one region induced by these hyperplanes, the bundles demanded by all $n$ buyers are fixed and profit is linear in the prices of these $n$ bundles.

In the case of non-anonymous prices, the same argument holds, except that every non-linear pricing mechanism is defined by $n\sum_{i = 1}^m (\kappa_i+1)$ parameters --- one parameter per item, quantity, and buyer tuple.
\end{proof}

\itemAdd*

\begin{proof}
In the case of anonymous prices, every item-pricing mechanisms is defined by $m$ prices $\vec{p} \in \R^m$, so the parameter space is $\R^m$. Let $j_i$ be the buyer with the highest value for item $i$. We know that item $i$ will be bought so long as $v_{j_i}(\vec{e}_i) \geq p(\vec{e}_i)$. Once the items bought are fixed, profit is linear. Therefore, there are $m$ hyperplanes splitting $\R^m$ into regions where profit is linear.

In the case of non-anonymous prices, the parameter space is $\R^{nm}$ since there is a price per buyer and per item. The items each buyer $j$ is willing to buy is defined by $m$ hyperplanes: $v_j(\vec{e}_i) \geq p_j(\vec{e}_i)$. So long as these preferences are fixed, profit is a linear function of the prices. Therefore, there are $nm$ hyperplanes splitting $\R^{nm}$ into regions where profit is linear.
\end{proof}

\secondPrice*

\begin{proof}
For a given valuation vector $\vec{v}$, let $j_i$ be the highest bidder for item $i$ and let $j_i'$ be the second highest bidder.
Under anonymous prices, item $i$ will be bought so long as $v_{j_i}(\vec{e}_i) \geq p(\vec{e}_i)$. If buyer $j_i$ buys item $i$, his payment depends on whether or not $v_{j_i'}(\vec{e}_i) \geq p(\vec{e}_i)$. Therefore, there are $t = 2m$ hyperplanes splitting $\R^m$ into regions where profit is linear.
In the case of non-anonymous prices, the only difference is that the parameter space is $\R^{nm}$. 
\end{proof}

\begin{definition}[Mixed-bundling auctions with reserve prices (MBARPs)]\label{def:MBARP}
	MBARPs are defined by a parameter $\gamma \geq 0$ and $m$ reserve prices $p\left(\vec{e}_1\right), \dots, p\left(\vec{e}_m\right)$. Let $\lambda$ be a function such that $\lambda\left(Q\right) = \gamma$ if some buyer receives the grand bundle under allocation $Q$ and 0 otherwise. For an allocation $Q$, let $\vec{q}_Q$ be the items not allocated. Given a valuation vector $\vec{v}$, the MBARP allocation is \[Q^* = \left(\vec{q}_1^*, \dots, \vec{q}_n^*\right) = \text{argmax}\left\{\sum_{j = 1}^n v_j\left(\vec{q}_j\right) + \sum_{i : q_Q[i] = 1} p\left(\vec{e}_i\right) + \lambda\left(Q\right) - c\left(Q\right)\right\}.\] 
	Using the notation \[Q^{-j} = \left(\vec{q}_1^{-j}, \dots, \vec{q}_n^{-j}\right) = \text{argmax}\left\{ \sum_{\ell \not = j} v_\ell\left(\vec{q}_\ell\right) + \sum_{i : q_Q[i] = 1} p\left(\vec{e}_i\right) + \lambda\left(Q\right) - c\left(Q\right)\right\},\] buyer $j$ pays \[\sum_{\ell \not= j} v_\ell\left(\vec{q}_\ell^{-j}\right) + \sum_{i : q_{Q^{-j}}[i] = 1} p\left(\vec{e}_i\right) + \lambda\left(Q^{-j}\right) - c\left(Q^{-j}\right) - \sum_{\ell \not= j} v_\ell\left(\vec{q}^*_\ell\right) - \sum_{i : q_{Q^*}[i] = 1} p\left(\vec{e}_i\right) - \lambda\left(Q^*\right) + c\left(Q^*\right).\]
\end{definition}
 
 \MBARP*
 
 \begin{proof}
An MBARP is defined by $m+1$ parameters since there is one reserve per item and one allocation boost. Let $K = (n+1)^m$ be the total number of allocations. 
Fix some valuation vector $\vec{v}$. We claim that the allocation of any MBARP is determined by at most $(n+1)K^2$ hyperplanes in $\R^{m+1}$. To see why this is, let $Q^k = \left(\vec{q}_1^k, \dots, \vec{q}_n^k\right)$ and $Q^{\ell} = \left(\vec{q}_1^{\ell}, \dots, \vec{q}_n^{\ell}\right)$ be any two allocations and let $\vec{q}_{Q^k}$ and $\vec{q}_{Q^{\ell}}$ be the bundles of items not allocated. Consider the ${K \choose 2}$ hyperplanes defined as \[\sum_{i = 1}^n v_i\left(\vec{q}_i^{\ell}\right) + \sum_{j : q_{Q^{\ell}}[i] = 1} p\left(\vec{e}_i\right) + \lambda\left(Q^{\ell}\right) - c\left(Q^{\ell}\right) = \sum_{i = 1}^n v_i\left(\vec{q}_i^{k}\right) + \sum_{j : q_{Q^{k}}[i] = 1} p\left(\vec{e}_i\right) + \lambda\left(Q^{k}\right) - c\left(Q^{k}\right).\] In the intersection of these ${K \choose 2}$ hyperplanes, the allocation of the MBARP is fixed.

By a similar argument, it is straightforward to see that $K^2$ hyperplanes determine the allocation of any MBARP in this restricted space without any one bidder's participation. This leads us to a total of $(n+1)K^2$ hyperplanes which partition the space of MBARP parameters in a way such that for any two parameter vectors in the same region, the auction allocations are the same, as are the allocations without any one bidder's participation.  Once these allocations are fixed, profit is a linear function in this parameter space.
\end{proof}

\begin{definition}[Affine maximizer auction]\label{def:AMA}
	An AMA is defined by a weight per buyer $w_j \in \R_{> 0}$ and a boost per allocation $\lambda\left(Q\right) \in \R_{\geq 0}$. The AMA allocation $Q^*$ is the one which maximizes the weighted social welfare, i.e., $Q^* =   \left(\vec{q}_1^*, \dots, \vec{q}_n^*\right) = \text{argmax}\left\{\sum_{j = 1}^n w_jv_j\left(\vec{q}_j\right) + \lambda\left(Q\right) - c\left(Q\right)\right\}.$ Using the notation \[Q^{-j} = \left(\vec{q}_1^{-j}, \dots, \vec{q}_n^{-j}\right) = \text{argmax}\left\{ \sum_{\ell \not= j} w_{\ell}v_{\ell}\left(\vec{q}_{\ell}\right) + \lambda\left(Q\right) - c\left(Q\right)\right\},\] each buyer $j$ pays 
	\[\frac{1}{w_j}\left[ \sum_{\ell \not= j} w_{\ell}v_{\ell}\left(\vec{q}_{\ell}^{-j}\right) + \lambda\left(Q^{-j}\right) - c\left(Q^{-j}\right)-\left(\sum_{\ell \not= j} w_{\ell} v_{\ell}\left(\vec{q}^*_{\ell}\right) +\lambda\left(Q^*\right) - c\left(Q^*\right)\right)\right].\]
	\end{definition}

\AMA*

\begin{proof}
Let $K = (n+1)^m$ be the total number of allocations and let $\vec{p}$ be a parameter vector where the first $n$ components correspond to the bidder weights $w_j$ for $j \in [n]$, the next $n$ components correspond to $1/w_j$ for $j \in [n]$, the next $2{n \choose 2}$ components correspond to $w_i/w_j$ for all $i \not= j$, the next $K$ components correspond to $\lambda(Q)$ for every allocation $Q$, and the final $nK$ components correspond to $\lambda(Q)/w_j$ for all allocations $Q$ and all bidders $j \in [n]$. In total, the dimension of this parameter space is at most $2n + 2n^2 + K + nK = O(nK)$. Let $\vec{v}$ be a valuation vector. We claim that this parameter space can be partitioned using $t = (n+1)K^2$ hyperplanes into regions where in any one region $\pspace'$, there exists a vector $\vec{k}$ such that $\profit_{\vec{v}}(\vec{p}) = \vec{k} \cdot \vec{p}$ for all $\vec{p} \in\pspace'$.

To this end, an allocation $Q = \left(\vec{q}_1, \dots, \vec{q}_n\right)$ will be the allocation of the AMA so long as $\sum_{i = 1}^n w_i v_i\left(\vec{q}_i\right) + \lambda(Q) - c(Q) \geq \sum_{i = 1}^n w_i v_i\left(\vec{q}_i'\right) + \lambda\left(Q'\right) - c(Q')$ for all allocations $Q' = \left(\vec{q}_1', \dots, \vec{q}_n'\right) \not= Q$. Since the number of different allocations is at most $K$, the allocation of the auction on $\vec{v}$ is defined by at most $K^2$ hyperplanes in $\R^{d}$. Similarly, the allocations $Q^{-1}, \dots, Q^{-n}$ are also determined by at most $K^2$ hyperplanes in $\R^{d}$.
Once these allocations are fixed, profit is a linear function of this parameter space.

The proof for VVCAs follows the same argument except that we redefine the parameter space to consist of vectors where the first $n$ components correspond to the bidder weights $w_j$ for $j \in [n]$, the next $n$ components correspond to $1/w_j$ for $j \in [n]$, the next $2{n \choose 2}$ components correspond to $w_i/w_j$ for all $i \not= j$, the next $K' = n2^m$ components correspond to the bidder-specific bundle boosts $c_{j,\vec{q}}$ for every quantity vector $\vec{q}$ and bidder $j \in [n]$, and the final $nK'$ components correspond to $c_{k,\vec{q}}/w_j$ for every quantity vector $\vec{q}$ and every pair of bidders $j,k \in [n]$. The dimension of this parameter space is at most $2n + 2n^2 + K' + nK' \leq 2K' + nK' + K' + nK' = O(nK')$.

Finally, the proof for $\lambda$-auctions follows the same argument as the proof for AMAs except there are zero bidder weights. Therefore, the parameter space consists of vectors with $K$ components corresponding to $\lambda(Q)$ for every allocation $Q$.
\end{proof} 

\begin{lemma}\label{lem:expectation}
For all $\vec{v} \in \domain$ and all $M \in \mclass$, $\profit_M(\vec{v}) = \E_{\vec{z}}\left[\profit_{M} '\left(\vec{v}, \vec{z}\right)\right]$.
\end{lemma}

\begin{proof} By definition of $\profit_m'$,
\begin{align*}
&\E_{\vec{z}}\left[\profit_{M} '\left(\vec{v}, \vec{z}\right)\right]\\
=\text{ }&\E_{\vec{z}}\left[p_{\vec{v}} - c\left(\sum_{j: z[j] < \phi_{\vec{v}}[j]} \vec{e}_j\right)\right]\\
=\text{ }& p_{\vec{v}} - \sum_{\vec{r} \in \{0,1\}^m}c\left(\vec{r}\right)\prod_{j: r[j] = 1}\Pr\left[z[j] < \phi_{\vec{v}}[j] \right]\prod_{j: r[j] = 0}\Pr\left[z[j] \geq \phi_{\vec{v}}[j] \right]\\
=\text{ }& p_{\vec{v}} - \sum_{\vec{r} \in \{0,1\}^m}c\left(\vec{r}\right)\prod_{j: r[j] = 1}\phi_{\vec{v}}[j]\prod_{j: r[j] = 0}\left(1-\phi_{\vec{v}}[j]\right).
\end{align*}

From the other direction, \begin{align*}
\profit_{M} \left(\vec{v}\right) &= p_{\vec{v}} - \E_{\vec{q} \sim \vec{\phi}_{\vec{v}}}\left[c(\vec{q})\right]\\
&= p_{\vec{v}} - \sum_{\vec{r} \in \{0,1\}^m}c\left(\vec{r}\right)\prod_{j: r[j] = 1}\Pr\left[q[j] = 1 \right]\prod_{j: r[j] = 0}\Pr\left[q[j] = 0 \right]\\
&= p_{\vec{v}} - \sum_{\vec{r} \in \{0,1\}^m}c\left(\vec{r}\right)\prod_{j: r[j] = 1}\phi_{\vec{v}}[j]\prod_{j: r[j] = 0}\left(1-\phi_{\vec{v}}[j]\right).
\end{align*} Therefore, $\profit_M(\vec{v}) = \E_{\vec{z}}\left[\profit_{M} '\left(\vec{v}, \vec{z}\right)\right]$.
\end{proof}

\lotteryEquiv*

\begin{proof}
We know that with probability at least $1-\delta$ over the draw of a sample \[\left\{\left(\vec{v}^{(1)}, \vec{z}^{(1)}\right), \dots, \left(\vec{v}^{(N)}, \vec{z}^{(N)}\right)\right\} \sim \left(\dist \times U([0,1])^{m}\right)^N,\] for all mechanisms $M \in \mclass$, \begin{align*}&\left|\frac{1}{N} \sum_{j = 1}^N \profit_{M} '\left(\vec{v}^{(j)}, \vec{z}^{(j)}\right) - \E_{\vec{v}, \vec{z} \sim \dist \times U([0,1])^{m}}\left[\profit_M'(\vec{v}, \vec{z})\right] \right|\\
= \text{ }&O\left(U \sqrt{\frac{Pdim(\mclass')}{N}} + U \sqrt{\frac{\log(1/\delta)}{N}}\right).\end{align*} We also know from Lemma~\ref{lem:expectation} that \[\E_{\vec{v}, \vec{z} \sim \dist \times U([0,1])^{m}}\left[\profit_M'(\vec{v}, \vec{z})\right] = \E_{\vec{v} \sim \dist}\left[\profit_M(\vec{v})\right].\] Therefore, the theorem statement holds.
\end{proof}

\lottery*

\begin{proof}
	A length-$\ell$ lottery menu is defined by $\ell(m+1)$ parameters. The first $m+1$ parameters (denoted $\left(\phi^{(1)}[1], \dots, \phi^{(1)}[m], p^{(1)}\right)$) define the first lottery in the menu, the second $m+1$ parameters (denoted $\left(\phi^{(2)}[1], \dots, \phi^{(2)}[m], p^{(2)}\right)$) define the second lottery in the menu, and so on. The buyer will prefer the $j^{th}$ menu entry (defined by the parameters $\left(\phi^{(j)}[1], \dots, \phi^{(j)}[m], p^{(j)}\right)$)  so long as
	$\vec{v} \cdot \vec{\phi}^{(j)} - p^{(j)}> \vec{v} \cdot \vec{\phi}^{(k)}- p^{(k)}$ for any $k \not=j$. In total, these inequalities define ${\ell+1 \choose 2}$ hyperplanes in $\R^{\ell(m+1)}$. In any region defined by these hyperplanes, the menu entry that the buyer prefers is fixed. Next, for each menu entry $\left(\vec{\phi}^{(k)}, p^{(k)}\right)$, there are $m$ hyperplanes determining the vector $\sum_{j: w[j] < \phi^{(k)}[j]} \vec{e}_j$, and thus the cost  $c \left(\sum_{j: w[j] < \phi^{(k)}[j]} \vec{e}_j\right)$. These vectors have the form $w[j] = \phi^{(k)}[j].$ Thus, there are a total of $\ell m$ hyperplanes determining the costs. Let $\hyp$ be the union of all $(\ell+1)^2 + m \ell$ hyperplanes. Within any connected component of $\R^{\ell(m+1)}\setminus \hyp$, the menu entry that the buyer buys is fixed and for each menu entry, $c \left(\sum_{j: w[j] < \phi^{(k)}[j]} \vec{e}_j\right)$ is fixed. Therefore, profit is a linear function of the prices $p^{(1)}, \dots, p^{(\ell)}$.
\end{proof}

\subsection{Additional lottery results}\label{sec:additional_lotteries}

\paragraph{Lotteries for a unit-demand buyer.} Recall that if the buyer is unit-demand, then for any bundle $\vec{q} \in \{0,1\}^m$, $v_1\left(\vec{q}\right) = \max_{i : q[i] \geq 1} v_1\left(\vec{e}_i\right)$. We assume that under a lottery $\left(\phi^{(j)}, p^{(j)}\right)$ with a unit-demand buyer, the buyer will only receive one item, and the probability that item is item $i$ is $\phi^{(j)}[i]$. Thus, we assume that $\sum_{i = 1}^m \phi^{(j)}[i] \leq 1$. Since $v_1(\vec{e}_i)\cdot \phi^{(j)}[i]$ is their value for item $i$ times the probability they get that item, their expected utility is $\sum_{i = 1}^m v_1(\vec{e}_i) \cdot \phi^{(j)}[i] - p^{(j)}$, as in the case with an additive buyer. Therefore, the following theorem follows by the exact same proof as Lemma~\ref{lem:lottery}.

\begin{theorem}
Let $\mclass'$ be the class of functions defined in Section~\ref{sec:delineable}.4. Then $\mclass'$ is \[\left(\ell\left(m+1\right), \left(\ell+1\right)^2 + m\ell\right)\text{-delineable}.\]
\end{theorem}

\paragraph{Lotteries for multiple unit-demand or additive buyers.} In order to generalize to multi-buyer settings, we assume that there are $n$ units of each item for sale and that each buyer will receive at most one unit of each item. The buyers arrive simultaneously and each will buy the lottery that maximizes her expected utility. Thus, the following is a corollary of Lemma~\ref{lem:lottery}.

\begin{theorem}
Let $\mclass'$ be the class of functions defined in Section~\ref{sec:delineable}.4. Then $\mclass'$ is \[\left(\ell\left(m+1\right), n\left(\left(\ell+1\right)^2 + m\ell\right)\right)\text{-delineable.}\]
\end{theorem}
\subsection{Mixed bundling auctions}\label{app:MBA}

Mixed bundling auctions (MBAs) are defined by a single parameter $\gamma$. They correspond to a $\lambda$-auction where $\lambda(Q) = \gamma$ if some buyer receives the grand bundle under allocation $Q$ and 0 otherwise. The class of MBAs is particularly simple, and we prove an even tighter bound on the Rademacher complexity of MBAs than that guaranteed by Theorem~\ref{thm:main_pdim}. Our analysis requires us to understand how the profit of a $\gamma$-MBA on a single bidding instance changes as a function of $\gamma$. We take advantage of this function's structural properties, first uncovered by \citet{Jehiel07:Mixed}: no matter the number of buyers and no matter the number of items, there exists an easily characterizable value $\gamma^*$ such that the function in question is increasing as $\gamma$ grows from 0 to $\gamma^*$, and then it is non-increasing as $\gamma$ grows beyond $\gamma^*$. This is depicted in Figure~\ref{fig:revGraphn}.
\begin{figure}
  \centering
  \includegraphics[scale=1]{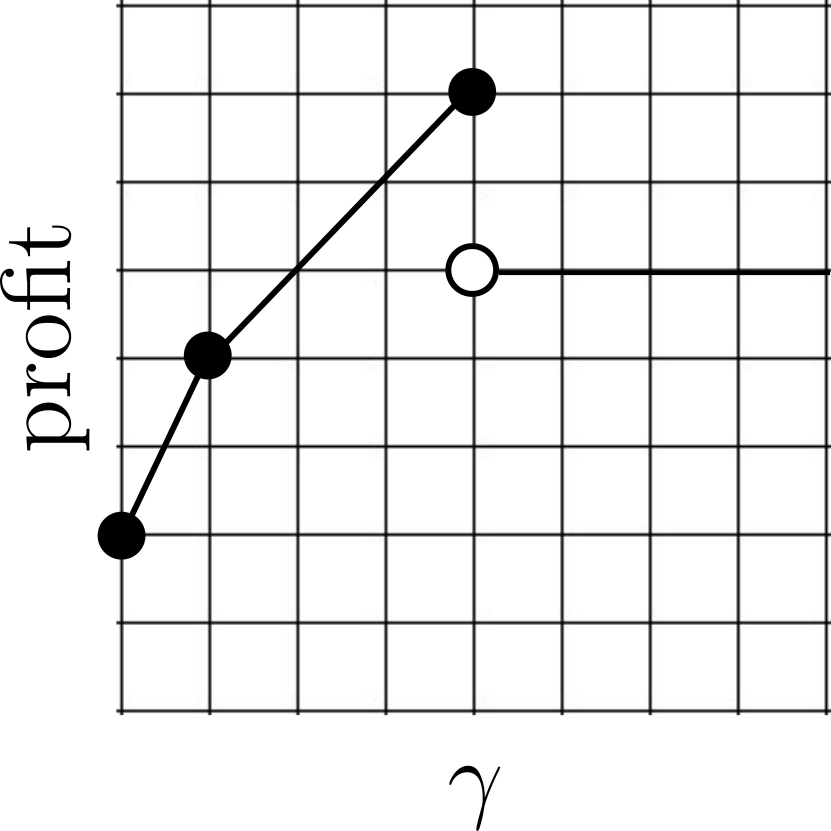}
  \caption{An example of the $\gamma$-MBA revenue of a single bidding instance as $\gamma$ varies.}
  \label{fig:revGraphn}
\end{figure}
Intuitively, $\gamma^*$ represents the number at which $\gamma$ has grown so large that the MBA has morphed into a second price auction on the grand bundle. As a result, no matter how much larger $\gamma$ grows beyond $\gamma^*$, the value of $\gamma$ no longer factors into the profit function. This simple structure allows us to prove the strong generalization guarantee described in Theorem~\ref{thm:MBA}.

\begin{theorem}\label{thm:MBA}
Let $\pazocal{M}$ be the class of MBAs. Then $\pdim(\mclass) = 2$.
\end{theorem}

\begin{proof} 
First, we show that the pseudo-dimension of the class of $n$-buyer, $m$-item MBAs is at most 2. Let $\sample = \left\{\vec{v}^{(1)}, \dots, \vec{v}^{(N)}\right\}$ be a set of $n$-buyer valuation functions that can be shattered by a set $\Gamma$ of $2^N$ MBAs. This means that there exist $N$ witnesses $z^{(1)}, \dots, z^{(N)}$ such that each MBA in $\Gamma$ induces a binary labeling of the samples $\vec{v}^{(j)}$ of $\sample$ (whether the profit of the MBA on $\vec{v}^{(j)}$ is at least $z_j$ or strictly less than $z^{(j)}$). Since $\sample$ is shatterable, we can thus label $\sample$ in every possible way using MBAs in $\Gamma$.

Now, fix one sample $\vec{v}^{(i)} \in \sample$. We denote the profit of the $\gamma$-MBA on $\vec{v}^{(i)}$ as a function of $\gamma$ as ${\normalfont \profit}_{\vec{v}^{(i)}}(\gamma)$. From Lemma~\ref{lem:rev_struct}, we know that there exists $\gamma^*_i \in [0,\infty)$, such that ${\normalfont \profit}_{\vec{v}^{(i)}}(\gamma)$ is non-decreasing on the interval $[0,\gamma^*_i]$ and non-increasing on the interval $(\gamma^*_i, \infty)$. Therefore, there exist two thresholds $t^{(1)}_i \in [0,\gamma^*_i]$ and $t^{(2)}_i \in (\gamma^*_i, \infty) \cup \{\infty\}$ such that ${\normalfont \profit}_{\vec{v}^{(i)}}(\gamma)$ is below its threshold for $\gamma \in [0,t^{(1)}_i)$, above its threshold for $\gamma \in (t^{(1)}_i, t^{(2)}_i)$, and below its threshold for $\gamma \in (t^{(2)}_i, \infty)$. Now, merge these thresholds for all $N$ samples on the real line and consider the interval $(t_1,t_2)$ between two adjacent thresholds. The binary labeling of the samples in $\sample$ on this interval is fixed. In other words, for any sample $\vec{v}^{(j)} \in \sample$, ${\normalfont \profit}_{\vec{v}^{(j)}}(\gamma)$ is either at least $z^{(j)}$ or strictly less than $z^{(j)}$ for all $\gamma \in (t_1,t_2)$. There are at most $2N+1$ intervals between adjacent thresholds, so at most $2N+1$ different binary labelings of $\sample$. Since we assumed $\sample$ is shatterable, it must be that $2^N \leq 2N+1$, so $N \leq 2.$

Finally, we show that the pseudo-dimension of the class of $n$-buyer, $m$-item MBAs is at least 2 by constructing a set $\sample = \left\{\vec{v}^{(1)}, \vec{v}^{(2)}\right\}$ that can be shattered by the set of MBAs. To construct this set of samples $\sample$, let \[v_1^{(1)}\left(\vec{q}\right) = v_2^{(1)}\left(\vec{q}\right) = \begin{cases} 0 &\text{if } ||\vec{q}||_1 < \lfloor m/2 \rfloor\\
3 &\text{if } \lfloor m/2 \rfloor \leq ||\vec{q}||_1 \end{cases} \text{ and } v_1^{(2)}\left(\vec{q}\right) = v_2^{(2)}\left(\vec{q}\right) = \begin{cases} 0 &\text{if } ||\vec{q}||_1 < \lfloor m/2 \rfloor\\
3 &\text{if } \lfloor m/2 \rfloor \leq ||\vec{q}||_1 < m\\
4 &\text{if } ||\vec{q}||_1 = m. \end{cases}\] Finally, let buyers 3 through $n$ have all-zero valuations in both $\vec{v}^{(1)}$ and $\vec{v}^{(2)}$ and let the cost function be 0 for all allocations.

Now, let $z^{(1)} = 3$ and $z^{(2)} = 4$. We define four MBAs parameterized by the coefficients $\gamma_1 = 0, \gamma_2 = 1.5, \gamma_3 = 1.75, \gamma_4  = 2.5.$ It is easy to check that this set of MBAs shatters $\sample$, witnessed by $z^{(1)}$ and $z^{(2)}$. For example, see Table~\ref{tab:shattering}.
\begin{table}\centering
{\begin{tabular}{lll}
\textbf{$\gamma$ value} & \textbf{Profit on $\vec{v}^1$} & \textbf{Profit on $\vec{v}^2$} \\\hline
0                 & $0\leq z^{(1)}$                               & $2 \leq z^{(2)}$                               \\
1.5                & $3 \leq z^{(1)}$                               & $5 > z^{(2)}$                            \\
1.75                  & $3.5 > z^{(1)}$                               & $5.5 > z^{(2)}$
\\
2.5                & $5 > z^{(1)}$                               & $4 \leq z^{(2)}$                               \\
\end{tabular}}
\caption{Example of a shattered set of size 2\label{tab:shattering}}
\end{table}

The generalization guarantee follows from Theorem~\ref{thm:pdim}.
\end{proof}

\begin{lemma}\label{lem:rev_struct}
For a valuation vector $\vec{v}$, let $\profit_{\vec{v}}(\gamma)$ be the profit of the $\gamma$-MBA on $\vec{v}$ as a function of $\gamma$. There exists $\gamma^* \in [0,\infty)$ such that ${\normalfont \profit}_{\vec{v}}(\gamma)$ is non-decreasing on the interval $[0,\gamma^*]$ and non-increasing on the interval $(\gamma^*, \infty)$.
\end{lemma}

For additive buyers, this lemma is implied by Theorem 1 in the paper by \citet{Jehiel07:Mixed} which provides the derivative of $\profit_{\vec{v}}(\gamma)$. The techniques used by \citet{Jehiel07:Mixed} extend immediately to general buyers as well, as we show here.

\begin{proof}[Proof of Lemma~\ref{lem:rev_struct}]
We will show that ${\normalfont \profit}_{\vec{v}}$ can be decomposed into simple components, each of which can be easily analyzed on its own, and by combining these analyses, we prove the lemma statement. Suppose $Q^* = \left(\vec{q}_1^*, \dots, \vec{q}_n^*\right)$ is the resulting allocation of a certain $\gamma$-MBA $M$ and $Q^{-i} = \left(\vec{q}_1^{-i}, \dots, \vec{q}_n^{-i}\right)$ is the boosted social-welfare maximizing allocation without buyer $i$'s participation. More explicitly, $Q^* = \text{argmax} \left\{\sum_{i = 1}^n v_i\left(\vec{q}_i\right) + \lambda\left(Q\right) - c(Q) \right\}$ and $Q^{-i} = \text{argmax} \left\{\sum_{k \not=i} v_k\left(\vec{q}_k\right) + \lambda\left(Q\right) - c(Q)\right\}$, where $\lambda\left(Q\right)$ is set according to the MBA allocation boosting rule for all $Q$. Then buyer $i$ pays \[p_{i, \vec{v}}\left(\gamma\right) = v_i\left(\vec{q}^*_i\right) - \left[\sum_{j = 1}^n v_j\left(\vec{q}^*_j\right) + \lambda\left(Q^*\right) - c(Q^*) - \left(\sum_{j \not= i} v_j\left(\vec{q}^{-i}_j\right) + \lambda\left(Q^{-i}\right) - c(Q^{-i}) \right)\right].\] This means that \begin{align*}&{\normalfont \profit}_{\vec{v}}(\gamma) = \sum_{i = 1}^n p_{i, \vec{v}}\left(\gamma\right)\\
= \text{ } &(1 - n)\sum_{i = 1}^n v_i\left(\vec{q}_i^*\right) - n\left(\lambda\left(Q^*\right) - c\left(Q^*\right)\right) + \sum_{i = 1}^n \sum_{j \not= i} v_j\left(\vec{q}_j^{-i}\right) + \lambda\left(Q^{-i}\right) - c\left(Q^{-i}\right).\end{align*}

The profit function can be split into $n+1$ functions:
$f_{i,\vec{v}}(\gamma) = \sum_{j \not= i} v_j\left(\vec{q}^{-i}_j\right) + \lambda\left(Q^{-i}\right) - c\left(Q^{-i}\right)$ for $i \in \{1, \dots, n\}$ and $g_{\vec{v}}(\gamma) = (1 - n)\sum_{i = 1}^n v_i\left(\vec{q}_i^*\right) - n\left(\lambda\left(Q^*\right) - c\left(Q^*\right)\right).$ We claim that $f_{i,\vec{v}}(\gamma)$ is continuous for all $i$, whereas $g_{\vec{v}}(\gamma)$ has at most one discontinuity. This means that ${\normalfont \profit}_{\vec{v}}(\gamma) = \sum_{i = 1}^n f_{i,\vec{v}}(\gamma) + g_{\vec{v}}(\gamma)$ has at most one discontinuity as well. Moreover, the slope of $\sum_{i = 1}^n f_{i,\vec{v}}(\gamma)$ is between zero and $n$, whereas the slope of $g_{\vec{v}}(\gamma)$ is zero until its discontinuity, and then is $-n$. Therefore, the slope of ${\normalfont \profit}_{\vec{v}}(\gamma)$ is at least zero before its discontinuity and at most zero after its discontinuity. This is enough to prove the lemma statement.

To see why these properties are true for the functions $f_{i,\vec{v}}(\gamma)$, first let $\tilde{Q}^{-i} = \left(\tilde{\vec{q}}^{-i}_1, \dots, \tilde{\vec{q}}^{-i}_n\right)$ be the VCG allocation without buyer $i$, i.e., $\tilde{Q}^{-i} = \text{argmax} \left\{\sum_{k \not=i} v_k\left(\vec{q}_k\right) - c(Q)\right\}$. If one buyer is allocated the grand bundle in allocation $\tilde{Q}^{-i}$, then this allocation will only be more valuable as $\gamma$ grows, so $\tilde{Q}^{-i} = \text{argmax} \left\{\sum_{k \not=i} v_k\left(\vec{q}_k\right) + \lambda\left(Q\right) - c(Q)\right\}$ for all values of $\gamma$, which means that $f_{i,\vec{v}}(\gamma) = \sum_{j \not= i} v_j\left(\tilde{\vec{q}}_j^{-i}\right) + \lambda\left(\tilde{Q}^{-i}\right) - c\left(\tilde{Q}^{-i}\right) = \sum_{j \not= i} v_j\left(\tilde{\vec{q}}_j^{-i}\right) + \gamma - c\left(\tilde{Q}^{-i}\right)$ for all values of $\gamma$ as well. Clearly, in this case, $f_{i,\vec{v}}(\gamma)$ is increasing and continuous. Otherwise, using the notation $c^1$ to denote the cost of producing the grand bundle, we know there exists some value $\gamma_i$ such that $\sum_{j \not= i} v_j\left(\tilde{\vec{q}}_j^{-i}\right) + \lambda\left(\tilde{Q}^{-i}\right) - c\left(\tilde{Q}^{-i}\right) = \sum_{j \not= i} v_j\left(\tilde{\vec{q}}_j^{-i}\right) - c\left(\tilde{Q}^{-i}\right) \geq \max_{k \not= i} \left\{v_k\left(\vec{1}\right) + \gamma - c^1\right\}$ if $\gamma \leq \gamma_i$ and
$\sum_{j \not= i} v_j\left(\tilde{\vec{q}}_j^{-i}\right) - c\left(\tilde{Q}^{-i}\right) < \max_{k \not= i} \left\{v_k\left(\vec{1}\right) + \gamma - c^1\right\}$ if $\gamma > \gamma_i.$ This means that $\tilde{Q}^{-i}$ is the allocation of the $\gamma$-MBA without buyer $i$'s participation for $\gamma \leq \gamma_i$, and the allocation of the $\gamma$-MBA without buyer $i$'s participation for $\gamma > \gamma_i$ is the one where the highest buyer for the grand bundle (excluding buyer $i$) wins the grand bundle. Therefore, \[f_{i,\vec{v}}(\gamma) = \begin{cases} \sum_{j \not= i} v_j\left(\tilde{\vec{q}}_j^{-i}\right) - c\left(\tilde{Q}^{-i}\right) & \text{if } \gamma \leq \gamma_i\\ \max_{k \not= i} \left\{v_k\left(\vec{1}\right) + \gamma - c^1\right\} &\text{if } \gamma > \gamma_i.\end{cases}\] Notice that $\sum_{j \not= i} v_j\left(\tilde{\vec{q}}_j^{-i}\right) - c\left(\tilde{Q}^{-i}\right) = \max_{k \not= i} \left\{v_k\left(\vec{1}\right) + \gamma_i - c^1\right\}$, so $f_{i,\vec{v}}(\gamma)$ is continuous. Finally, it is clear that the slope of each $f_{i,\vec{v}}(\gamma)$ is between 0 and 1, so the slope of $\sum_{i = 1}^n f_{i,\vec{v}}(\gamma)$ is between 0 and $n$.

Similarly, let $\tilde{Q} = \left(\tilde{\vec{q}}_1, \dots, \tilde{\vec{q}}_n\right)$ be the allocation of the VCG mechanism run on $\vec{v}$. Then there exists some $\gamma^*$ such that $\tilde{Q}$ is the allocation of the $\gamma$-MBA for $\gamma \leq \gamma^*$ and the allocation of the $\gamma$-MBA for $\gamma > \gamma^*$ is the one where the highest bidder for the grand bundle wins the grand bundle. More explicitly, $\sum_{i=1}^n v_i\left(\tilde{\vec{q}}_i\right) + \lambda\left(\tilde{Q}\right) - c\left(\tilde{Q}\right) \geq \max_{k \in [n]}\left\{v_k\left(\vec{1}\right) + \gamma - c^1\right\}$ if $\gamma \leq \gamma^*$ and
$\sum_{i=1}^n v_i\left(\tilde{\vec{q}}_i\right) + \lambda\left(\tilde{Q}\right) - c\left(\tilde{Q}\right) < \max_{k \in [n]}\left\{v_k\left(\vec{1}\right) + \gamma - c^1\right\}$ if $\gamma > \gamma^*$. Therefore, \[g_{\vec{v}}(\gamma) = \begin{cases} (1-n)\sum_{i=1}^n v_i\left(\tilde{\vec{q}}_i\right) - n\left(\lambda\left(\tilde{Q}\right) - c\left(\tilde{Q}\right)\right) & \text{if } \gamma \leq \gamma^*\\ (1-n)\max \left\{v_k\left(\vec{1}\right)\right\} - n\left(\gamma - c^1\right) &\text{if } \gamma > \gamma^*.\end{cases}\] Therefore, $g_{\vec{v}}(\gamma)$ has at most one discontinuity, which falls at $\gamma^*$. Moreover, the slope of $g_{\vec{v}}(\gamma)$ is 0 for $\gamma < \gamma^*$ and $-n$ for $\gamma > \gamma^*$. As described, these properties of $f_{i,\vec{v}}(\gamma)$ and $g_{\vec{v}}(\gamma)$ are enough to show that the lemma statement holds.
\end{proof}

\subsection{Proof of Theorem~\ref{thm:lower}}\label{app:lower}

\begin{theorem}
	\label{thm:lower}
	For a class of auctions $\pazocal{M}$, let $N_{\pazocal{M}}(\epsilon, \delta)$ be the number of samples required to ensure that for any distribution $\pazocal{D}$, with probability at least $1-\delta$ over the draw of a set of samples of size $N_{\pazocal{M}}(\epsilon, \delta)$ from $\pazocal{D}$, for all auctions $M \in \pazocal{M}$, average profit is $\epsilon$-close to expected profit. 
	\begin{enumerate}
		\item If $\pazocal{M}$ is the set of AMAs or $\lambda$-auctions, then $N_{\pazocal{M}}(\epsilon, \delta) \geq \frac{n^m - n}{2}.$
		\item If $\pazocal{M}$ is the set of VVCAs, then $N_{\pazocal{M}}(\epsilon, \delta) \geq 2^m - 2.$
	\end{enumerate}
\end{theorem}

First, we prove part 1 and then we prove part 2.

\bigskip\emph{Lower Bound on Sample Complexity for $\lambda$-Auctions.}
We prove that $N_{\pazocal{M}}(\epsilon, \delta) \geq \frac{n^m-n}{2}$ samples are required to ensure that for any distribution $\pazocal{D}$, with probability at least $1-\delta$ over the draw of a set of samples of size $N_{\pazocal{M}}(\epsilon, \delta)$ from $\pazocal{D}$, for all $\lambda$-auctions $M \in \pazocal{M}$, average profit is $\epsilon$-close to expected profit.
 Since $\lambda$-auctions are a subset of AMAs, this lower bound applies to AMAs as well.

To prove Theorem~\ref{thm:AMAlower}, we construct a set $V$ of $n$-bidder, $m$-item valuation functions taking values in $\{0,1\}$ where, under each valuation function, each bidder is interested in a specific subset of items, and these subsets are all pairwise disjoint. Moreover, $|V| = n^m - n$. The high level idea is to show that for any subset $H$ of $V$, there exists a $\lambda$-auction that has high profit over valuation functions in $H$, but low profit on the valuation functions in $V \setminus H$. Theorem~\ref{thm:AMA_high_low} describes $V$ in more detail. Now suppose that the distribution over the bidders' valuation functions is the uniform distribution over $V$. This means that if a set of samples consist of only a small subset of $V$, then we cannot guarantee that every profit function will achieve average profit over the set of samples which is close to its expected profit over the distribution, as we require.

We now present Theorem~\ref{thm:AMA_high_low}, wherein we describe the set $V$ of valuation functions which we will use to prove Theorem~\ref{thm:AMAlower}.

\begin{theorem}\label{thm:AMA_high_low}
For any $n,m \geq 2$ and any $\beta \in (0,1)$, there exists a set of $N = n^m-n$ $n$-bidder, $m$-item additive valuation functions $V = \left\{\vec{v}^1, \dots, \vec{v}^N\right\}$ such that for any $H \subseteq V$, there exists a $\lambda$-auction $M_H$ with profit 0 on $\vec{v}^i$ if $\vec{v}^i \not\in H$ and profit at least $2-2\beta$ on $\vec{v}^i$ otherwise.
\end{theorem}

\begin{proof}
We define the set $V = \left\{\vec{v}^1, \dots, \vec{v}^N\right\}$ of $n$-bidder, $m$-item additive valuation functions, where \[\vec{v}^j = \left(v_1^j(\vec{e}_1), \dots, v_1^j(\vec{e}_m), \dots, v_n^j(\vec{e}_1) \dots, v_n^j(\vec{e}_m)\right),\] with $N = n^m - n$. Recall that every allocation $Q$ is written as $\left(\vec{q}_1,\dots,  \vec{q}_n\right)$ where $\vec{q}_1,\dots,  \vec{q}_n$ are disjoint subsets of the $m$ items being auctioned. First, let $\hat{Q}^j$ be the allocation where bidder $j$ receives all $m$ items. Next, let $\tilde{Q}^1, \dots, \tilde{Q}^N$ be a fixed ordering of the $n^m-n$ allocations where all $m$ items are allocated except $\hat{Q}^1, \dots, \hat{Q}^n$. Let the bundles allocated to the $n$ bidders in $\tilde{Q}^\ell$ be $\tilde{\vec{q}}^{\ell}_1, \dots, \tilde{\vec{q}}^{\ell}_n$ and let $S_\ell$ be the set of bidders who are allocated some item in allocation $\tilde{Q}^{\ell}$. In other words, $S_{\ell} = \left\{j \ | \ \tilde{\vec{q}}^{\ell}_j \not=\vec{0}\right\}$. For a sanity check, notice that $\sum_{i \in S_{\ell}} \tilde{\vec{q}}^{\ell}_i = \vec{1}$.

We will now define the valuation vectors $\left\{\vec{v}^1, \dots, \vec{v}^N\right\}$ in terms of this set of special allocations $\tilde{Q}^1, \dots, \tilde{Q}^N$, so each vector $\vec{v}^{\ell}$ depends on the allocation $\tilde{Q}^{\ell}$. Specifically, we define $\vec{v}^\ell$ for $\ell \in [N]$ as follows.
If $i \not\in S_\ell$ $\left( \text{i.e., }\tilde{\vec{q}}^\ell_i = \vec{0}\right)$, set $v_i^\ell(\vec{e}_j) = 0$ for all $j \in [m]$. Otherwise, set \[v_i^\ell(\vec{e}_j) = \begin{cases} 0 &\text{if } \tilde{\vec{q}}^\ell_i[j] = 0\\
1 &\text{if } \tilde{\vec{q}}^\ell_i[j] = 1
\end{cases}.\]

We proceed to prove that for any subset $H \subseteq V$, there exists a $\lambda$-auction with 0 profit on all valuation functions in $V \setminus H$ and at least $2-2\beta$ profit on all valuation functions in $H$. To define this $\lambda$-auction, we set the $\lambda$ terms such that \[\lambda\left(Q\right) = \begin{cases}
0 &\text{if } Q = \tilde{Q}^\ell \text{ for some } \vec{v}^\ell \in H\\
1 - \beta &\text{otherwise}\end{cases}.\]

\begin{lemma}\label{lemma:n_bid_high}
If $\vec{v}^\ell \in H$, then the profit on $\vec{v}^\ell$ is at least $2-2\beta$.
\end{lemma}

\begin{proof}[Proof of Lemma~\ref{lemma:n_bid_high}]
First, note that $\sum_{i = 1}^n v_i^\ell\left(\tilde{\vec{q}}_i^{\ell}\right) + \lambda\left(\tilde{Q}^\ell\right) = m$, and for all allocations $Q = \left(\vec{q}_1, \dots, \vec{q}_n\right) \not = \tilde{Q}^\ell$, $\sum_{i = 1}^n v_i^\ell\left(\vec{q}_i\right) + \lambda\left(Q\right) \leq m-1+1-\beta < m$. Therefore, the $\lambda$-auction allocation is $\tilde{Q}^\ell$.

In order to analyze the profit of this $\lambda$-auction, we must understand the payments of each bidder, which means that we must investigate what the outcome of this $\lambda$-auction would be without any one bidder's participation. To this end, suppose $i \in S_\ell$, so bidder $i$ is allocated some item in $\tilde{Q},$ i.e., $\tilde{\vec{q}}_i^{\ell} \not= \vec{0}$. Then $\sum_{j \not= i} v_j^\ell\left(\tilde{\vec{q}}_j^\ell\right) + \lambda\left(\tilde{Q}^\ell\right) = m-\left|\left|\tilde{\vec{q}}_i^{\ell}\right|\right|_1$ because bidder $i$'s valuation for the bundle $\tilde{\vec{q}}_i^{\ell}$ is exactly $\left|\left|\tilde{\vec{q}}_i^{\ell}\right|\right|_1$.

By construction, no bidder receives all $m$ items in $\tilde{Q}^\ell$, so we know that there exists some $i' \in S_\ell, i' \not = i$. With this fact in mind, let $\tilde{Q}^{\ell,-i} = \left(\tilde{\vec{q}}^{\ell, -i}_1, \dots, \tilde{\vec{q}}^{\ell,-i}_n\right)$ be the allocation where all bidders in $S_\ell$ are allocated the same items as they are in $\tilde{Q}^\ell$ and bidder $i$ receives the empty set. This is one possible allocation of the $\lambda$-auction without bidder $i$'s participation, and therefore the social welfare of the other bidders will be at least as high under this allocation as it would be in the true allocation of the $\lambda$-auction without bidder $i$'s participation. By construction, $\lambda\left(\tilde{Q}^{\ell,-i}\right) = 1 - \beta$. Therefore, $\sum_{\ell \not= i} v_j^\ell \left(\tilde{\vec{q}}_j^{\ell,-i}\right) + \lambda\left(\tilde{Q}^{\ell,-i}\right) = m-\left|\left|\tilde{\vec{q}}_i^{\ell}\right|\right|_1 + 1 - \beta$ which means that bidder $i$ must pay at least $\left(m-\left|\left|\tilde{\vec{q}}_i^{\ell}\right|\right|_1 + 1 - \beta\right) - \left(m-\left|\left|\tilde{\vec{q}}_i^{\ell}\right|\right|_1\right) = 1 - \beta.$ We know that $|S_\ell|\geq 2$, i.e., there are at least 2 bidders who receive a non-empty bundle and therefore must pay at least $1 - \beta$, so the profit of this $\lambda$-auction is at least $2-2\beta$.
\end{proof}

\begin{lemma}\label{lemma:n_bid_low} If $\vec{v}^\ell \not\in H$, then the profit on $\vec{v}^\ell$ is 0.
\end{lemma}

\begin{proof}[Proof of Lemma~\ref{lemma:n_bid_low}]
First, note that $\sum_{i = 1}^n v_i^\ell\left(\tilde{\vec{q}}_i^{\ell}\right) + \lambda\left(\tilde{Q}^\ell\right) = m + 1-\beta$, and for all allocations $Q = \left(\vec{q}_1, \dots, \vec{q}_n\right) \not= \tilde{Q}^\ell$, $\sum_{i = 1}^n v_i^\ell\left(\vec{q}_i\right) + \lambda\left(Q\right) \leq m-1 + 1-\beta < m$, so the $\lambda$-auction allocation is $\tilde{Q}^\ell$. Now, suppose $i \in S_\ell$. Then $\sum_{j \not= i} v_j^\ell\left(\tilde{\vec{q}}_j^\ell\right) + \lambda\left(\tilde{Q}^\ell\right) = m - \left|\left|\tilde{\vec{q}}_i^{\ell}\right|\right|_1 + 1 - \beta.$ Since bidder $i$ is the only bidder with nonzero valuations for the items in $\tilde{\vec{q}}_i^{\ell}$ under $\vec{v}^\ell$, any allocation $\tilde{Q}^{\ell,-i}$ without his participation will have social welfare at most $\sum_{j \not= i} v_j^\ell\left(\tilde{\vec{q}}_j^{\ell,-i}\right) + \lambda\left(\tilde{Q}^{\ell,-i}\right) \leq m - \left|\left|\tilde{\vec{q}}_i^{\ell}\right|\right|_1 + 1 - \beta.$ Therefore, bidder $i$ pays nothing.

Of course, for any bidder $i \not\in S_\ell$, her presence in the auction makes no difference on the resulting allocation because her valuation function under $\vec{v}^\ell$ is 0 on all items, so he pays nothing as well. Therefore, the profit on $\vec{v}^\ell$ is 0.
\end{proof}

Putting Lemmas~\ref{lemma:n_bid_high} and~\ref{lemma:n_bid_low} together, we have the desired result.
\end{proof}

We now use Theorem~\ref{thm:AMA_high_low} to prove Theorem~\ref{thm:AMAlower}.

\begin{theorem}\label{thm:AMAlower}
For any $\epsilon \in (0,1)$, there exists a distribution $\pazocal{D}$ and a $\lambda$-auction $M^*$ such that, with probability 1 over the draw of a set of samples $\sample$ of size at most $\frac{n^m-n}{2}$, \[\frac{1}{|\sample|}\sum_{\vec{v} \in \sample} {\normalfont \profit}_{M^*} \left(\vec{v}\right) - \E_{\vec{v}\sim \pazocal{D}}\left[ {\normalfont \profit}_{M^*}\left(\vec{v}\right)\right] > \epsilon.\]
\end{theorem}

\begin{proof}
Let $\beta = 1 - \epsilon$ and let $V$ be the set of valuation functions proven to exist in Theorem~\ref{thm:AMA_high_low} corresponding to $\beta$ (i.e. for any $H \subseteq V$, there exists a $\lambda$-auction $M_H$ with profit 0 on $\vec{v}$ if $\vec{v} \in H$ and profit at least $2-2\beta$ on $\vec{v}$ otherwise). Let $\pazocal{D}$ be the uniform distribution on $V$.

Suppose that $\sample$ is a set of at most $\frac{n^m-n}{2}$ samples. Of course, $\sample \subseteq V$, so let $M^*$ be the $\lambda$-auction with 0 profit on every valuation function not in the set of samples and profit at least $2 - 2\beta$ on every valuation function in the set of samples. We know that $M^*$ exists due to Theorem~\ref{thm:AMA_high_low}.

Notice that the average empirical profit of $M^*$ on $\sample$ is at least $2 - 2\beta$. Meanwhile, the probability, on a random draw $\vec{v} \sim \pazocal{D}$ that ${\normalfont \profit}_{M^*}\left(\vec{v}\right)$ is 0 is exactly the probability that $\vec{v} \not\in \sample$. Given that the set of training examples has measure $\frac{|\sample|}{n^m - n}\leq\frac{1}{2},$ we have that \begin{align*}
&\frac{1}{|\sample|}\sum_{\vec{v} \in \sample} {\normalfont \profit}_{M^*} \left(\vec{v}\right) - \E_{\vec{v}\sim \pazocal{D}}\left[ {\normalfont \profit}_{M^*}\left(\vec{v}\right)\right] \geq 2-2\beta - (2-2\beta)\Pr_{\vec{v}\sim\pazocal{D}}\left[\vec{v} \in \sample\right]\\
>\text{ }&2-2\beta - (1-\beta) = 1-\beta = \epsilon.
\end{align*}
as desired.
\end{proof}

\bigskip\emph{Lower Bound on Sample Complexity for VVCAs.}We now prove that it is not possible to learn over the set of VVCA profit function under and arbitrary distribution with subexponential sample complexity. In particular, we prove that no algorithm can learn over the class of $n$-bidder, $m$-item VVCA profit functions with sample complexity $2^m - 2$. This holds even when the bidders' valuation functions are additive.

The format of this proof similar to that of Theorem~\ref{thm:AMAlower}. Namely, we construct a set $V$ of $n$-bidder, $m$-item valuation functions such that $|V| = 2^m - 2$. We then show that for any subset $H$ of $V$, there exists a VVCA that has high profit over valuation functions in $H$, but low profit on the valuation functions in $V \setminus H$. The set $V$ is described in more detail in Theorem~\ref{thm:VVCA_high_low}. As described in Theorem~\ref{thm:AMAlower}, this immediately implies hardness for learning over the uniform distribution on $V$. Given the parallel proof structure, we present Theorem~\ref{thm:VVCA_high_low} and refer the reader to Theorem~\ref{thm:AMAlower} to see how it implies hardness for learning.

\begin{theorem}\label{thm:VVCA_high_low}
For any $m \geq 2$ and any $\beta \in (0,1)$, there exists a set of $N = 2^m-2$ 2-bidder additive valuation functions $V = \{\vec{v}^1, \dots, \vec{v}^N\}$ such that for any $H \subseteq V$, there exists a VVCA with profit 0 on $\vec{v}^i$ if $\vec{v}^i \in V$ and profit $1-\beta$ on $\vec{v}^i$ if $\vec{v}^i \not\in V$.
\end{theorem}

\begin{proof}
We define the set $V = \{\vec{v}^1, \dots, \vec{v}^N\}$ of 2-bidder valuation functions, where \[\vec{v}^j = (v_1^j(\vec{e}_1), \dots, v_1^j(\vec{e}_m), v_2^j(\vec{e}_1) \dots, v_2^j(\vec{e}_m)),\] with $N = 2^m - 2$. Recall that every allocation vector $Q$ can be written as $\left(\vec{q}_1, \vec{q}_2\right)$ where $\vec{q}_1$ and $\vec{q}_2$ are disjoint subsets of the $m$ items being auctioned. In order to define the valuation functions in $V$, we define $\tilde{\vec{q}}^1, \dots, \tilde{\vec{q}}^N$ to be a arbitrary, fixed ordering of the vectors in the set $\{0,1\}^m \setminus \{\vec{0}, \vec{1}\}$. We will define each valuation function in $V$ in terms of this ordering. In particular, let $\tilde{Q}^\ell = \left(\vec{1} - \tilde{\vec{q}}^\ell, \tilde{\vec{q}}^\ell\right)$ be the allocation where bidder 1 receives $\vec{1} - \tilde{\vec{q}}^\ell$ and bidder 2 receives $\tilde{\vec{q}}^\ell$. Finally, let $\vec{v}^\ell$ for $\ell \in [N]$ be defined as follows:

\[v_1^\ell(\vec{e}_i) = \begin{cases} 1 &\text{if } \tilde{\vec{q}}^\ell[i] = 0\\
0 &\text{otherwise}
\end{cases} \text{ and } v_2^\ell(\vec{e}_i) = \begin{cases} 1 &\text{if } \tilde{\vec{q}}^\ell[i] = 1\\
0 &\text{otherwise}
\end{cases}.\]

Clearly, if $w_1=w_2=1$ and $\lambda_1(Q) = \lambda_2(Q) = 0$ for all $Q \in \Qset$, then the VVCA allocation on any $\vec{v}^\ell \in S$ is the one in which bidder 2 receives $\vec{1} - \tilde{\vec{q}}^\ell$ and bidder 1 receives $\tilde{\vec{q}}^\ell$. This has a social welfare of $m$, whereas any other allocation has a social welfare at most $m-1$.

We claim that for any $H \subseteq V$, there exists a VVCA with profit 0 on $\vec{v}^i$ if $\vec{v}^i \in H$ and profit $1-\beta$ on $\vec{v}^i$ if $\vec{v}^i \not\in H$. The VVCA has bidder weights $w_1 = w_2 = 1$, and for all $\vec{v}^\ell \in H$, we set $\lambda_1(\tilde{Q}^\ell) = c_{1,\vec{1} - \tilde{\vec{q}}^\ell} = c_{2, \tilde{\vec{q}}^\ell} = \lambda_2(\tilde{Q}^\ell) = 0$. Otherwise, we set $\lambda_i(Q) = (1 - \beta)/2$ for each $i \in \{1,2\}$.

\begin{lemma}\label{lem:VVCA_high} If $\vec{v}^\ell \in H$, then the profit on $\vec{v}^\ell$ is $1-\beta$.
\end{lemma}

\begin{proof}[Proof of Lemma~\ref{lem:VVCA_high}]
First, note that $v_1(\vec{1} - \tilde{\vec{q}}^\ell) + v_2(\tilde{\vec{q}}^\ell) +\lambda_1(\tilde{Q}^\ell) + \lambda_2(\tilde{Q}^\ell) = m$, and for all allocations $Q = \left(\vec{q}_1, \vec{q}_2\right) \not = \tilde{Q}^\ell$, $v_1(\vec{q}_1) + v_2(\vec{q}_2) + \lambda_1(Q) + \lambda_2(Q) \leq m-1 + 1-\beta$. Therefore, the VVCA allocation is $\tilde{Q}^\ell$. However, this is neither bidder 1 nor bidder 2's favorite weighted allocation, since $v_1(\vec{1} - \tilde{\vec{q}}^\ell) + \lambda_1(\tilde{Q}^\ell) = |\vec{1} - \tilde{\vec{q}}^\ell| < v_1(\vec{1}) + c_{1,\vec{1}} = |\vec{1} - \tilde{\vec{q}}^\ell| + (1-\beta) /2$ and $v_2(\tilde{\vec{q}}^\ell) + \lambda_2(\tilde{Q}^\ell) = |\tilde{\vec{q}}^\ell|< v_2(\vec{1}) + c_{2,\vec{1}} = |\tilde{\vec{q}}^\ell| + (1-\beta) /2$. This follows from the fact that $\tilde{\vec{q}}^\ell \not = \vec{1}$ and $\vec{1} - \tilde{\vec{q}}^\ell \not = \vec{1}$ for all $\ell \in [N]$, so it must be that $c_{1,\vec{1}} = c_{2,\vec{1}} = (1 - \beta)/2.$

Since $|\vec{1} - \tilde{\vec{q}}^\ell|$ and $|\tilde{\vec{q}}^\ell|$ are bidder 1 and 2's highest valuations for any allocation, respectively, and because $(1-\beta)/2$ is the highest value of any $\lambda$ term, $v_1(\vec{1}) + c_{1,\vec{1}}$ and $v_2(\vec{1}) + c_{2,\vec{1}}$ are the maximum weighted valuation that either bidder has for any allocation under this VVCA. Therefore, the profit of this VVCA on $\vec{v}_\ell$ is $|\tilde{\vec{q}}^\ell| + |\vec{1} - \tilde{\vec{q}}^\ell| + 1-\beta - |\tilde{\vec{q}}^\ell| - |\vec{1} - \tilde{\vec{q}}^\ell| = 1-\beta$.
\end{proof}

\begin{lemma}\label{lem:VVCA_low} If $\vec{v}^\ell \not \in H$, then the profit on that valuation function pair is 0.
\end{lemma}

\begin{proof}[Proof of Lemma~\ref{lem:VVCA_low}]
First, note that $v_1(\vec{1} - \tilde{\vec{q}}^\ell) + v_2(\tilde{\vec{q}}^\ell) +\lambda_1(\tilde{Q}^\ell) + \lambda_2(\tilde{Q}^\ell) = m+1-\beta$, and for all allocations $Q = \left(\vec{q}_1, \vec{q}_2\right) \not= \tilde{Q}^\ell$, $v_1(\vec{q}_1) + v_2(\vec{q}_2) + \lambda_1(Q) + \lambda_2(Q) \leq m-1 + 1-\beta < m + 1-\beta$, so the AMA allocation is $\tilde{Q}^\ell$. Moreover, for all allocations $Q = (\vec{q}_1, \vec{q}_2)$, $v_1(\vec{1} - \tilde{\vec{q}}^\ell) + \lambda_1(\tilde{Q}^\ell) = |\vec{1} - \tilde{\vec{q}}^\ell| + (1-\beta)/2 \geq v_1(\vec{q}_1) + \lambda_1(Q)$ and  $v_2(\tilde{\vec{q}}^\ell) + \lambda_2(\tilde{Q}^\ell) = |\tilde{\vec{q}}^\ell| + (1-\beta)/2 \geq v_2(\vec{q}_2) + \lambda_2(Q)$. Therefore, both bidders receive one of their favorite weighted allocations, so the profit is 0.
\end{proof}
\end{proof}

\section{Connection to structured prediction}\label{APP:STRUCTURED}
In this section, we connect the hyperplane structure we investigate in this paper to the structured prediction literature in machine learning (e.g., \citep{Collins00:Discriminative}), thus proving even stronger generalization bounds for item-pricing mechanisms under buyers with unit-demand and general valuations and answering an open question by \citet{Morgenstern16:Learning}.
\citet{Balcan14:Learning} were the first to explore the connection between structured prediction and mechanism design, though in a different setting from us: they provided algorithms that make use of past data describing the purchases of a utility-maximizing agent to produce a hypothesis function that can accurately forecast the future behavior of the agent.

 \citet{Morgenstern16:Learning} used structured prediction to provide sample complexity guarantees for several ``simple'' mechanism classes. They observed that these classes have profit functions which are the composition of two simpler functions: A generalized allocation function $f^{\left(1\right)}_{\vec{p}} : \domain \to \range$ and a simplified profit function $f^{\left(2\right)}_{\vec{p}} : \domain \times \range \to \R$ such that $\profit_{\vec{p}}\left(\vec{v}\right) = f^{\left(2\right)}_{\vec{p}}\left(\vec{v}, f^{\left(1\right)}_{\vec{p}}\left(\vec{v}\right)\right)$. For example, $\range$ might be the set of allocations. In this case, we say that $\mclass$ is $\left(\fclass^{\left(1\right)}, \fclass^{\left(2\right)}\right)$-decomposable, where $\fclass^{\left(1\right)} = \left\{f^{\left(1\right)}_{\vec{p}} : \vec{p} \in \pspace\right\}$ and $\fclass^{\left(2\right)} = \left\{f^{\left(2\right)}_{\vec{p}} : \vec{p} \in \pspace\right\}$.
See Example~\ref{ex:2PT2} for an example of this decomposition.
\begin{example}[Item-pricing mechanisms \citep{Morgenstern16:Learning}]\label{ex:2PT2}
Let $\mclass$ be the class of anonymous item-pricing mechanisms over a single additive buyer and let $\vec{p} = (p_1,\dots, p_m)$ be a vector of prices. In this case, we can define $f_{\vec{p}}^{(1)}: \domain \to \{0,1\}^m$ where the $i^{th}$ component of $f_{\vec{p}}^{(1)}(\vec{v})$ is 1 if and only if the buyer buys item $i$. Define $\psi(\vec{v}, \vec{\alpha}) = (v(\vec{\alpha}), -\vec{\alpha})$ and define $\vec{w}^{\vec{p}} = (1, \vec{p})$. Then the $\vec{\alpha}$ that maximizes $ \left\langle\vec{w}^{\vec{p}}, \psi(\vec{v}, \vec{\alpha})\right\rangle$ is the $\vec{\alpha}$ that maximizes the buyer's utility, i.e., $f^{(1)}_{\vec{p}}(\vec{v})$, as desired.
Finally, we define $f_{\vec{p}}^{(2)}(\vec{v}, \vec{\alpha}) =\langle \vec{\alpha}, \vec{p}\rangle,$  and we have that $\profit_{\vec{p}}(\vec{v}) = f_{\vec{p}}^{(2)}\left(\vec{v}, f_{\vec{p}}^{(1)}(\vec{v})\right)$, as desired.
\end{example}

\citet{Morgenstern16:Learning} bound $\pdim\left(\mclass\right)$ using the ``complexity'' of $\fclass^{\left(1\right)}$, which they quantified using tools from structured prediction, namely, \emph{generalized linear functions.}

\begin{definition}[$a$-dimensional linear class]
A set $\fclass = \left\{f_{\vec{p}}: \domain  \to \range \ | \ \vec{p} \in \pspace\right\}$ is an \emph{$a$-dimensional linear class} if
 there is a function $\psi: \domain \times \range \to \R^a$ and a vector $\vec{w}^{\vec{p}} \in \R^a$ for each $\vec{p} \in \pspace$ such that $f_{\vec{p}}\left(\vec{v}\right) \in \argmax_{\vec{\alpha} \in \range} \langle \vec{w}^{\vec{p}}, \psi\left(\vec{v},\vec{\alpha}\right)\rangle$ and $|\argmax_{\vec{\alpha} \in \range} \langle \vec{w}^{\vec{p}}, \psi\left(\vec{v},\vec{\alpha}\right)\rangle| = 1$.
\end{definition}

If $\mclass$ is $\left(\fclass^{(1)}, \fclass^{(2)}\right)$-decomposable and $\fclass^{(1)}$ is an $a$-dimensional linear class over $\range$, we say that $\mclass$ is an $a$-dimensional linear class over $\range$.

The bounds \citet{Morgenstern16:Learning} provided using linear separability are loose in several settings: for anonymous and non-anonymous item-pricing mechanisms under additive buyers, their structured prediction approach gives a pseudo-dimension bound of $O\left(m^2\right)$ and $O\left(nm^2\log m\right)$, respectively. They left as an open question whether linear separability can be used to prove tighter guarantees. Using the hyperplane structures we study in this paper, we prove that the answer is ``yes.'' We require the following refined notion of $\left(d,t\right)$-delineable classes.

\begin{definition}[$\left(d,t_1,t_2\right)$-divisible]
Suppose $\mclass$ consists of mechanisms parameterized by vectors $\vec{p} \subseteq \R^d$ and that $\mclass$ is $\left(\fclass^{\left(1\right)}, \fclass^{\left(2\right)}\right)$-decomposable. We say that $\mclass$ is \emph{$\left(d,t_1,t_2\right)$-divisible} if:
\begin{enumerate}
\item For any $\vec{v} \in \domain$, there is a set $\hyp$ of $t_1$ hyperplanes such that for any connected component $\pspace'$ of $\R^d \setminus \hyp$, the function $f_{\vec{v}}^{\left(1\right)}\left(\vec{p}\right)$ is constant over all $\vec{p} \in \pspace'.$
\item For any $\vec{v} \in \domain$ and any $\vec{\alpha} \in \range$, there is a set $\hyp_2$ of $t_2$ hyperplanes such that for any connected component $\pspace'$ of $\R^d \setminus \hyp_2$, the function $f_{\vec{v}, \vec{\alpha}}^{\left(2\right)}\left(\vec{p}\right)$ is linear over all $\vec{p} \in \pspace'.$
\end{enumerate}
\end{definition}

Note that $\left(d,t_1,t_2\right)$-divisibility implies $\left(d,t_1+t_2\right)$-delineability. Theorem~\ref{thm:version2} connects linear separability and divisibility with pseudo-dimension.

\begin{restatable}{theorem}{separable}\label{thm:version2}
Suppose $\mclass$ is mechanism class that is $\left(d,t_1, t_2\right)$-divisible with $t_1, t_2 \geq 1$ and an $a$-dimensional linear class over $\range$. Let $\omega = \min\left\{|\range|^a, d\left(at_1\right)^d\right\}$. Then \[\pdim\left(\mclass\right) = O\left(\left(d+a\right)\log\left(d+a\right) + d\log t_2 + \log \omega\right).\]
\end{restatable}

\begin{proof}
To prove this theorem, we will use the following standard notation. For a class $\fclass$ of real-valued functions mapping $\domain$ to $\R$, let $\sample = \left\{\vec{v}^{(1)}, \dots, \vec{v}^{(N)}\right\}$ be a subset of $\domain$. We define
\[\Pi_{\fclass}(\sample) = \max_{z^{(1)}, \dots, z^{(N)} \in \R} \left|\left\{\begin{pmatrix}
\textbf{1}_{\left\{f\left(\vec{v}^{(1)}\right) \geq z^{(1)}\right\}}\\
\vdots\\
\textbf{1}_{\left\{f\left(\vec{v}^{(N)}\right) \geq z^{(N)}\right\}}
\end{pmatrix} : f \in \fclass\right\}\right|.\]
The pseudo-dimension of $\fclass$ is the size of the largest set $\sample$ such that $\Pi_{\fclass}(\sample) = 2^{|\sample|}$. We also use the notation $f(\sample)$ to denote the vector $\left(f(\vec{v}^{(1)}), \dots, f(\vec{v}^{(N)})\right)$. \citet{Morgenstern16:Learning} proved the following lemma.

\begin{lemma}[\citet{Morgenstern16:Learning}]\label{lem:MRshatter}
Suppose $\mclass$ is $\left(\fclass^{(1)}, \fclass^{(2)}\right)$-decomposable and an $a$-dimensional linear class. Let $\sample = \left\{\vec{v}^{(1)}, \dots, \vec{v}^{(N)}\right\}$ be a subset of $\domain$. Then \begin{align*}\Pi_{\mclass}(\sample) \leq &\left|\left\{\left(\sample', f^{(1)}_{\vec{p}}(\sample')\right) : \sample' \subseteq \sample, |\sample'| = a, \vec{p} \in \pspace \right\}\right|\\
\cdot &\max_{\vec{\alpha}^{(1)}, \dots, \vec{\alpha}^{(N)} \in \range} \left\{ \Pi_{\fclass^{(2)}} \left(\left\{ \left(\vec{v}^{(1)}, \vec{\alpha}^{(1)}\right), \dots, \left(\vec{v}^{(N)}, \vec{\alpha}^{(N)}\right)\right\}\right)\right\}.\end{align*}
\end{lemma}

Suppose the pseudo-dimension of $\mclass$ is $N$. By definition, there exists a set $\sample = \left\{\vec{v}^{(1)}, \dots, \vec{v}^{(N)}\right\}$ that is shattered by $\mclass$. By Lemmas~\ref{lem:MRshatter} and \ref{lem:f1labels}, this means that \[2^N = \Pi_{\mclass}(\sample) \leq N^a\omega \max_{\vec{\alpha}^{(1)}, \dots, \vec{\alpha}^{(N)} \in \range} \left\{ \Pi_{\fclass^{(2)}} \left(\left\{ \left(\vec{v}^{(1)}, \vec{\alpha}^{(1)}\right), \dots, \left(\vec{v}^{(N)}, \vec{\alpha}^{(N)}\right)\right\}\right)\right\}.\] To prove this theorem, we will show that \begin{equation}\max_{\vec{\alpha}^{(1)}, \dots, \vec{\alpha}^{(N)} \in \range} \left\{ \Pi_{\fclass^{(2)}} \left(\left\{ \left(\vec{v}^{(1)}, \vec{\alpha}^{(1)}\right), \dots, \left(\vec{v}^{(N)}, \vec{\alpha}^{(N)}\right)\right\}\right)\right\} < d^2 \left(N^2t_2\right)^d,\label{eq:shattering}\end{equation} which means that $2^N < N^{2d+a}d^2t_2^d\omega$, and thus $N= O\left((d+a) \log(d+a) + d \log t_2 + \log \omega\right)$.

To this end, let $\vec{\alpha}^{(1)}, \dots, \vec{\alpha}^{(N)}$ be $N$ arbitrary elements of $\range$ and let $z^{(1)}, \dots, z^{(N)}$ be $N$ arbitrary elements of $\R$. Since $\mclass$ is $(d, t_1, t_2)$-divisible, we know that for each $i \in [N]$, there is a set $\hyp_2^{(i)}$ of $t_2$ hyperplanes such that for any connected component $\pspace'$ of $\pspace \setminus \hyp_2^{(i)}$, $f^{(2)}_{\vec{v}^{(i)}, \vec{\alpha}^{(i)}}\left(\vec{p}\right)$ is linear over all $\vec{p} \in \pspace'.$  We now consider the overlay of all $N$ partitions $\pspace\setminus \hyp_2^{(1)}, \dots, \pspace\setminus \hyp_2^{(N)}$. Formally, this overlay is made up of the sets $\pspace_1, \dots, \pspace_{\tau}$, which are the connected components 
of $\pspace \setminus \left(\bigcup_{i = 1}^N \hyp_2^{(i)}\right)$. For each set $\pspace_j$ and each $i \in [N]$, $\pspace_j$ is completely contained in a single connected component of $\pspace \setminus \hyp_2^{(i)}$, which means that $f^{(2)}_{\vec{v}^{(i)}, \vec{\alpha}^{(i)}}\left(\vec{p}\right)$ is linear over $\pspace_j.$
Since $\left|\hyp^{(i)}_2\right|\leq t_2$ for all $i \in [N]$, $\tau < d(Nt_2)^d$ \citep{Buck43:Partition}.

Now, consider a single connected component $\pspace_j$ of $\pspace \setminus \left(\bigcup_{i = 1}^N \hyp^{(i)}_2\right)$. For any sample $\vec{v}^{(i)} \in \sample$, 
we know that $f^{(2)}_{\vec{v}^{(i)}, \vec{\alpha}^{(i)}}\left(\vec{p}\right)$ is linear over $\pspace_j$. 
Let $\vec{a}_j^{(i)} \in \R^d$ and $b_j^{(i)} \in \R$ be the weight vector and offset 
such that $f^{(2)}_{\vec{v}^{(i)}, \vec{\alpha}^{(i)}}\left(\vec{p}\right) = \vec{a}_j^{(i)} \cdot \vec{p} + b_j^{(i)}$ for all $\vec{p} \in \pspace_j$. We know that there is a hyperplane $\vec{a}_j^{(i)} \cdot \vec{p} + b_j^{(i)} = z^{(i)}$ where on one side of the 
hyperplane, $f^{(2)}_{\vec{v}^{(i)}, \vec{\alpha}^{(i)}}\left(\vec{p}\right) \leq z^{(i)}$ and on the other 
side, $f^{(2)}_{\vec{v}^{(i)}, \vec{\alpha}^{(i)}}\left(\vec{p}\right) > z^{(i)}$.
Let $\hyp_{\pspace_j}$ be all $N$ hyperplanes for all $N$ samples, i.e., $\hyp_{\pspace_j} = \left\{\vec{a}_j^{(i)} \cdot \vec{p} + b_j^{(i)} = z^{(i)} : i \in [N]\right\}.$ Notice that in any 
connected component $\pspace'$ of $\pspace_j \setminus \hyp_{\pspace_j}$, for all $i \in [N]$, $f^{(2)}_{\vec{v}^{(i)}, \vec{\alpha}^{(i)}}\left(\vec{p}\right)$ is either greater than $z^{(i)}$ or 
less than $z^{(i)}$ (but not both) for all $\vec{p} \in \pspace'$. 

In total, the number of 
connected components of $\pspace_j \setminus \hyp_{\pspace_j}$ is smaller than $dN^d$. 
The same holds for every partition $\pspace_j$. Thus, the total number of regions where 
for all $i \in [N]$, $f^{(2)}_{\vec{v}^{(i)}, \vec{\alpha}^{(i)}}\left(\vec{p}\right)$ is either greater 
than $z^{(i)}$ or less than $z^{(i)}$ (but not both) is smaller than $dN^d\cdot d(Nt_2)^d$. In other words, \[\left|\left\{\begin{pmatrix}
\textbf{1}\left(f^{(2)}_{\vec{p}}\left(\vec{v}^{(1)}, \vec{\alpha}^{(1)}\right) \geq z^{(1)}\right)\\
\vdots\\
\textbf{1}\left(f^{(2)}_{\vec{p}}\left(\vec{v}^{(N)}, \vec{\alpha}^{(N)}\right) \geq z^{(N)}\right)
\end{pmatrix} : \vec{p} \in \pspace\right\}\right| \leq dN^d\cdot d(Nt_2)^d.\] Since we chose $\vec{\alpha}^{(1)}, \dots, \vec{\alpha}^{(N)}$ and $z^{(1)}, \dots, z^{(N)}$ arbitrarily, we may conclude that Inequality~\eqref{eq:shattering} holds.
\end{proof}

\begin{lemma}\label{lem:f1labels}
Suppose $\mclass$ is an $a$-dimensional linear class over $\range$ and $(d,t_1, t_2)$-divisible. Then for any set $\sample \subseteq \domain$ of size $N$, \[\left|\left\{ (\sample', f^{(1)}_{\vec{p}}(\sample')) : \sample' \subseteq \sample, |\sample'| = a, \vec{p} \in \pspace\right\}\right| \leq N^a\min\left\{|\range|^a, d(at_1)^d\right\}.\]
\end{lemma}

\begin{proof}
To begin with, there are of course at most $N^a$ ways to choose a set $\sample' \subseteq \sample$ of size $a$. How many ways are there to label a fixed set $\sample' = \left\{\vec{v}^{(i_1)}, \dots, \vec{v}^{(i_a)}\right\}$ of size $a$ using functions from $\fclass^{(1)}$? An easy upper bound is $|\range|^a$. Alternatively, we can use the structure of $\mclass$ to prove that there are $d(at_1)^d$ ways to label $\sample'$. Since $\mclass$ is $(d,t_1, t_2)$-divisible, we know that for any $\vec{v}^{(i_j)} \in \sample'$, there is a set $\hyp_1^{(i_j)}$ of $t_1$ hyperplanes such that for any connected component $\pspace'$ of $\pspace \setminus \hyp_1^{(i_j)}$, $f^{(1)}_{\vec{v}^{(i_j)}}(\vec{p})$ is constant over all $\vec{p} \in \pspace'.$
We now consider the overlay of all $a$ partitions $\pspace\setminus \hyp_1^{(i_j)}$ for all $\vec{v}^{(i_j)} \in \sample'$. Formally, this overlay is made up of the sets $\pspace_1, \dots, \pspace_{\tau}$, which are the connected components 
of $\pspace \setminus \left(\bigcup_{\vec{v}^{(i_j)} \in \sample'} \hyp_1^{(i_j)}\right)$. For each set $\pspace_t$ and each $\vec{v}^{(i_j)} \in \sample'$, $\pspace_t$ is completely contained in a single connected component of $\pspace \setminus \hyp_1^{(i_j)}$, which means that $f^{(1)}_{\vec{v}^{(i_j)}}\left(\vec{p}\right)$ is constant over $\pspace_t.$ 
This means that the number of ways to label $\sample'$ is at most $\tau$. Since $\left|\hyp_1^{(i_j)}\right|\leq t_1$ for all $\vec{v}^{(i_j)} \in \sample'$, $\tau < d(at_1)^d$ \citep{Buck43:Partition}.
Therefore, $\left|\left\{ \left(\sample', f^{(1)}_{\vec{p}}(\sample')\right) : \sample' \subseteq \sample, |\sample'| = a, \vec{p} \in \pspace\right\}\right| \leq N^a\min\left\{|\range|^a, d(at_t)^d\right\}$, so the lemma statement holds.
\end{proof}
\subsection{Divisible mechanism classes}\label{sec:divisible}
We now instantiate Theorem~\ref{thm:version2}.

\begin{restatable}{theorem}{itemUnit}\label{thm:item_pricing_unit}
Let $\pazocal{M}$ and $\pazocal{M}'$ be the classes of item-pricing mechanisms with anonymous prices and non-anonymous prices. If the buyers are unit-demand, then $\mclass$ is  $\left(m, nm^2, 1\right)$-divisible and $\mclass'$ is $\left(nm,nm^2,1\right)$-divisible. Also, $\mclass$ and $\mclass'$ are $\left(m+1\right)$- and $\left(nm+1\right)$-dimensionally linearly separable over $\left\{0,1\right\}^m$ and $[n]^m$. Therefore, \[\pdim\left(\mclass\right) = O\left(\min \left\{m^2, m\log\left(nm\right)\right\}\right) \text{ and } \pdim\left(\mclass'\right) = O\left(nm \log (nm)\right).\]
\end{restatable}

\begin{proof}
We begin with anonymous reserves. Let $f_{\vec{p}}^{(1)}: \domain \to \{0,1\}^m$ be defined so that the $i^{th}$ component is 1 if and only if item $i$ is sold. For each buyer $j$, there are ${m \choose 2}$ hyperplanes defining their preference ordering on the items: $v_j(\vec{e}_i) - p(\vec{e}_i) = v_j(\vec{e}_k) - p(\vec{e}_k)$ for all $i \not = k$. This gives a total of at most $t_1 = nm^2$ hyperplanes splitting $\R^m$ into regions where $f_{\vec{v}}^{(1)}(\vec{p})$ is constant. Next, we can write $f_{\vec{p}}^{(2)}(\vec{v}, \vec{\alpha}) = \vec{\alpha} \cdot \vec{p}$, which is always linear, so we may set $t_2 = 1$.

Under non-anonymous reserve prices, let $f_{\vec{p}}^{(1)}: \domain \to \{0,1\}^{nm}$ be defined so that for every buyer $j$ and every item $i$, there is a component of $f_{\vec{p}}^{(1)}(\vec{v})$ that is 1 if and only if buyer $j$ receives item $i$. As with anonymous prices, there are $t_1 = nm^2$ hyperplanes splitting $\R^{nm}$ into regions where $f_{\vec{v}}^{(1)}(\vec{p})$ is constant. Next, we can write $f_{\vec{p}}^{(2)}(\vec{v}, \vec{\alpha}) = \vec{\alpha} \cdot \vec{p}$, which is always linear, so we may set $t_2 = 1$.

\citet{Morgenstern16:Learning} proved that $\mclass$ and $\mclass'$ are $(m+1)$- and $(nm+1)$-dimensionally linearly separable over $\{0,1\}^m$ and $[n]^m$, respectively.
\end{proof}

When prices are anonymous, if $ n < 2^m$, Theorem~\ref{thm:item_pricing_unit} improves on the pseudo-dimension bound of $O\left(m^2\right)$ \citet{Morgenstern16:Learning} gave for this class, and otherwise it matches their bound. When the prices are non-anonymous our bound improves on their bound of $O\left(nm^2 \log n\right)$.

\begin{restatable}{theorem}{itemGeneral}\label{thm:item_pricing_general}
Let $\pazocal{M}$ and $\pazocal{M}'$ be the classes of item-pricing mechanisms with anonymous prices and non-anonymous prices, respectively. If the buyers have general values, then $\mclass$ is $\left(m, n2^{2m}, 1\right)$-divisible and $\mclass'$ is $\left(nm,n2^{2m},1\right)$-divisible. Also, $\mclass$ is $\left(m+1\right)$-dimensionally linearly separable over $\left\{0,1\right\}^m$ and $\mclass'$ is $\left(nm+1\right)$-dimensionally linearly separable over $[n]^m$. Thus, $Pdim\left(\mclass\right) = O\left(m^2\right)$ and $Pdim\left(\mclass'\right) = O\left(nm\left(m + \log n\right)\right)$.
\end{restatable}

\begin{proof}
We begin with anonymous reserves. Let $f_{\vec{p}}^{(1)}: \domain \to \{0,1\}^m$ be defined so that the $i^{th}$ component is 1 if and only if item $i$ is sold. For each buyer $j$, there are ${2^m \choose 2}$ hyperplanes defining their preference ordering on the bundles: $v_j(\vec{q}) - \sum_{i: q[i] = 1} p(\vec{e}_i) = v_j(\vec{q}') - \sum_{i: q'[i] = 1} p(\vec{e}_i)$ for all $\vec{q}, \vec{q}' \in \{0,1\}^m$. This gives a total of at most $t_1 = n2^{2m}$ hyperplanes splitting $\R^m$ into regions where $f_{\vec{v}}^{(1)}(\vec{p})$ is constant. Next, we can write $f_{\vec{p}}^{(2)}(\vec{v}, \vec{\alpha}) = \vec{\alpha} \cdot \vec{p}$, which is always linear, so we may set $t_2 = 1$.

Under non-anonymous reserve prices, let $f_{\vec{p}}^{(1)}: \domain \to \{0,1\}^{nm}$ be defined so that for every buyer $j$ and every item $i$, there is a component of $f_{\vec{p}}^{(1)}(\vec{v})$ that is 1 if and only if buyer $j$ receives item $i$. As with anonymous prices, there are $t_1 = n2^{2m}$ hyperplanes splitting $\R^{nm}$ into regions where $f_{\vec{v}}^{(1)}(\vec{p})$ is constant. Next, we can write $f_{\vec{p}}^{(2)}(\vec{v}, \vec{\alpha}) = \vec{\alpha} \cdot \vec{p}$, which is always linear, so we may set $t_2 = 1$.

\citet{Morgenstern16:Learning} proved that $\mclass$ and $\mclass'$ are $(m+1)$- and $(nm+1)$-dimensionally linearly separable over $\{0,1\}^m$ and $[n]^m$, respectively.
\end{proof}

When there are anonymous prices, the number of hyperplanes in the partition is large, so considering the hyperplane partition does not help us. As a result, Theorem~\ref{thm:item_pricing_general} implies the same bound \citet{Morgenstern16:Learning} gave. In the case of non-anonymous prices, analyzing the hyperplane partition gives a better bound than their bound of $O\left(nm^2 \log n\right)$.

In Theorem~\ref{thm:second_price_sep}, we use Theorem~\ref{thm:version2} to prove pseudo-dimension bounds of $O\left(m \log m\right)$ and $O\left(nm \log nm\right)$ for the classes of second price auctions for additive buyers with anonymous and non-anonymous reserves, respectively. In Theorem~\ref{thm:item_pricing_add_sep}, we prove the same for item-pricing mechanisms. We thus answer the open question by \citet{Morgenstern16:Learning}. These bounds match those implied by Lemmas~\ref{lem:item_pricing_add} and \ref{lem:second_price}.

\begin{theorem}\label{thm:second_price_sep}
Let $\pazocal{M}$ and $\pazocal{M}'$ be the classes of anonymous and non-anonymous second price item auctions. Then $\mclass$ is $\left(m, m, m\right)$-divisible and $\mclass'$ is $(nm, m, m)$-divisible. Also, $\mclass$ and $\mclass'$ are $(m+1)$- and $(nm+1)$-dimensionally linearly separable over $\{0,1\}^m$ and $[n]^m$. Therefore, $Pdim(\mclass) = O(m \log m)$ and $Pdim(\mclass') = O(nm \log (nm))$.
\end{theorem}

\begin{proof}
We begin with anonymous reserves. For a given valuation vector $\vec{v}$, let $j_i$ be the highest buyer for item $i$ and let $j_i'$ be the second highest buyer. Let $f_{\vec{p}}^{(1)}: \domain \to \{0,1\}^m$ be defined so that the $i^{th}$ component is 1 if and only if item $i$ is sold. There are $t_1 = m$ hyperplanes splitting $\R^m$ into regions where $f_{\vec{v}}^{(1)}(\vec{p})$ is constant: the $i^{th}$ component of $f_{\vec{v}}^{(1)}(\vec{p})$ is 1 if and only if $v_{j_i}(\vec{e}_i) \geq p(\vec{e}_i)$. Next, we can write $f_{\vec{p}}^{(2)}(\vec{v}, \vec{\alpha}) = \sum_{i: \alpha[i] = 1 } \max\left\{v_{j_i'}(\vec{e}_i), p(\vec{e}_i)\right\} - c(\vec{\alpha})$, which is linear so long as either $v_{j_i'}(\vec{e}_i) <p(\vec{e}_i)$ or $v_{j_i'}(\vec{e}_i) \geq  p(\vec{e}_i)$ for all $i \in [m]$. Therefore, there are $t_2 = m$ hyperplanes $\hyp_2$ such that for any connected component $\pspace'$ of $\pspace \setminus \hyp_2$, $f_{\vec{v}, \vec{\alpha}}^{(2)}(\vec{p})$ is linear over all $\vec{p} \in \pspace'.$

Under non-anonymous reserve prices, let $f_{\vec{p}}^{(1)}: \domain \to \{0,1\}^{nm}$ be defined so that for every buyer $j$ and every item $i$, there is a component of $f_{\vec{p}}^{(1)}(\vec{v})$ that is 1 if and only if buyer $j$ receives item $i$. There are $t_1 = m$ hyperplanes splitting $\R^{nm}$ into regions where $f_{\vec{v}}^{(1)}(\vec{p})$ is constant: for every item $i$, the component corresponding to buyer $j_i$ is 1 if and only if $v_{j_i}(\vec{e}_i) \geq p_j(\vec{e}_i)$. Next, we can write $f_{\vec{p}}^{(2)}(\vec{v}, \vec{\alpha}) = \sum_{i: \alpha[i] = 1 } \max\left\{v_{j_i'}(\vec{e}_i), p_{j_i}(\vec{e}_i)\right\} - c(\vec{\alpha})$, which is linear so long as either $v_{j_i'}(\vec{e}_i) <p_{j_i}(\vec{e}_i)$ or $v_{j_i'}(\vec{e}_i) \geq  p_{j_i}(\vec{e}_i)$ for all $i \in [m]$. Therefore, there are $t_2 = m$ hyperplanes $\hyp_2$ such that for any connected component $\pspace'$ of $\pspace \setminus \hyp_2$, $f_{\vec{v}, \vec{\alpha}}^{(2)}(\vec{p})$ is linear over all $\vec{p} \in \pspace'.$

\citet{Morgenstern16:Learning} proved that $\mclass$ and $\mclass'$ are $(m+1)$- and $(nm+1)$-dimensionally linearly separable over $\{0,1\}^m$ and $[n]^m$, respectively.
\end{proof}

\begin{theorem}\label{thm:item_pricing_add_sep}
Let $\pazocal{M}$ and $\pazocal{M}'$ be the classes of item-pricing mechanisms with anonymous prices and non-anonymous prices, respectively. If the buyers are additive, then $\mclass$ is $(m, m, 1)$-divisible and $\mclass'$ is $(nm,nm,1)$-divisible. Also, $\mclass$ and $\mclass'$ are $(m+1)$- and $(nm+1)$-dimensionally linearly separable over $\{0,1\}^m$ and $[n]^m$. Therefore, $Pdim(\mclass) = O(m \log m)$ and $Pdim(\mclass') = O(nm \log (nm))$.
\end{theorem}

\begin{proof}
We begin with anonymous reserves. For a given valuation vector $\vec{v}$, let $j_i$ be the buyer with the highest valuation for item $i$. Let $f_{\vec{p}}^{(1)}: \domain \to \{0,1\}^m$ be defined so that the $i^{th}$ component is 1 if and only if item $i$ is sold. There are $t_1 = m$ hyperplanes splitting $\R^m$ into regions where $f_{\vec{v}}^{(1)}(\vec{p})$ is constant: the $i^{th}$ component of $f_{\vec{v}}^{(1)}(\vec{p})$ is 1 if and only if $v_{j_i}(\vec{e}_i) \geq p(\vec{e}_i)$. Next, we can write $f_{\vec{p}}^{(2)}(\vec{v}, \vec{\alpha}) = \vec{\alpha} \cdot \vec{p}$, which is always linear, so we may set $t_2 = 1$.

Under non-anonymous reserve prices, let $f_{\vec{p}}^{(1)}: \domain \to \{0,1\}^{nm}$ be defined so that for every buyer $j$ and every item $i$, there is a component of $f_{\vec{p}}^{(1)}(\vec{v})$ that is 1 if and only if buyer $j$ receives item $i$. There are $t_1 = nm$ hyperplanes splitting $\R^{nm}$ into regions where $f_{\vec{v}}^{(1)}(\vec{p})$ is constant: $v_{j}(\vec{e}_i) = p_{j}(\vec{e}_i)$ for all $i$ and all $j$. Next, we can write $f_{\vec{p}}^{(2)}(\vec{v}, \vec{\alpha}) = \vec{\alpha} \cdot \vec{p}$, which is always linear, so we may set $t_2 = 1$.

\citet{Morgenstern16:Learning} proved that $\mclass$ and $\mclass'$ are $(m+1)$- and $(nm+1)$-dimensionally linearly separable over $\{0,1\}^m$ and $[n]^m$, respectively.
\end{proof}

\section{Proofs from Section~\ref{SEC:DATA}}\label{APP:DATA}
\begin{lemma}\label{lem:product}
Let $\domain = \domain_1 \times \cdots \times \domain_d$. Let $\fclass = \left\{f_{\vec{p}} : \vec{p} \in \pspace\right\}$ be a set of functions mapping $\domain$ to $\R$, parameterized by a set $\pspace = \pspace_1 \times \cdots \times \pspace_d$.
Suppose for $i \in [d]$, there exists a class $\fclass_i = \left\{f_p^{(i)} : p \in \pspace_i\right\}$ of functions mapping $\domain_i$ to $\R$ such that for any $\vec{p} = \left(p[1], \dots, p[d]\right) \in \pspace$, $f_{\vec{p}}$ decomposes additively as $f_{\vec{p}}\left(v_1, \dots, v_d\right) = \sum_{i = 1}^d f^{(i)}_{p[i]}\left(v_i\right)$. Then \[\sup_{\vec{v} \in \domain, \vec{p} \in \pspace} f_{\vec{p}}(\vec{v}) = \sum_{i = 1}^d \sup_{v \in \domain_i, p \in \pspace_i} f^{(i)}_{p}\left(v\right).\]
\end{lemma}

\begin{proof}
Recall that for any set $A \subseteq \R$, $s = \sup A$ if and only if:
\begin{enumerate}
\item For all $\epsilon > 0$, there exists $a \in A$ such that $a > s - \epsilon$, and
\item For all $a \in A$, $a \leq s$.
\end{enumerate}

Let $t_i = \sup_{v \in \domain_i, p \in \pspace_i} f^{(i)}_{p}\left(v\right)$ and let $t = \sum_{i=1}^d t_i$. We will show that $t = \sup_{\vec{v} \in \domain, \vec{p} \in \pspace} f_{\vec{p}}(\vec{v}).$

First, we will show that condition (1) holds. In particular, we want to show that for all $\epsilon > 0$, there exists $\vec{v} \in \domain$ and $\vec{p} \in \pspace$ such that $f_{\vec{p}}(\vec{v}) > t - \epsilon$. Since $t_i = \sup_{v \in \domain_i, p \in \pspace_i} f^{(i)}_{p}\left(v\right)$, we know that there exists $v_i \in \domain_i, p_i \in \pspace$ such that $ f^{(i)}_{p_i }\left(v_i\right) > t_i - \epsilon / d$. Therefore, letting $\vec{p} = \left(p_1, \dots, p_d\right)$, we know that $f_{\vec{p}}\left(v_1, \dots, v_d\right) = \sum_{i = 1}^d f^{(i)}_{p_i}\left(v_i\right) > \sum_{i = 1}^d t_i - \epsilon = t-\epsilon.$ Since $\left(v_1, \dots, v_d\right) \in \domain$ and $\left(p_1, \dots, p_d\right) \in \pspace$, we may conclude that condition (1) holds.

Next, we will show that condition (2) holds. In particular, we want to show that for all $\vec{v} \in \domain$ and $\vec{p} \in \pspace$, $f_{\vec{p}}(\vec{v}) \leq t$. We know that $f^{(i)}_{p[i]}\left(v[i]\right) \leq t_i$, which means that $f_{\vec{p}}(\vec{v}) = \sum_{i = 1}^d f^{(i)}_{p[i]}\left(v[i]\right) \leq \sum_{i = 1}^d t_i =t.$ Therefore, condition (2) holds.
\end{proof}

\secondPriceProduct*

\begin{proof}
We begin with anonymous second-price auctions, which are parameterized by a set $\pspace \subset \R^m$. Without loss of generality, we may write $\pspace = \pspace_1 \times \cdots \times \pspace_m$, where $\pspace_i \subset \R$. Given a valuation vector $\vec{v}$ and an item $i$, let $\vec{v}(i) \in \R^n$ be all $n$ buyers' values for item $i$. Let $\profit_p(\vec{v}(i))$ be the profit obtained by selling item $i$ with a reserve price of $p$. Notice that for any $\vec{p} \in \pspace$, $\profit_{\vec{p}}(\vec{v}) = \sum_{i = 1}^m \profit_{p[i]}(\vec{v}(i))$. Let $\domain_i$ be the support of the distribution over $\vec{v}(i)$ and let $U_i = \sup_{p \in \pspace_i, \vec{v}(i) \in \domain_i} \profit_{p}(\vec{v}(i))$. Next, let
$\domain$ be the support of $\dist$. By definition, since $U$ is the maximum profit achievable via second price auctions over valuation vectors from $\domain$, we may write $U = \sup_{\vec{v} \in \domain, \vec{p} \in \pspace}\profit_{\vec{p}}(\vec{v})$.  Since $\dist$ is item-independent, we know that $\domain = \domain_1 \times \cdots \times \domain_m$. Therefore, we may apply Lemma~\ref{lem:product}, which tells us that $U = \sum_{i = 1}^m U_i$. Finally, each class of functions $\left\{\profit_{p} : p \in \pspace_i\right\}$ is $(1, 2)$-delineable, since for $\vec{v}(i) \in \domain_i$, $\profit_{\vec{v}(i)}(p)$ is linear so long as $p$ is larger than the largest component of $\vec{v}(i)$, between the second largest and largest component of $\vec{v}(i)$, or smaller than the second largest component of $\vec{v}(i)$. By Corollary~\ref{cor:data}, we may conclude that for any set of samples $\sample \sim \dist^N$, $\erad(\mclass) \leq O\left(U\sqrt{1/N}\right)$.

The bound on $\erad(\mclass')$ follows by almost the exact same logic, except for a few adjustments. First of all, the class is defined by $nm$ parameters coming from some set $\pspace \subseteq \R^{nm}$, since there are $n$ non-anonymous prices per item. Without loss of generality, we assume $\pspace = \pspace_1 \times \cdots \times \pspace_m$, where $\pspace_i \subseteq \R^n$ is the set of non-anonymous prices for item $i$. Given a set of non-anonymous prices $\vec{p} \in \R^n$ for item $i$, let $\profit_{\vec{p}}(\vec{v}(i))$ be the profit of selling the item the bidders defined by $\vec{v}(i)$ given the reserve prices $\vec{p}$. Notice that $\profit_{\vec{v}(i)}(\vec{p})$ is linear so long as for each bidder $j$, $p[j]$ is either larger than their value for item $i$ or smaller than their value. Thus, the set $\left\{\profit_{\vec{p}} : \vec{p} \in \pspace_i\right\}$ is $(n,n)$-delineable. Defining each $U_i$ in the same way as before, Lemma~\ref{lem:product} guarantees that $U = \sum_{i = 1}^m U_i$. Therefore, by Corollary~\ref{cor:data}, we may conclude that for any set of samples $\sample \sim \dist^N$, $\erad(\mclass') \leq O\left(U\sqrt{n \log n/N}\right)$.
\end{proof}

\itemProduct*

\begin{proof}
We begin with anonymous item-pricing mechanisms, which are parameterized by a set $\pspace \subset \R^m$. Without loss of generality, we may write $\pspace = \pspace_1 \times \cdots \times \pspace_m$, where $\pspace_i \subset \R$. Given a valuation vector $\vec{v}$ and an item $i$, let $\vec{v}(i) \in \R^n$ be all $n$ buyers' values for item $i$. Let $\profit_p(\vec{v}(i))$ be the profit obtained by selling item $i$ at a price of $p$, i.e., $\profit_p(\vec{v}(i)) = \textbf{1}_{\left\{||\vec{v}(i)||_{\infty} \geq p\right\}} (p - c(\vec{e}_i))$. Notice that for any $\vec{p} \in \pspace$, $\profit_{\vec{p}}(\vec{v}) = \sum_{i = 1}^m \profit_{p[i]}(\vec{v}(i))$. Let $\domain_i$ be the support of the distribution over $\vec{v}(i)$ and let $U_i = \sup_{p \in \pspace_i, \vec{v}(i) \in \domain_i} \profit_{p}(\vec{v}(i))$. Next, let
$\domain$ be the support of $\dist$. By definition, since $U$ is the maximum profit achievable via item-pricing mechanisms over valuation vectors from $\domain$, we may write $U = \sup_{\vec{v} \in \domain, \vec{p} \in \pspace}\profit_{\vec{p}}(\vec{v})$.  Since $\dist$ is item-independent, we know that $\domain = \domain_1 \times \cdots \times \domain_m$. Therefore, we may apply Lemma~\ref{lem:product}, which tells us that $U = \sum_{i = 1}^m U_i$. Finally, each class of functions $\left\{\profit_{p} : p \in \pspace_i\right\}$ is $(1, 1)$-delineable, since for $\vec{v}(i) \in \domain_i$, $\profit_{\vec{v}(i)}(p)$ is linear so long as $||\vec{v}(i)||_{\infty} \leq p$ or $||\vec{v}(i)||_{\infty} > p$. By Corollary~\ref{cor:data}, we may conclude that for any set of samples $\sample \sim \dist^N$, $\erad(\mclass) \leq O\left(U\sqrt{1/N}\right)$.

The bound on $\erad(\mclass')$ follows by almost the exact same logic, except for a few adjustments. First of all, the class is defined by $nm$ parameters coming from some set $\pspace \subseteq \R^{nm}$, since there are $n$ non-anonymous prices per item. Without loss of generality, we assume $\pspace = \pspace_1 \times \cdots \times \pspace_m$, where $\pspace_i \subseteq \R^n$ is the set of non-anonymous prices for item $i$. Given a set of non-anonymous prices $\vec{p} \in \R^n$ for item $i$, let $\profit_{\vec{p}}(\vec{v}(i))$ be the profit of selling the item to the buyers defined by $\vec{v}(i)$ given the prices $\vec{p}$. Notice that $\profit_{\vec{v}(i)}(\vec{p})$ is linear so long as for each buyer $j$, $p(\vec{e}_j)$ is either larger than their value for item $i$ or smaller than their value. Thus, the set $\left\{\profit_{\vec{p}} : \vec{p} \in \pspace_i\right\}$ is $(n,n)$-delineable. Defining each $U_i$ in the same way as before, Lemma~\ref{lem:product} in Appendix~\ref{APP:DATA} guarantees that $U = \sum_{i = 1}^m U_i$. Therefore, by Corollary~\ref{cor:data}, we may conclude that for any set of samples $\sample \sim \dist^N$, $\erad(\mclass') \leq O\left(U\sqrt{n \log n/N}\right)$.
\end{proof}

\medskip\emph{Menus of item lotteries.} A \emph{length-$\ell$ item lottery menu} is a set of $\ell$ lotteries per item. The menu for item $i$ is $M_i = \left\{\left(\phi_i^{\left(0\right)}, p_i^{\left(0\right)}\right), \left(\phi_i^{\left(1\right)}, p_i^{\left(1\right)}\right), \dots, \left(\phi_i^{\left(\ell\right)}, p_i^{\left(\ell\right)}\right)\right\}$, where $\phi_i^{\left(0\right)}= p_i^{\left(0\right)}= 0$. The buyer chooses a lottery $\left(\phi_i^{\left(j_i\right)}, p_i^{\left(j_i\right)}\right)$ per menu $M_i$, pays $\sum_{i = 1}^m p^{\left(j_i\right)}$, and receives each item $i$ with probability $\phi_i^{\left(j_i\right)}$. 

\begin{restatable}{lemma}{lotteryProduct}\label{lem:lottery_product}
	Let $\mclass$ be the set of length-$\ell$ item lottery menus. If the buyer is additive, $\dist$ is item-independent, and the cost function is additive, then for any set $\sample \sim \dist^N$, $\erad\left(\mclass\right) \leq 180\sqrt{2\ell \log\left(8\ell^3\right)}$.
\end{restatable}

\begin{proof}
For a given menu $M = \left(M_1, \dots, M_m\right)$ of item lotteries, let $\profit_{M_i}(\vec{v})$ be the profit achieved from menu $M_i$. Since the cost function is additive, \[\profit_{M_i} (\vec{v}) = p_{i,\vec{v}} - \E_{q \sim \phi_{i, \vec{v}}}\left[c(q)\right] = p_{i,\vec{v}} - c(\vec{e}_i)\cdot \phi_{i, \vec{v}},\] where $(p_{i, \vec{v}}, \phi_{i, \vec{v}})$ is the lottery in $M_i$ that maximizes the buyer's utility.
Notice that $\profit_{M}(\vec{v}) = \sum_{i = 1}^m \profit_{M_i}(v(\vec{e}_i))$. Let $\domain_i$ be the support of the distribution $\dist_i$ over $v(\vec{e}_i)$ and let $U_i = \sup_{M_i, v(\vec{e}_i) \in \domain_i} \profit_{M_i}(v(\vec{e}_i))$. By definition, since $U$ is the maximum profit achievable via item menus over valuation vectors from $\domain$, we may write $U = \sup_{\vec{v} \in \domain, M \in \mclass}\profit_{M}(\vec{v})$.  Since $\dist$ is a product distribution, we know that $\domain = \domain_1 \times \cdots \times \domain_m$. Therefore, we may apply Lemma~\ref{lem:product}, which tells us that $U = \sum_{i = 1}^m U_i$. Finally, for each $i \in [n]$, the class of all single-item lotteries $M_i$ is $(2\ell, \ell^2)$-delineable, since for $v(\vec{e}_i) \in \domain_i$, the lottery the buyer chooses depends on the ${\ell + 1 \choose 2}$ hyperplanes $\phi_i^{(j)}v(\vec{e}_i) - p_i^{(j)} = \phi_i^{(j')}v(\vec{e}_i) - p_i^{(j')}$ for $j, j' \in \{0, \dots, \ell\}$, and once the lottery is fixed, $\profit_{M_i} (\vec{v})$ is a linear function.
\end{proof}

\pdimLower*

\begin{proof}
Let $\pazocal{M}$ be the class of item-pricing mechanisms with anonymous prices. We construct a set $\sample$ of $m$ single-buyer, $m$-item valuation vectors that can be shattered by $\pazocal{M}$. Let $\vec{v}^{(i)}$ be valuation vector where $v_1^{(i)}(\vec{e}_i) = 3$ and $v_1^{(i)}(\vec{e}_j) = 0$ for all $j \not= i$ and let $\sample = \left\{\vec{v}^{(1)}, \dots, \vec{v}^{(m)}\right\}$. For any $T \subseteq [m]$, let $M_T$ be the mechanism defined such that the price of item $i$ is 2 if $i \in T$ and otherwise, its price is 0. If $i \in T$, then $\profit_{M_T}(\vec{v}^{(i)}) = 2$ and otherwise, $\profit_{M_T}(\vec{v}^{(i)}) = 0$. Therefore, the targets $z^{(1)} = \dots = z^{(m)} = 1$ witness the shattering of $\sample$ by $\pazocal{M}$. This example also proves that the pseudo-dimension of the class of second-price auctions with anonymous reserve prices is also at least $m$, since in the single-buyer case, this class is identical to $\pazocal{M}$.

Next, let $\pazocal{M}'$ be the class of item-pricing mechanisms with non-anonymous prices. We construct a set $\sample$ of $nm$ $n$-buyer, $m$-item valuation vectors that can be shattered by $\pazocal{M}'$. For $i \in [m]$ and $j \in [n]$, let $\vec{v}^{(i,j)}$ be valuation vector where $v_j^{(i,j)}(\vec{e}_i) = 3$ and $v_{j'}^{(i,j)}(\vec{e}_{i'}) = 0$ for all $(i', j') \not= (i,j)$. Let $\sample = \left\{\vec{v}^{(i,j)}\right\}_{i \in [m], j \in [n]}$. For any $T \subseteq [m] \times [n]$, let $M_T$ be the mechanism defined such that the price of item $i$ for buyer $j$ is 2 if $(i,j) \in T$ and otherwise, it is 0. If $(i,j) \in T$, then $\profit_{M_T}(\vec{v}^{(i,j)}) = 2$ and otherwise, $\profit_{M_T}(\vec{v}^{(i,j)}) = 0$. Therefore, the targets $z^{(i,j)} = 1$ for all $i\in [m], j \in [n]$ witness the shattering of $\sample$ by $\pazocal{M}$. This example with the prices as reserve prices also proves that the pseudo-dimension of the class of second-price auctions with non-anonymous reserve prices is at least $nm$.
\end{proof}

\section{Proofs from Section~\ref{SEC:SPM}}\label{APP:SPM}
\itemSPM*

\begin{proof}
This theorem follows from the fact that $\mclass_k$ is $(km, nm)$-delineable. Every mechanism in $\mclass_k$ is defined by $km$ parameters, one price per item per price group, and for every buyer $j$, the items they are willing to buy are defined by the $m$ hyperplanes $v_j(\vec{e}_i) = p_j(\vec{e}_i)$ for every item $i$. Therefore, the theorem follows from Theorems~\ref{thm:pdim} and \ref{thm:main_pdim}, and by multiplying $\delta$ with $w(k)$.
\end{proof}

\bigskip\emph{Two-part tariffs.} Let $\pazocal{M}$ be the class of anonymous two-part tariff menus, by which we mean the union of all length-$\ell$ menus of two-part tariffs with anonymous prices. Similarly, let $\pazocal{M}'$ be the class of non-anonymous two-part tariff menus. For a given menu $M$ of two-part tariffs, let $\ell_M$ be the length of its menu.

\begin{theorem}\label{thm:2part_nonuniform}
Let $w: \N \to [0,1]$ be a weight function such that $\sum w(i) \leq 1$. Then for any $\delta \in (0,1)$, with probability at least $1-\delta$ over the draw of a set of samples of size $N$ from $\pazocal{D}$, for any mechanism $M \in \pazocal{M}$, the difference between the average profit of $M$ over the set of samples and the expected profit of $M$ over $\pazocal{D}$ is 
\[O\left(U\sqrt{\frac{\ell_M \log(n\kappa\ell_M)}{N}} + U\sqrt{\frac{1}{N} \log \frac{1}{\delta \cdot w(\ell_M)}}\right).\] Also, with probability at least $1-\delta$ over the draw of a set of samples of size $N$ from $\pazocal{D}$, for any mechanism $M \in \pazocal{M}'$, the difference between the average profit of $M$ over the set of samples and the expected profit of $M$ over $\pazocal{D}$ is at most \[O\left(U\sqrt{\frac{n\ell_M \log(n\kappa\ell_M)}{N}} + U\sqrt{\frac{1}{N} \log \frac{1}{\delta \cdot w(\ell_M)}}\right).\]
\end{theorem}

\bigskip\emph{$\Qset$-boosted AMAs.} For an AMA $M$, let $\Qset_M$ be the set of all allocations $Q$  such that $\lambda\left(Q\right) > 0$.
\begin{theorem}\label{thm:Oboosted_SPM}
	Let $\pazocal{M}$ be the class of AMAs and let $w$ be a weight function that maps sets of allocations $\Qset$ to $[0,1]$ such that $\sum w\left(\Qset\right) \leq 1$. With probability $1-\delta$ over $\sample \sim \pazocal{D}^N$, for any $M \in \pazocal{M}$, \[\left|\profit_{\sample}\left(M\right) - \profit_{\dist}\left(M\right)\right| \leq 360U \sqrt{\frac{nm\left(n+|\Qset_M|\right)\log(4n)}{N}} + 4U\sqrt{\frac{2}{N}\ln\frac{4}{\delta \cdot w\left(\Qset_M\right)}}.\]
\end{theorem}

\bigskip\emph{$\Qset$-boosted $\lambda$-auctions.} For the next theorem, given a $\lambda$-auction $M$, let $\Qset_M$ be the set of all allocations $Q$ such that $\lambda(Q) > 0$.

\begin{theorem}\label{thm:Oboosted_lambda_nonuniform}
Let $\pazocal{M}$ be the class of $\lambda$-auctions and let $w$ be a weight function which maps sets of allocations $\Qset$ to $[0,1]$ such that $\sum w(\Qset) \leq 1$. Then for any $\delta \in (0,1)$, with probability at least $1-\delta$ over the draw of a set of samples of size $N$ from $\pazocal{D}$, for any mechanism $M \in \pazocal{M}$, the difference between the average profit of $M$ over the set of samples and the expected profit of $M$ over $\pazocal{D}$ is at most \[O\left(U\sqrt{\frac{|\Qset_M| \log(n|\Qset_M|)}{N}} + U\sqrt{\frac{1}{N} \log \frac{1}{\delta \cdot w(\Qset_M)}}\right).\]
\end{theorem}

\bigskip\emph{Menu lotteries.} Let $\pazocal{M}$ be the class of lottery menus, by which we mean the union of all length-$\ell$ lottery menus. For a given lottery menu $M$, let $\ell_M$ be the length of its menu.

\begin{theorem}\label{thm:lottery_nonuniform}
Let $w: \N \to [0,1]$ be a weight function such that $\sum w(i) \leq 1$. Then for any $\delta \in (0,1)$, with probability at least $1-\delta$ over the draw of a set of samples of size $N$ from $\pazocal{D}$, for any mechanism $M \in \pazocal{M}$, the difference between the average profit of $M$ over the set of samples and the expected profit of $M$ over $\pazocal{D}$ is 
\[O\left(U\sqrt{\frac{\ell_M \log(n\ell_M)}{N}} + U\sqrt{\frac{1}{N} \log \frac{1}{\delta \cdot w(\ell_M)}}\right).\]
\end{theorem}

\end{document}